\documentclass[11pt]{article}
\pdfoutput=1
\usepackage{tikz}
\usepackage{subcaption}
\usepackage{wrapfig}
\usetikzlibrary{fit, backgrounds}
\usetikzlibrary{positioning}
\usepackage{pifont}
\usepackage{figures/tex_figures/quiver} 
\usepackage{float}

\usepackage{pgfplots}
\usetikzlibrary{patterns}
\pgfplotsset{compat=1.6}
\usepgfplotslibrary{fillbetween}

\usepackage{geometry, amsmath, amsthm, latexsym, amssymb, graphicx, xcolor, dsfont}
\usepackage{natbib}
\usepackage{mathtools}

\DeclarePairedDelimiter{\floor}{\lfloor}{\rfloor}
\newcommand{\reals}{\mathbb{R}}

\usepackage{hyperref}
\usepackage{cleveref}
\hypersetup{
	colorlinks   = true, 
	urlcolor     = blue, 
	linkcolor    = blue, 
	citecolor   = blue 
}

\usepackage{minitoc}

\newtheorem{theorem}{Theorem}[section]
\newtheorem{proposition}[theorem]{Proposition}
\newtheorem{example}[theorem]{Example}
\newtheorem{lemma}[theorem]{Lemma}
\newtheorem{corollary}[theorem]{Corollary}
\newtheorem{definition}[theorem]{Definition}
\newtheorem{remark}[theorem]{Remark}

\title{Logarithmic Width Suffices for Robust Memorization}

\newcommand{\Dcal}{\mathcal{D}}

\newcommand{\secref}[1]{Section~\ref{#1}}
\newcommand{\subsecref}[1]{Subsection~\ref{#1}}

\newcommand{\figref}[1]{Figure~\ref{#1}}
\renewcommand{\eqref}[1]{Eq.~(\ref{#1})}
\newcommand{\lemref}[1]{Lemma~\ref{#1}}
\newcommand{\thmref}[1]{Theorem~\ref{#1}}
\newcommand{\propref}[1]{Proposition~\ref{#1}}
\newcommand{\appref}[1]{Appendix~\ref{#1}}
\newcommand{\exref}[1]{Example~\ref{#1}}

\newcommand{\defref}[1]{Definition~\ref{#1}}

\newcommand{\remarkref}[1]{Remark~\ref{#1}}

\date{}
\author{
Amitsour Egosi\thanks{Weizmann Institute of Science, Israel, \texttt{amitsour.egosi@weizmann.ac.il}}
 \quad
     Gilad Yehudai\thanks{Center for Data Science, New York University, \texttt{gy2219@nyu.edu}} \quad
	Ohad Shamir \thanks{Weizmann Institute of Science, Israel, \texttt{ohad.shamir@weizmann.ac.il}.}
}

\begin{document}

\doparttoc 
\faketableofcontents 

\maketitle{}

\begin{abstract}
    The memorization capacity of neural networks with a given architecture has been thoroughly studied in many works. Specifically, it is well-known that memorizing $N$ samples can be done using a network of constant width, independent of $N$. However, the required constructions are often quite delicate. In this paper, we consider the natural question of how well feedforward ReLU neural networks can memorize \emph{robustly}, namely while being able to withstand adversarial perturbations of a given radius. 
    We establish both upper and lower bounds on the possible radius for general $l_p$ norms, implying (among other things) that width \emph{logarithmic} in the number of input samples is necessary and sufficient to achieve robust memorization (with robustness radius independent of $N$).
\end{abstract}

\section{Introduction}

The ability of neural networks to \emph{memorize} labeled datasets is a central question in the study of their expressive power. Given some input domain $\mathcal{X}$, output domain $\mathcal{Y}$, and dataset size $N$, we say that a network memorizes datasets of  size $N$, if for every labeled dataset $\mathcal{D}\subseteq \mathcal{X}\times\mathcal{Y}$, where $|\Dcal|=N$, we can find parameters such that the resulting network $f:\mathcal{X}\rightarrow\mathcal{Y}$ perfectly fits the dataset (that is, $f(x)=y$ for every labeled pair $(x,y)\in\mathcal{D}$). The main question here -- which has been studied in many recent works (see \secref{sec:related} for  details) -- is to characterize the size/architecture of the networks that have enough expressive power to memorize any dataset of a given size $N$.

However, merely fitting a given dataset is not enough for most tasks, and a desirable property for trained networks is that they remain robust to noise and minor modifications in the dataset. This robustness property allows neural networks to generalize from observed data points to unseen data points. Furthermore, neural networks have been shown to be vulnerable to adversarial attacks \citep{szegedy2013intriguing,carlini2017adversarial, papernot2017practical, athalye2018obfuscated} in the form of slightly perturbed examples, where (in the context of visual data) the perturbation is often imperceptible to the human eye. Moreover, existing constructions of memorizing networks are often quite delicate, and not at all robust to such perturbations. This motivates the question of characterizing the networks that have enough capacity to \emph{robustly} memorize a dataset. Concretely, considering datasets of the form $\{(x_1,y_1),\dots,(x_N,y_N)\}\subset\reals^d \times \{1,\dots,C\}$ in a multiclass setting, and a robustness radius $\sigma > 0$, the problem we wish to study is the following: How large does a standard feedforward ReLU network $f$ need to be, so that for any dataset of size $N$ as above, there exists a choice of parameters such that $f(a_i) = y_i$ for every $a_i \in B^{d}_{p}(x_i,\sigma)$ (where $B^{d}_{p}(x_i,\sigma)$ is a ball of radius $\sigma$ around $x_i$ in $l_p$ norm). 

When considering the notion of the size of a network in the problem of robust memorization, one can define it in terms of depth, width or the total number of parameters of the network. Several works have observed empirically that wider networks tend to be more robust to adversarial perturbations \citep{madry2017towards,wu2021wider,zhu2022robustness}. This connection between the radius of robustness and the necessary width for robust memorization is still not well studied. A recent work \citep{yuoptimal} showed that width $ k \geq d$ is necessary for \emph{optimal} robust memorization in $l_{\infty}$ norm. Optimality in their work requires the existence of a memorizing network of width $k$ for \emph{all} possible robustness radii, not accounting for possible finer relation between the width of the network and the robustness radius. 

In this work we study this connection between width and robustness, and in particular we seek to determine what is the minimal width $k$ required to ensure that for any dataset there exists a width $k$ network that can memorize it with robustness radius $\sigma$. In the non-robust case, it is known that memorization can be achieved with constant-width networks \citep{park2021provable,vardi2021optimal}. We show that for robust memorization, there exists a trade-off between the width $k$ and radius $\sigma$. In our analysis we consider datasets with minimal $l_{2}$ distance of $\delta$ between different classes, called $\delta$-separated datasets. This separation assumption is necessary for robust memorization, since the robustness radius is limited by the distance between differently-labeled points (see \remarkref{remark:valid range for radius} for more details). Our main contributions can be summarized as follows:

\begin{itemize}
    \item We show nearly tight bounds on the possible robustness radius $\sigma$ in $l_p$ norm for memorizing a $\delta$-separated dataset of size $N$ using a network of width $k$. Specifically, the following holds for some universal constants $c_{1}, c_{2}$ and for every $p \in [2, \infty]$:
    \begin{enumerate}
        \item If $\frac{\sigma}{\delta} < \frac{c_{1}}{d^{1-1/p}}N^{-\frac{2}{k-6}}$, then any such dataset can be robustly memorized by a network of width $k$ (\thmref{thm:upper bound memorization}). 
        \item If $\frac{\sigma}{\delta} > c_{2}N^{-\frac{2}{k}}$ then there exists a dataset that cannot be robustly memorized by any network of width $k$ (\thmref{thm:lower bound memorization}).
    \end{enumerate}
    \item Both of the results above rely on a robust variant of the Johnson-Lindenstrauss Lemma that we develop (\thmref{thm:main pres}), which revolves around projecting high-dimensional points to a lower-dimensional subspace while maintaining separation between neighborhoods of data points, and may be of independent interest.
    \item The bounds above apply to the regime where the desired width is relatively small (less than the data dimension). In addition, we develop bounds for the more permissive regime where the width is larger than the data dimension (\thmref{thm:upper bound memorization big k} and \thmref{thm:upper bound memorization big k special cases}), extending the results of \citep{li2022robust,yuoptimal} to other norms as well as to smaller widths. 
\end{itemize}

The results above show that for guaranteeing robust memorization (with robustness radius independent of $N$ and with width smaller than the dimension),
a necessary and sufficient condition is that the width would depend logarithmically on $N$. Alternatively, if we wish to robustly memorize with constant width $k$ independent of $N$, then the robustness parameter $\frac{\sigma}{\delta}$ necessarily decays polynomially in the dataset size $N$. This means that constructions similar to those from \cite{park2021provable,vardi2021optimal}, which achieve optimal memorization in terms of the number of parameters and with width independent of $N$, cannot achieve optimal robustness.

\section{Related Works}\label{sec:related}

\paragraph{Memorization}

Memorization in neural networks is a well studied field with many established results. \cite{baum1988capabilities,bubeck2020network,huang1991bounds,huang1998upper,sartori1991simple,zhang2021understanding} proved under different settings that $O(N)$ neurons and parameters are enough to memorize $N$ data points. \cite{huang2003learning,yun2019small} improved these results and showed that $O(\sqrt{N})$ neurons are enough to memorize $N$ points with a $3$-layer neural networks, although the number of parameters is still $O(N)$.  \cite{park2021provable} gave the first sub-linear parameter memorization bound, with $N^{2/3}$ parameters to memorize $N$ points. Finally, \cite{vardi2021optimal} proved that memorizing $N$ points can be done using a network with $\tilde{O}(\sqrt{N})$ parameters. This is known to be optimal up to log terms due to VC dimension lower-bounds \citep{goldberg1995bounding,bartlett2019nearly}. Note that the width of the memorizing networks in \cite{park2021provable,vardi2021optimal} is a universal constant, namely $12$ in \cite{vardi2021optimal}. Also, note that our results imply that the constructions from \cite{park2021provable,vardi2021optimal} cannot achieve optimal robustness

\paragraph{Robust memorization}

Several works proved the existence of networks that memorize robustly using different methods. \cite{yang2020closer,bastounis2021mathematics} proved there exists locally Lipschitz classifiers, which implies some form of local robustness, although they did not give specific bounds on the size of the classifier. \cite{li2022robust} showed the existence of robust memorization networks through VC dimension arguments. Most closely related to our work is \cite{yuoptimal}, which proves upper and lower bounds for robust memorization. In particular, they show that robust memorization with the optimal robust radius in $l_\infty$ norm (including the constants) cannot be achieved  if the width is smaller than the data dimension. We extend their result by showing the intricate trade-offs between the width of the network and the robustness radius. 

\paragraph{Robustness and width}
Several papers observed empirically that there is a connection between the width of the neural network and its robustness properties. \cite{madry2017towards} observed that wider networks tend to be more robust, even without adversarial training. \cite{wu2021wider,zhu2022robustness} study the effect of the width on adversarial training, and provide theoretical justification in the NTK regime \citep{jacot2018neural,allen2019convergence,gao2019convergence}. Our work focuses on the expressive capacity required for robustness, rather than the optimization process which is studied in these works.

\section{Preliminaries}

\paragraph{Notations.} For every $0<p<\infty$ and $x\in\mathbb{R}^d$ denote $\lVert x \rVert_{p}=\left(\sum\limits^{d}_{i=1}\lvert x_{i} \rvert^{p}\right)^{1/p}$, and $\lVert x \rVert_{\infty}=\underset{1\leq i \leq d}{\max}\lvert x_{i} \rvert$. For $0 < p < 1$ the function $\lVert \cdot \rVert_{p}$ is a quasi-norm, and for $1\leq p \leq \infty$ it is the $l_p$ norm with an induced metric $\text{dist}_{p}$. For all $0<p\leq \infty$ we define the $l_p$ ball of radius $r$ around $x$ as $B_{p}^{d}\left(x, r\right)=\left\{x^{\prime}\in\mathbb{R}^{d}\mid \lVert x^{\prime}- x \rVert_{p} \leq r\right\}$. Note that all balls in our work are closed balls. $e$ is Euler's number $2.718...$~.

For any $0<p\leq \infty$ we will denote $c^{+}_{p}(d)=d^{\left[\frac{1}{2}-\frac{1}{p}\right]_{+}}$ and $c^{-}_{p}(d)=d^{\left[\frac{1}{2}-\frac{1}{p}\right]_{-}}$, where $[x]_{+}=\max\{0,x\}$ (also called the ReLU activation) and $[x]_{-}=\min\{0,x\}$. Note that $c^{+}_{p}(d)$ is the radius of the $l_{2}$ ball that encloses the unit $l_{p}$ ball, and that $c^{-}_{p}(d)$ is the radius of the $l_{2}$ ball that is inscribed in the unit $l_{p}$ ball (see \lemref{lma:lp ball and lq ball} in \appref{subsec:additional lemmas}). In these definitions, for $p=\infty$ we define $\frac{1}{\infty}=0$. For additional notations used in the appendices, see \appref{sec:additional notations}.

\paragraph{Neural Networks.} In this paper, we focus on feedforward ReLU neural networks, defined as follows: 
\begin{definition}\label{def:Neural Network}
    Let $d\in \mathbb{N}_{\geq2}$, $L\in\mathbb{N}$ and $d_{0}, d_{1}, ..., d_{L}\in \mathbb{N}$ with $d_{L}=1, d_{0}=d$, and let $W^{(l)}\in \mathbb{R}^{d_{l}\times d_{l-1}}$, $b^{(l)}\in \mathbb{R}^{d_{l}}$ for all $1\leq l \leq L$. Denote $T^{(l)}(x)=W^{(l)}x+b^{(l)}$. We will define a \textbf{feed forward ReLU neural network} to be $f:\mathbb{R}^{d_{0}}\longrightarrow \mathbb{R}$ given by
    \begin{equation*}
        f=T^{(L)}\circ [\cdot]_{+} \circ T^{(L-1)} \circ ... \circ [\cdot]_{+} \circ T^{(1)} 
    \end{equation*}
    where $[\cdot]_{+}$ is applied element-wise. We will say that the \textbf{depth} of $f$ is $\mathcal{L}(f):=L$, the \textbf{architecture} of $f$ is $\mathcal{A}(f):=\left(d_{0}, d_{1}, ..., d_{L}\right)$, and the \textbf{width} of $f$ is $\mathcal{W}(f):=\max\left\{d_{1}, ..., d_{L-1}\right\}$.
\end{definition}

\paragraph{Data Assumptions and Robustness.}
Let $N,d\in\mathbb{N}_{\geq2}$  , $0<\delta, \sigma$. We will use $\delta$ to denote the separation distance between different data classes and $\sigma$ to denote the radius of robustness. Formally, we use the following definitions:
\begin{definition}\label{def:dataset}
    Let $\mathcal{D}=\left\{(x_{i}, y_{i})\right\}^{N}_{i=1}\subseteq \mathbb{R}^{d}\times\left[C \right]$ be a \textbf{dataset} of size $N$ with $C$ classes, comprised of \textbf{data points} $x_{i}$ and \textbf{labels} $y_{i}$. We will denote by $\mathcal{D}_{d,N,C}$ the set of all such datasets. We say that a dataset $\mathcal{D}\in\mathcal{D}_{d,N,C}$ is a \textbf{$\delta$-separated} dataset for some $0<\delta$, if $\min \left\{\lVert x_{i} - x_{j}\rVert_{2} \mid  y_{i} \neq y_{j}\right\}=\delta$, and denote by $\mathcal{D}_{d,N,C}(\delta)$ the set of all such datasets.
\end{definition}

\begin{definition}\label{def:robust memorization}
    Let $\mathcal{D}\in\mathcal{D}_{d,N,C}$, $p\in(0,\infty]$ and $0\leq \sigma$. We say that a function $f:\reals^d\rightarrow\reals$ \textbf{$(\sigma,p)$-robustly memorizes} the dataset $\mathcal{D}$ if for all $i\in\left[N\right]$ and $x\in B^{d}_{p}(x_{i},\sigma)$ one has $f(x)=y_{i}$.
\end{definition}

\section{Main Results}\label{sec:main results}

In this section, we present the main theorems that connect robust memorization and the width of the memorizing neural network, as well as proof sketches (with full proofs appearing in \appref{sec:proof of main results}). In our results in this section we will use the definition of $\delta$-separated dataset from above, where for concreteness we measure separation in terms of the $l_2$ norm (see  \appref{sec:Separation in $l_{q}$ norm} for an extension to $l_q$ norms for any $q\in [1,\infty]$).  In the following, we let $N,d,C\in\mathbb{N}_{\geq 2}$, $k\in\mathbb{N}$, $0 < \delta, \sigma$ and $p\in (0,\infty]$.

\begin{remark}[Robustness parameter $\frac{\sigma}{\delta}$ cannot exceed $\frac{1}{2c_p^+(d)}$]\label{remark:valid range for radius}
    Given some $0<\delta$, we wish to find the maximal possible value of $\sigma$ that allows for $(\sigma, p)$-robust memorization, of any $\delta$-separated dataset, using a width $k$ network. In the case of $\sigma$-neighborhoods with respect to the $l_{2}$ norm, the value of $\frac{\sigma}{\delta}$ must lie in the range $[0,\frac{1}{2})$. Indeed, if we allow $\frac{\delta}{2} \leq \sigma$ then the $\sigma$-neighborhood of two data points with different labels might intersect, so we cannot ensure robust memorization. Similarly, for general $l_{p}$ norms, if we allow $\frac{\delta}{2} \leq c^{+}_{p}(d)\sigma$ then two $l_{2}$ balls of radius $c^{+}_{p}(d)\sigma$ might intersect, and so their enclosed $l_{p}$ balls of radius $\sigma$ might intersect. Therefore, the task of guaranteeing a $(\sigma,p)$-robust memorization for \textbf{every} possible $\delta$-separated dataset can only be considered in the range $0\leq \frac{\sigma}{\delta} <\frac{1}{2c^{+}_{p}(d)}$.
\end{remark}

\subsection{Robust Memorization With Large Width}

We first consider the easier case, where the desired width $k$ can be larger than the data dimension $d$. In this case, for all values $\sigma$ in the applicable range, one can $(\sigma,p)$-robustly memorize any $\delta$-separated dataset with a width $k$ network:

\begin{theorem}\label{thm:upper bound memorization big k}
    If $d+6 \leq k$ and 
    $\frac{\sigma}{\delta} < \frac{1}{2c^{+}_{p}(d)}$,
    then for every $\delta$-separated dataset $\mathcal{D}\in\mathcal{D}_{d,N,C}(\delta)$, there exists a neural network $f:\mathbb{R}^d\rightarrow \mathbb{R}$ with width $k$ and depth $O\left(Nd\log_{2}\left(\frac{d}{1-\frac{2c^{+}_{p}(d)\sigma}{\delta}}\right)\right)$ that $(\sigma, p)$-robustly memorizes the dataset $\mathcal{D}$.
\end{theorem}

Note that as $\frac{\sigma}{\delta}$ approaches $\frac{1}{2c^{+}_{p}(d)}$ the depth of the network grows accordingly. If however $\frac{2c^{+}_{p}(d)\sigma}{\delta}$ is bounded from above by some universal constant, we obtain depth of $O\left(Nd\log_{2}\left(d\right)\right)$. For the special case where $p\in \{1, \infty\}$ the range of the width in \thmref{thm:upper bound memorization big k} can be improved and the log factor in the depth of the network can be removed:

\begin{theorem}\label{thm:upper bound memorization big k special cases}
    Let $p\in\{1,\infty\}$. If $d+4\leq k$ and
    $
        \frac{\sigma}{\delta} < \frac{1}{2c^{+}_{p}(d)}
    $,
    then for every $\delta$-separated dataset $\mathcal{D}\in\mathcal{D}_{d,N,C}(\delta)$, there exists a neural network $f:\mathbb{R}^d\rightarrow \mathbb{R}$ with width $k$ and depth $O\left(Nd\right)$ that $(\sigma, p)$-robustly memorizes the dataset $\mathcal{D}$.
\end{theorem}

\Cref{thm:upper bound memorization big k,thm:upper bound memorization big k special cases} do not depend on the support of the dataset, and for fixed ratio $\sigma/\delta$, the depth we obtain does not depend on $\delta$. Furthermore, we allow for robust neighborhoods under $l_{p}$ for any $p\in(0,\infty]$. Thus, our results extend the results in \cite[Theorem~2.2]{li2022robust} and \cite[Theorem~B.6]{yuoptimal}. We further extend \Cref{thm:upper bound memorization big k,thm:upper bound memorization big k special cases} to allow for any choice of both separation and robustness norms in \appref{subsec:large width l_q}.

\subsection{Robust Memorization With Small Width}

We now turn to study the more challenging case where the desired width is smaller than the data dimension, which is our main contribution. The theorem below shows that in this regime, it is still possible to $(\sigma,p)$-robustly memorize any $\delta$-separated dataset with a width $k$ network, provided that the radius of robustness $\sigma$ is small enough:

\begin{theorem}\label{thm:upper bound memorization}
    Suppose $7\leq k \leq d+5$ and
    \begin{equation*}    
        \frac{\sigma}{\delta} \leq a_{p,d}N^{-\frac{2}{k-6}}~,~~\text{where}~~ a_{p,d}:=\frac{1}{8\sqrt{e}}d^{-\frac{1}{2}+\left[\frac{1}{p}-\frac{1}{2}\right]_{-}}~.
    \end{equation*}
    Then for every $\delta$-separated dataset $\mathcal{D}\in\mathcal{D}_{d,N,C}(\delta)$, there exists a neural network $f:\mathbb{R}^d\rightarrow \mathbb{R}$ with width $k$ and depth $O\left(Nk\log_{2}\left(k\right)\right)$ that $(\sigma, p)$-robustly memorizes the dataset $\mathcal{D}$.
\end{theorem}

The amount by which $\frac{\sigma}{\delta}$ has to be small depends on the desired width $k$, input dimension $d$, the robustness metric $l_p$ and on the dataset size $N$. The bound on $\frac{\sigma}{\delta}$ in Theorems \ref{thm:upper bound memorization big k} did no depend on $N$, and so one can then ask if the dependence on $N$ in \thmref{thm:upper bound memorization} can be improved. The next theorem shows that any improvement of the bound will still have a similar dependence on $N$, and that a bound of the form $\frac{\sigma}{\delta} < CN^{-\frac{2}{k}}$, is a necessary requirement for the case of small width $k$:

\begin{theorem}\label{thm:lower bound memorization}
    Suppose $1\leq k \leq d-1$ and
    \begin{equation*}
        \frac{\sigma}{\delta} > b_{p,d}N^{-\frac{2}{k}}~,~~\text{where}~~b_{p,d}:=2416d^{\left[\frac{1}{p}-\frac{1}{2}\right]_{+}}~.
    \end{equation*}
    Then there exists a $\delta$-separated dataset $\mathcal{D}\in \mathcal{D}_{d,N,2}(\delta)$, such that every neural network $f:\mathbb{R}^d\rightarrow \mathbb{R}$ with width $k$ and any depth cannot $(\sigma, p)$-robustly memorize the dataset $\mathcal{D}$.
\end{theorem}

From \thmref{thm:upper bound memorization}, we get that if 
\begin{equation}\label{eq:k lower bound dep}
    6 + \frac{2}{\log\left(\frac{\delta}{\sigma}a_{p,d}\right)}\log(N) < k~,
\end{equation}
then every $\delta$-separated dataset of size $N$ can be $(\sigma, p)$-robustly memorized by a width $k$ neural network. On the other hand, from  \thmref{thm:lower bound memorization} we get that if 
\begin{equation}\label{eq:k upper bound dep}
    k < \frac{2}{\log\left(\frac{\delta}{\sigma}b_{p,d}\right)}\log(N)~,
\end{equation}
then, there exists a $\delta$-separated dataset of size $N$ that cannot be $(\sigma, p)$-robustly memorized by any width $k$ neural network. Hence from \Cref{eq:k lower bound dep,eq:k upper bound dep} we conclude the following corollary:

\begin{corollary}
    Let $p\in[2, \infty]$. There exists universal constants $C_{1}, C_{2}$ s.t. in the regime $k < d$,
    \begin{itemize}
        \item A width of $ k > C_{1}\frac{\log(N)}{\log\left(\frac{\delta}{\sigma}\right)+\log\left(\frac{1}{d}\right)}$ is sufficient for $(\sigma, p)$-robust memorization of every $\delta$-separated dataset of size $N$.
        \item A width of $ k > C_{2}\frac{\log(N)}{\log\left(\frac{\delta}{\sigma}\right)}$ is necessary for $(\sigma, p)$-robust memorization of every $\delta$-separated dataset of size $N$.
    \end{itemize}    
\end{corollary}

We thus see that indeed in order to perform robust memorization with robustness radius independent of $N$ and with width smaller than the data dimension $d$, a dependence logarithmic in $N$ is both a necessary and a sufficient condition for the width.

\begin{remark}[Fixed ratio $k/d$]\label{remark:fixed ratio k/d}
    In the proof of \thmref{thm:upper bound memorization} we are in fact proving a better bound of the form $\frac{\sigma}{\delta} \leq a_{p,d}\sqrt{k-6}\cdot N^{-\frac{2}{k-6}}$, which is of the order of $\sqrt{\frac{k}{d}}\cdot N^{-\frac{2}{k}}$ when $p=2$. Therefore, in  the regime where $k/d$ is fixed (and $p=2$), we get the following (for some constants $C_{1}, C_{2}$):
    \begin{itemize}
        \item A width of $ k > C_{1}\frac{\log(N)}{\log\left(\frac{\delta}{\sigma}\right)}$ is sufficient for $(\sigma, p)$-robust memorization of every $\delta$-separated dataset of size $N$.
        \item A width of $ k > C_{2}\frac{\log(N)}{\log\left(\frac{\delta}{\sigma}\right)}$ is necessary for $(\sigma, p)$-robust memorization of every $\delta$-separated dataset of size $N$.
    \end{itemize}
\end{remark}

As discussed in the introduction, \citep{yuoptimal} showed that for $p=\infty$, achieving optimal  robust memorization (i.e for every $\sigma < \frac{\delta}{2}$) is not possible when $k < d$. In contrast, \remarkref{remark:fixed ratio k/d} implies that even when $c_{1}d < k < d$, \emph{nearly}-optimal robust memorization is still possible, i.e. for every $\sigma < c_{2}\frac{\delta}{2}$ (for some universal constants $0<c_{1},c_{2}<1$).

\begin{remark}[Non-robust memorization]
    In the case of non-robust memorization, i.e when $\sigma=0$, we get from \thmref{thm:upper bound memorization} that memorization is possible with networks whose width is a universal constant (namely, $7$). This is consistent with previous results in \cite{park2021provable,vardi2021optimal} about non-robust memorization.  
\end{remark}

\Cref{thm:upper bound memorization,thm:lower bound memorization} can also be interpreted as results on the dependence between robustness radius and width. Fixing $\delta, d, p$ and $N$, we get bounds for the values of the radius $\sigma$ for which robust memorization is always possible, as a function of the desired width $k$ of the memorizing network. Both the upper bound from \thmref{thm:upper bound memorization} (green curve in \figref{fig:main results plot}) and the lower bound from \thmref{thm:lower bound memorization} (red curve in \figref{fig:main results plot}) are proportional to $N^{-\frac{2}{k}}$. The gap between them (gray stripes in \figref{fig:main results plot}) stems from the difference between the terms $a_{p,d}$ and $b_{p,d}$.

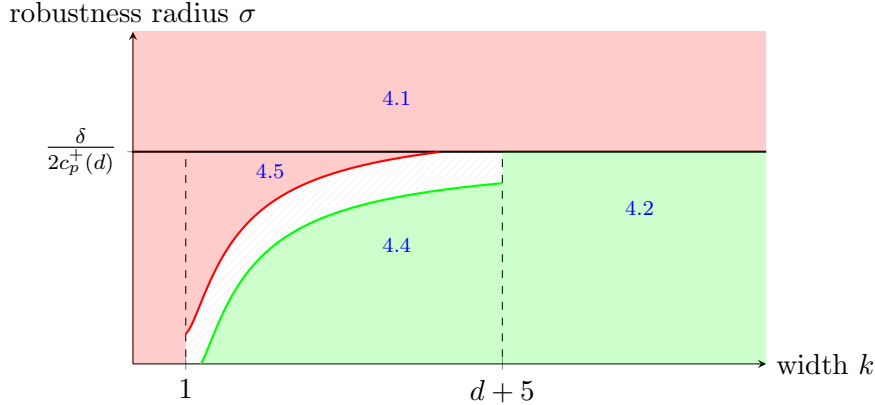
\begin{figure}
    \centering
    \begin{tikzpicture}
        \begin{axis}[
            xlabel={width $k$},
            ylabel={robustness radius $\sigma$},
            every axis x label/.style={at={(ticklabel* cs:1)}, anchor=west},
            every axis y label/.style={at={(ticklabel* cs:1)}, anchor=south},
            ymin=0, ymax=7.5,
            xmin=0, xmax=50,
            xtick={0, 30},
            xticklabels={$1$,$d+5$},
            ytick={5},
            yticklabels={$\frac{\delta}{2c^{+}_{p}(d)}$},
            domain=0:10,
            samples=100,
            width=10cm,
            height=6cm,
            axis lines=left,
            grid=none,
            enlargelimits=true,
            no markers=true
        ]
    
        \addplot[color=black, thick] coordinates {(-10,5) (55,5)};
        \addplot[dashed, color=black] coordinates {(30,-1) (30,5)};
        \addplot[dashed, color=black] coordinates {(0,-1) (0,5)};
        
        \addplot+[
            name path=lower_bound,
            domain=0:30,
            samples=100,
            thick,
            color=green
        ] {6*10^(-2/(x))-1};
    
        \addplot+[
            name path=upper_bound,
            domain=0:24,
            samples=100,
            thick,
            color=red
        ] {6*10^(-2/(x+1))};
    
        \addplot+[name path=bottom, domain=0:30, draw=none] {-1};
        \addplot[fill=green, opacity=0.2] fill between[of=lower_bound and bottom];
    
        \addplot+[name path=top, domain=0:30, draw=none] {5};
        \addplot[fill=red, opacity=0.2] fill between[of=top and upper_bound];
    
    
        \addplot[pattern=north east lines, opacity=0.2] fill between[of=upper_bound and lower_bound, soft clip={domain=0:24}];
    
        \addplot[pattern=north east lines, opacity=0.2] fill between[of=top and lower_bound, soft clip={domain=24:30}];
        
        \addplot[
            draw=none,
            fill=red,
            opacity=0.2,
            area legend,
            forget plot,
            shader=interp, 
            samples=2
        ] coordinates {(-10, 10) (55, 10) (55, 5) (-10, 5)} \closedcycle;
    
        \addplot[
            draw=none,
            fill=red,
            opacity=0.2,
            area legend,
            forget plot,
            shader=interp, 
            samples=2
        ] coordinates {(-10, 5) (0, 5) (0, -1) (-10, -1)} \closedcycle;

        \addplot[
            draw=none,
            fill=green,
            opacity=0.2,
            area legend,
            forget plot,
            shader=interp, 
            samples=2
        ] coordinates {(30, 5) (55, 5) (55, -1) (30, -1)} \closedcycle;
        
        \node at (axis cs:43,3) [anchor=south] {{\scriptsize \ref{thm:upper bound memorization big k}}};
        \node at (axis cs:20,2) [anchor=south] {{\scriptsize \ref{thm:upper bound memorization}}};
        \node at (axis cs:20,6) [anchor=south] {{\scriptsize \ref{remark:valid range for radius}}};
        \node at (axis cs:8,4) [anchor=south] {{\scriptsize \ref{thm:lower bound memorization}}};
    
        \end{axis}
    \end{tikzpicture}
    \caption{Illustration of main results describing regions where robust memorization is possible (green), not possible (red) and unknown (gray stripes). $k$ is the width, $\sigma$ the radius of robustness and $\delta$ the separation distance of the dataset of dimension $d$. \remarkref{remark:valid range for radius} and Theorems \ref{thm:upper bound memorization big k}, \ref{thm:upper bound memorization}, \ref{thm:lower bound memorization} are indicated in the regions that they discuss.}
    \label{fig:main results plot}
\end{figure}

\subsubsection{The Gap Between $a_{p,d}$ and $b_{p,d}$}\label{subsubsec:gap of main bounds}

We proceed to discuss the gap between the upper bound in \Cref{thm:upper bound memorization} and the lower bound in \Cref{thm:lower bound memorization}, and specifically the gap between the multiplicative factors $a_{p,d}$ and $b_{p,d}$.
Note that by definition of $c^{+}_{p}(d), c^{-}_{p}(d)$ we have $a_{p,d}=\frac{1}{8\sqrt{e}}d^{-\frac{1}{2}+\left[\frac{1}{p}-\frac{1}{2}\right]_{-}}=\frac{1}{2c^{+}_{p}(d)}\cdot \frac{1}{\sqrt{16ed}}$, and that $b_{p,d}=2416d^{\left[\frac{1}{p}-\frac{1}{2}\right]_{+}}=\frac{1}{2c^{-}_{p}(d)}\cdot 4832$. The $c^{+}_{p}(d)$ factor in $a_{p,d}$ comes from the fact that we have to ensure that the \emph{enclosing} $l_{2}$ balls of the $l_{p}$ neighborhoods are disjoint (as discussed in \remarkref{remark:valid range for radius}). On the other hand, the $c^{-}_{p}(d)$ factor in $b_{p,d}$ comes from the fact that we have to show that $l_{p}$ neighborhoods in the constructed dataset intersect, and we do that by showing that the \emph{inscribed} $l_{2}$ balls in them intersect. The gap between $a_{p,d}$ and $b_{p,d}$ is thus given by 
\begin{equation*}
    \frac{b_{p,d}}{a_{p,d}}=c\cdot\frac{c^{+}_{p}(d)}{c^{-}_{p}(d)}\sqrt{d}=c\cdot \sqrt{d}\cdot d^{\lvert \frac{1}{2}-\frac{1}{p} \rvert}
\end{equation*}
for $c=19328\sqrt{e}$. The need to reduce the dimension of the data with a linear map (see proof intuition in \subsecref{subsec:proof intuition}) introduces the dependence $d^{\lvert \frac{1}{2}-\frac{1}{p} \rvert}$ on $d$ in the gap between the bounds, for all $p\neq 2$. In the special case that $p=2$, this dependence vanishes and we are left with a gap of $\sqrt{d}$. The reason for this gap stems from the non-tightness of our robust variant of the Johnson-Lindenstrauss lemma, which will be presented in \secref{sec:pres lin map}.

\subsection{Proof Intuition}\label{subsec:proof intuition}

We now turn to provide a sketch for the proof of our main results. 

We begin by discussing the proof of \thmref{thm:upper bound memorization big k}, which shows that for width larger than the dimension and sufficient depth, one can robustly memorize any $\delta$-separated dataset for any applicable robustness radius. As discussed in \remarkref{remark:valid range for radius}, since the separation is measured in $l_{2}$ norm, the $l_{p}$ balls of radius $\sigma$ around data points (from different classes) must be contained in disjoint $l_{2}$ balls of radius $r=c^{+}_{p}(d)\sigma$. Hence, a function that assigns each of these $l_{2}$ balls with its appropriate label will $(\sigma, p)$-robustly memorize the data. Given a collection of labeled $l_{2}$ balls in $\mathbb{R}^d$ we can perform this assignment using a function that computes the weighted sum of ball indicators $\sum\limits_{i=1}^{N}y_{i}\cdot 1_{B^{d}_{2}(x_{i}, r)}$ over all $N$ data points.

Since exact computation of the $l_{2}$ norm is not possible with ReLU networks, we first approximate the function $y_{i}\cdot 1_{B^{d}_{2}(x_{i}, r)}$  using the function
\[ f_{i}(x) =  \begin{cases} 
      y_{i} & \lVert x - x_{i} \rVert_{2}\leq r \\
      v(x) & r < \lVert x - x_{i} \rVert_{2} \leq  r+ w\\
      0 &   r + w < \lVert x - x_{i} \rVert_{2} \\
   \end{cases}~,
    \]
where $v(x)$ is some value bounded by $y_{i}$, and $w=\delta - 2r$. We then approximate $f_i$ (and specifically $\lVert x - x_{i} \rVert_{2}$) using a ReLU network, by sequentially approximating for every $1\leq j \leq d$ the square of each coordinate of the vector $x-x_{i}$. The resulting network completes the proof of \thmref{thm:upper bound memorization big k}. In the case that $p\in \{1, \infty\}$ we can replace the approximation of $\lVert x - x_{i} \rVert_{2}$ with a ReLU network that computes exactly the norm $\lVert x - x_{i}\rVert_{p}$, removing the logarithmic factor in the depth and yielding \thmref{thm:upper bound memorization big k special cases}.

Performing the necessary computations for each of the coordinates $1\leq j \leq d$ as above means that the width of the resulting network is at least $d$ (computing sequentially, as we did, still requires the propagation of the input vector in each layer for future computations). Therefore, handling the regime where the width is smaller than $d$ must involve some dimensionality reduction of the dataset. Specifically, if the desired width is $k$, then the first layer of any memorizing network must implement a linear mapping that reduces the dimension of the dataset to at most $k$ dimensions (see \defref{def:Neural Network}). If there is any hope to robustly memorize the dataset, this map cannot introduce an intersection between the $(\sigma, p)$-neighborhoods of points from different classes: Indeed, let $T$ be a linear map of rank $k$ such that there exists $(x_{i},y_{i}), (x_{j},y_{j})$ with $y_{i}\neq y_{j}$ and $T(B^{d}_{p}(x_{i},\sigma))\cap T(B^{d}_{p}(x_{j},\sigma))\neq\emptyset$. Then any network whose first layer is $T$, cannot $(\sigma,p)$-robustly memorize the dataset. 

To ensure that in general $T(B^{d}_{p}(x_{i},\sigma))$, $T(B^{d}_{p}(x_{j},\sigma))$ are disjoint we have to require that $T(B^{d}_{2}(x_{i},r))$, $T(B^{d}_{2}(x_{j},r))$ are disjoint (recall that $r=c^{+}_{p}(d)\sigma$). The first step is thus to characterize the conditions that guarantee the ability or lack thereof to linearly map any $\delta$-separated dataset to $k$ dimensions while preserving separation and avoiding intersection of the images of the $l_{2}$ neighborhoods. In \secref{sec:pres lin map} we discuss the existence of such mappings. In the positive direction we find conditions that ensure the existence of such a map $T$. Normalizing $T$ appropriately, we obtain a map $T^{\prime}$ that shares the properties of $T$ and also satisfies that $T^{\prime}(B^{d}_{2}(x_{i},r))$ is an $l_{2}$ ball in $\mathbb{R}^{k}$. Composing this map $T^{\prime}$ with the network from \thmref{thm:upper bound memorization big k} (where now the dimension of the data is $k$) yields \thmref{thm:upper bound memorization}. In the negative direction, we find conditions that allow us to construct a $\delta$-separated dataset that no $T$ of rank $k$ can preserve, establishing the proof of \thmref{thm:lower bound memorization}.

\section{Preserving Linear Maps}\label{sec:pres lin map}

As discussed at the end of the previous section, a key tool that we need in order to establish robust memorization with small width is the existence of a linear transformation, which maps a given $\delta$-separated dataset into a lower-dimensional subspace, while preserving a separation between the neighborhoods of points. More formally, we define a \emph{preserving} linear map into $k$ dimensions (with respect to the $l_2$ norm) as follows:
\begin{definition}\label{def:preserving func}
     Let $\mathcal{D}\in\mathcal{D}_{d,N,C}(\delta)$ be a $\delta$-separated dataset, and let $\sigma < \frac{\delta}{2}$. We say that a linear function $T:\mathbb{R}^{d}\longrightarrow\mathbb{R}^{k}$ \textbf{$(\sigma,k)$-preserves} $\mathcal{D}$ if
    \begin{equation*}
        \lVert a - a^{\prime} \rVert_{2} \leq \lVert T(a) - T(a^{\prime}) \rVert_{2}
    \end{equation*}
    for every $(x,y),(x^{\prime},y^{\prime})\in \Dcal$ with $y\neq y^{\prime}$ and every $a\in B^{d}_{2}(x,\sigma), a^{\prime}\in B^{d}_{2}(x^{\prime},\sigma)$.
\end{definition}

Note that \defref{def:preserving func} requires only a lower bound on the norm of $T(a-a^{\prime})$, without an upper bound. This is different from the usual notions of approximate isometry (as used, for example, in the celebrated Johnson-Lindenstrauss lemma), and results from the fact that we are only interested in preventing unwanted intersections. As a result, the definition does not involve a scaling factor $\epsilon$ of the type $\epsilon\lVert a - a^{\prime} \rVert_{2} \leq \lVert T(a) - T(a^{\prime}) \rVert_{2}$ since one can always scale the linear map by $1/\epsilon$ to obtain a map that satisfies the above definition. We further discuss the connection between our results and the Johnson-Lindenstrauss lemma in \subsecref{subsec:JL}.

We are interested in the problem of determining the conditions under which such a linear map exists. Concretely, our problem can be formulated in the following manner:

\begin{quote}
    \textit{Let $\mathcal{D}\in\mathcal{D}_{d,N,C}(\delta)$ be a $\delta$-separated dataset. Under what conditions on $N, \delta, \sigma, d, k$ can we guarantee that there exists a linear map $T:\mathbb{R}^{d}\longrightarrow\mathbb{R}^{k}$ that $(\sigma, k)$-preserves $\mathcal{D}$?}
\end{quote}

\subsection{Conditions for the Existence of a Preserving Linear Map}

Following the discussion in \remarkref{remark:valid range for radius} we know that since we deal with $l_{2}$ neighborhoods of radius $\sigma$, the ratio $\frac{\sigma}{\delta}$ can only be considered in the range $[0,\frac{1}{2})$. Any $\sigma$ such that $\frac{1}{2} \leq \frac{\sigma}{\delta}$ would cause intersecting $l_{2}$ neighborhoods, and thus any linear transformation will not be preserving as defined above. On the other hand, when $\frac{\sigma}{\delta}=0$, any finite dataset has a $(0,k)$-preserving map for every $k$ (since then it suffices that distinct data points have distinct images, which is always possible). Thus, the question is which values of $\frac{\sigma}{\delta}\in (0,\frac{1}{2})$ allow for the existence of $(\sigma,k)$-preserving linear maps. 
The main result of this section is the following theorem, which provides an almost tight characterization:

\begin{theorem}\label{thm:main pres}
    Let $N,d,C\in \mathbb{N}_{\geq 2}$ such that $1\leq k \leq d-1$ and let $\sigma < \frac{\delta}{2}$. There exists universal constants $C_{1},C_{2}$ such that
    \begin{enumerate}
        \item If $\frac{\sigma}{\delta} < C_{1}\sqrt{\frac{k}{d}}N^{-\frac{2}{k}}$ then every $\delta$-separated dataset $\mathcal{D}\in\mathcal{D}_{d,N,C}(\delta)$, has a $(\sigma,k)$-preserving linear map.
        \item If $\frac{\sigma}{\delta} > C_{2}N^{-\frac{2}{k}}$ then there exists a $\delta$-separated dataset $\mathcal{D}\in\mathcal{D}_{d,N,2}(\delta)$, for which no $(\sigma,k)$-preserving linear map exists. 
    \end{enumerate}
\end{theorem}

We present here a brief sketch of the main proof ideas. The full proof follows immediately from \Cref{thm:ortho pres pos,thm:gen pres neg} in \appref{subsubsec:pres thms}.

We first discuss the positive result (in item 1 of \thmref{thm:main pres}). Given a $\delta$-separated dataset $\mathcal{D}$, we consider the collection of normalized differences:
\begin{equation*}
    S=\left\{\frac{a-a^{\prime}}{\lVert a-a^{\prime} \rVert_{2}} \mid a\in B^{d}_{2}(x,\sigma), a^{\prime}\in B^{d}_{2}(x^{\prime},\sigma) \text{ s.t } (x,y),(x^{\prime},y^{\prime})\in \Dcal \text{ with } y\neq y^{\prime}\right\}~.
\end{equation*}
Note that by definition, a linear map $T$ will $(\sigma,k)$-preserve $\mathcal{D}$ if and only if $1\leq \lVert Ts\rVert_{2}$ for every $s\in S$. We then show using a probabilistic argument that if $\frac{\sigma}{\delta}$ is small enough, there exists some orthogonal projection matrix $P$ of rank $k$ such that $\epsilon\leq \lVert Ps\rVert_{2}$ for every $s\in S$, where $\epsilon$ is some value in $[0, \frac{1}{2}]$. Taking $T=\frac{1}{\epsilon}P$ proves item 1 of \thmref{thm:main pres}.

\begin{figure}
\vspace{-3em}
    \centering
    
    \begin{subfigure}[m]{0.5\textwidth}
        \centering
        
       \begin{tikzpicture}[scale=0.8]
        \def\fact{0.77}
        
        \draw (0,0) circle (\fact*4cm);
        
        \draw (0,0) circle (\fact*1cm);
        
        \draw (0,0) -- (0,-\fact*1) node[midway, left] {$r$};
        \filldraw (0,0) circle (1pt);
        \draw (-\fact*1,0) -- (-\fact*4,0) node[midway, below] {$\delta$};
        \draw[thick] (-45:\fact*1cm) -- ++(\fact*0.25,-\fact*0.25) node[below, below] {$\sigma$};
        \foreach \i in {0,1,...,6} {
            \draw[dashed] (-45+\i*34:\fact*1cm) circle (\fact*0.35cm);
            \filldraw[red] (-45+\i*34:\fact*1cm) circle (2.5pt);
            
        }
        
        \draw[thick] (-45:\fact*4cm) -- ++(\fact*0.7,-\fact*0.7) node[midway, below] {$r$};
        \foreach \i in {0,1,...,6} {
            \draw[dashed] (-45+\i*28:\fact*4cm) circle (\fact*1cm);
            \filldraw[blue] (-45+\i*28:\fact*4cm) circle (2.5pt);
            
        }
        \end{tikzpicture}
        \caption{\parbox[t]{6cm}{An illustration of the dataset. Points from the first class are colored red, and points from the second class are colored blue.}}\label{fig:two sphere covers dataset}
    \end{subfigure}%
    ~
    \begin{subfigure}[m]{0.5\textwidth}
        \centering
    
        \begin{tikzpicture}[scale=0.8]
        \def\fact{0.65}
    
        \draw (0,0) circle (\fact*4cm);
        \draw (0,0) -- (0,-\fact*4) node[pos=0.3, left] {$r+\delta$};
        \filldraw (0,0) circle (1pt);
        
        \foreach \i in {0,1,...,12} {
            \coordinate (Ci) at (-45+\i*28:\fact*4cm);
         
            \foreach \j in {0,1,...,10} {
                \coordinate (Cij) at ($(Ci) + (-45+\j*34:\fact*1cm)$);
                
                \draw[red] (Cij) circle (\fact*0.35cm);
                
                \filldraw (Cij) circle (1pt);
            }
        }
        \end{tikzpicture}
        \caption{\parbox[t]{6cm}{The set of points $x^{\prime}_{j} - x_{i}$ and their $\sigma$-neighborhoods colored red.}}
        \label{fig:neighborhoods of differences}
    \end{subfigure}%
 
  \caption{A dataset that cannot be $(\sigma, k)$-preserved.}
  \label{fig:two sphere dataset}
\end{figure}
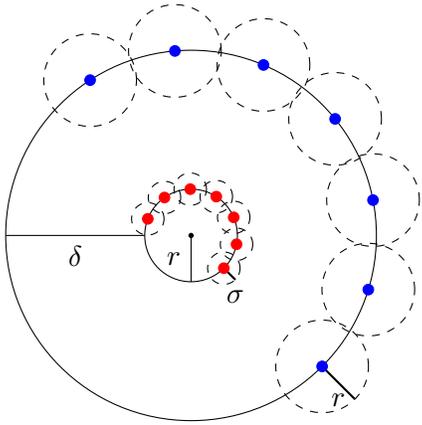
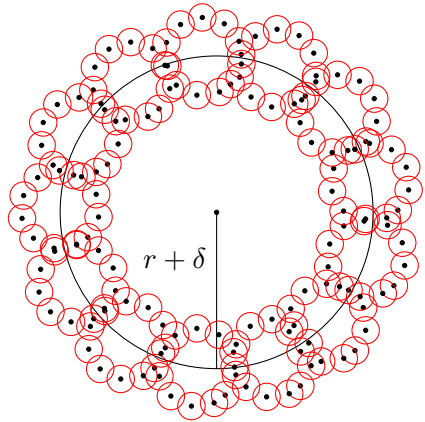

For the negative result in item 2 of \thmref{thm:main pres}, we construct a $\delta$-separated dataset that cannot be $(\sigma, k)$-preserved. To do so, consider some origin-centered $(k+1)$-dimensional ball embedded in $\reals^d$, which we will denote as $\tilde{B}^{k+1}_{2}(0,r)$ (for some $r>0$). When $\frac{\sigma}{\delta}$ is big enough, there are enough points to construct the following $\delta$-separated dataset, containing 2 classes: One class, $\{x_{1}, ..., x_{N/2}\}$, will be the centers of a $\sigma$-cover of the boundary $\partial\tilde{B}^{k+1}_{2}(0,r)$ (red points in \figref{fig:two sphere covers dataset}), and the other class $\{x^{\prime}_{1}, ..., x^{\prime}_{N/2}\}$ will be the centers of an $r$-cover of the boundary of a larger $(k+1)$-dimensional embedded ball $\tilde{B}^{k+1}_{2}(0,r+\delta)$ (blue points in \figref{fig:two sphere covers dataset}). These two classes comprise together a $\delta$-separated dataset since $\delta \leq \lVert x_{i}-x^{\prime}_{j}\rVert_{2}$ for any $i,j$. Now, define the collection of $\sigma$-neighborhoods (see \figref{fig:neighborhoods of differences})
\begin{equation*}
    U=\left\{ a^{\prime}_{j} - a_{i} \mid a^{\prime}_{j}\in B^{d}_{2}(x^{\prime}_{j},0), a_{i}\in B^{d}_{2}(x_{i},\sigma) \right\}~.
\end{equation*}
By the construction of $U$, it follows that for every point in the boundary $x\in \partial \tilde{B}^{k+1}_{2}(0,r+\delta)$, there exists some $u\in U$ such that $u\in \text{Span}\{x\}$. On the other hand, since $\partial \tilde{B}^{k+1}_{2}(0,r+\delta)$ is the boundary of a $(k+1)$-dimensional ball embedded in $\reals^d$, it follows from dimensionality considerations that any $d-k$ dimensional subspace $K$ must intersect some point in the ball's boundary. Taken together, we obtain that for every $d-k$ dimensional subspace $K$, there exists some $u\in U\cap K$. Let $T$ be a linear map of rank $k$, and denote $K=\text{Ker}T$. We conclude from the above that there exists some $u=a^{\prime}_{j}-a_{i}$ such that $u\in K$ and so by definition $T$ does not $(\sigma, k)$-preserve $\mathcal{D}$.

As explained in the proof intuition of our main results (\subsecref{subsec:proof intuition}), Theorems \ref{thm:upper bound memorization} and \ref{thm:lower bound memorization} are proved by using the existence or lack-thereof of a preserving map $T$. The $\sqrt{\frac{k}{d}}$ gap between item 1 and 2 of \thmref{thm:main pres} is thus the reason for the $\sqrt{d}$ gap between \thmref{thm:upper bound memorization} and \thmref{thm:lower bound memorization} that is mentioned in \subsecref{subsubsec:gap of main bounds}. We refer the reader to \appref{sec:tighter bounds} for further discussion of this $\sqrt{\frac{k}{d}}$ gap.

\subsection{Comparison to the Johnson–Lindenstrauss Lemma}\label{subsec:JL}

The problem of finding a preserving linear map as in \thmref{thm:main pres} bears similarities to the problem addressed by the Johnson–Lindenstrauss (JL) lemma (see \cite{dasgupta2003elementary}). Informally, the JL lemma states that a high-dimensional datasets can be embedded into a subspace of much lower dimension while approximately preserving distances. Formally:

\begin{theorem}[JL lemma]\label{thm:JL}
    Let $X\subset \mathbb{R}^{d}$ with $|X| = N$, and let $0 < \epsilon < 1$. If $\frac{C_{JL}\ln (N)}{\epsilon^2}< k$ (where $C_{JL}$ is some universal constant), then there exists a linear map $T:\mathbb{R}^{d}\longrightarrow\mathbb{R}^{k}$ such that for every $x,x^{\prime}\in X$
    \begin{equation}\label{eq:JL}
        (\sqrt{1-\epsilon}) \lVert x - x^{\prime} \rVert_{2} \leq \lVert T(x) - T(x^{\prime}) \rVert_{2} \leq (\sqrt{1+\epsilon}) \lVert x - x^{\prime} \rVert_{2}~.
    \end{equation}
    A map that satisfies \eqref{eq:JL} is called a \emph{JL map}.
\end{theorem}

The JL lemma provides conditions for a $(\sigma, k)$-preserving map of any $\delta$-separated dataset in the case where $\sigma=0$. However, satisfying the JL map condition alone (as defined in the theorem) is not enough to ensure $(\sigma, k)$-preservability for general $0\leq \sigma < \frac{\delta}{2}$. This is because a JL map provides approximate isometry for the data points themselves but can still stretch the space in a way that turns their $\sigma$-neighborhoods into hyper-ellipsoids that intersect. Indeed, below is a simple example of a dataset where for every $\sigma > 0$ there exists a JL map that is not $(\sigma, k)$-preserving (see \figref{fig:JL map not sigma pres} for an illustration):

\begin{example}\label{ex:JL not sigma preserving}
    Let $x=0, x^{\prime}=e_{1}\in\mathbb{R}^{d}$, $\Dcal=\{(x,0), (x^{\prime},1)\}$. Note that $\mathcal{D}$ is a $1$-separated dataset and let $0 < \sigma < 1/2$. Let $0 < \epsilon < 1$ and $\frac{C_{JL}\ln(2)}{\epsilon^2} < k \leq d-1$. Define $T:\mathbb{R}^{d}\longrightarrow\mathbb{R}^{k}$ by 
    \begin{equation*}
        T\left(\sum\limits^{d}_{i=1}\alpha_{i}e_{i}\right) = \left(\alpha_{1} - \frac{1}{\sigma}\alpha_{2}\right)e_{1}+\sum\limits^{k}_{i=2}\alpha_{i+1}e_{i}~.
    \end{equation*}
    Then, $\lVert T(x^{\prime}) - T(x) \rVert_{2} = \lVert e_{1} \rVert_{2} = \lVert x^{\prime} \rVert_{2} =  \lVert x^{\prime} - x \rVert_{2}$ and so $T$ is a JL map. However, $T$ does not $(\sigma, k)$-preserve $\Dcal$: Taking $a=-\sigma e_{2}$ and $a^{\prime}=x^{\prime}$, we have $a\in B^{d}_{2}(x, \sigma)$ ,  $a^{\prime}\in B^{d}_{2}(x^{\prime}, \sigma)$, yet $T(a^{\prime})-T(a)=T(a^{\prime}-a)=T\left(e_{1}+\sigma e_{2}\right)=(1-\frac{1}{\sigma}\sigma)e_{1} = 0$.
\end{example}

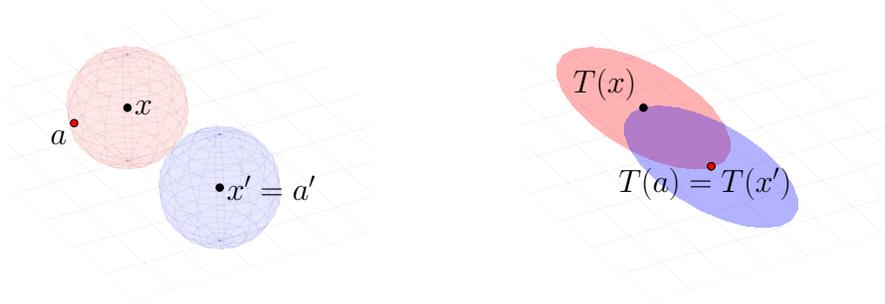
\begin{figure}[H]
    \centering
    \begin{subfigure}[m]{0.4\textwidth}
        \centering
        \begin{tikzpicture}[scale=0.7]
            
            \def\r{0.5}
            \begin{axis}[
                scale only axis,
                enlargelimits=false,
                height=6cm,
                view={60}{30},              
                axis equal,                 
                axis lines=middle,          
                hide x axis,               
                hide y axis,               
                hide z axis,               
                ticks=none,                 
                xmin=0, xmax=2,        
                ymin=-0.5, ymax=0.5,        
                zmin=-0.5, zmax=0.75,        
            ]
        
            \addplot3[
                surf,
                mesh,                        
                opacity=0.05,                 
                colormap={grid color}{[1pt] rgb255=(128,128,128) rgb255=(128,128,128)}, 
                samples=10,                  
                samples y=10,
                domain=-1.75:2.75,
                y domain=-1.5:1.75,
            ]
            (
                {x},                         
                {y},                         
                {0}                          
            );
        
            \addplot3[
                surf,
                samples=15,
                domain=0:360,               
                y domain=0:180,             
                opacity=0.05,                
                colormap={solid color}{[1pt] rgb255=(255,0,0) rgb255=(255,0,0)},
            ]
            (
                {\r*cos(x)*sin(y)},            
                {\r*sin(x)*sin(y)},            
                {\r*cos(y)}                    
            );
            
            \addplot3[
                surf,
                samples=15,
                domain=0:360,               
                y domain=0:180,             
                opacity=0.05,                
                colormap={solid color}{[1pt] rgb255=(0,0,255) rgb255=(0,0,255)},
            ]
            (
                {\r*cos(x)*sin(y)}+1.5,            
                {\r*sin(x)*sin(y)},            
                {\r*cos(y)}                    
            );

            \addplot3[
                only marks,                 
                mark=*,
                mark options={fill=black},
                mark size=2pt             
            ] coordinates {
                (0, 0, 0)                   
                (1.5, 0, 0)                   
            };
        
            \addplot3[
                only marks,                 
                mark=*,
                mark options={fill=red},
                mark size=2pt             
            ] coordinates {
                (0, -\r, 0)                  
            };
        
            \node at (axis cs: 0,0,0) [anchor=west, font=\LARGE] {$x$};
            \node at (axis cs: 0,-\r,0) [anchor=north east, font=\LARGE] {$a$};
            \node at (axis cs: 1.5,0,0) [anchor=west, font=\LARGE] {$x^{\prime}=a^{\prime}$};
            
            \end{axis}
            
        \end{tikzpicture}    
    \end{subfigure}%
    ~
    \begin{subfigure}[m]{0.4\textwidth}
        \centering
        \begin{tikzpicture}[scale=0.7]
            \def\r{0.5}
            \begin{axis}[
                scale only axis,
                enlargelimits=false,
                height=6cm,
                view={60}{30},              
                axis equal,                 
                axis lines=middle,          
                hide x axis,               
                hide y axis,               
                hide z axis,               
                ticks=none,                 
                xmin=0, xmax=2,        
                ymin=-0.5, ymax=0.5,        
                zmin=-0.5, zmax=0.75,        
            ]
        
            \addplot3[
                surf,
                mesh,                        
                opacity=0.05,                 
                colormap={grid color}{[1pt] rgb255=(128,128,128) rgb255=(128,128,128)}, 
                samples=10,                  
                samples y=10,
                domain=-1.75:2.75,
                y domain=-1.5:1.75,
            ]
            (
                {x},                         
                {y},                         
                {0}                          
            );
            
            \addplot3[
                surf,
                shader=interp,              
                samples=13,
                domain=0:360,               
                y domain=0:180,             
                opacity=0.3,                
                colormap={solid color}{[1pt] rgb255=(255,0,0) rgb255=(255,0,0)},
            ]
            (
                {\r*cos(x)*sin(y)-sin(x)*sin(y)},            
                {\r*cos(y)},            
                {0}                    
            );
        
            \addplot3[
                surf,
                shader=interp,              
                samples=13,
                domain=0:360,               
                y domain=0:180,             
                opacity=0.3,                
                colormap={solid color}{[1pt] rgb255=(0,0,255) rgb255=(0,0,255)},
            ]
            (
                {\r*cos(x)*sin(y)+1.1-sin(x)*sin(y)},            
                {\r*cos(y)},            
                {0}                    
            );
        
            \addplot3[
                only marks,                 
                mark=*,
                mark options={fill=black},
                mark size=2pt             
            ] coordinates {
                (0, 0, 0)                   
                (1.1, 0, 0)                   
            };
        
            \addplot3[
                only marks,                 
                mark=*,
                mark options={fill=red},
                mark size=2pt             
            ] coordinates {
                (1.1, 0, 0)                   
            };
        
            Labels for points
            \node at (axis cs: 0,0,0) [anchor=south east, font=\LARGE] {$T(x)$};
            \node at (axis cs: 1,0,0) [anchor=north, font=\LARGE] {$T(a)=T(x^{\prime})$};
            \end{axis}
        \end{tikzpicture}
    \end{subfigure}%
  \caption{The dataset $\mathcal{D}$ before (left) and after (right) applying $T$. Distance between the data points $x, x^{\prime}$ is preserved but the images of their $\sigma$-neighborhoods intersect.}
  \label{fig:JL map not sigma pres}
\end{figure}

\exref{ex:JL not sigma preserving} shows that simply using the standard formulation of the JL lemma is not enough to obtain a $(\sigma, k)$-preserving map. This is because, a priori, the JL map obtained from the JL lemma does not preserve neighborhoods. Our proof for item 1 of \thmref{thm:main pres} is a variant of the proof of the JL lemma, that enables us to guarantee also a bound on the distortion of neighborhoods.

\section{Conclusion}\label{sec:conclusion}

In this paper, we showed that for the task of robust memorization of datasets of size $N$, there exists a trade-off between the radius of robustness $\sigma$ and the width $k$ of the memorizing network. We showed that in the regime where the width is less than the data dimension, robust memorization can only be done with a robustness radius of $N^{-\frac{2}{k}}$ (up to a constant), and in particular achieving the optimal robust memorization capacity (up to a constant) can only be done with width $k$ logarithmic in $N$. This is in contrast to the non-robust case where constant width is sufficient (see \cite{vardi2021optimal}).

To obtain our bounds, we develop a robust variant of the Johnson-Lindenstrauss lemma (\thmref{thm:main pres}) with almost tight bounds. An interesting question for future work is the tightness of our bounds in \thmref{thm:main pres}. Namely, establishing whether the radius of robustness has to depend on the data dimension $d$ in order to guarantee the ability to linearly reduce the dataset to $k$ dimensions in a robust manner. Answering this question would provide a complete and full characterization of the relation between width and robustness radius under the $l_{2}$ norm. 

Another possible direction for future work is to study the relation between robustness and number of parameters. In our work we focused mostly on the width, and for a constant width our construction requires $O(N)$ parameters. In the non-robust case, $\tilde{O}(\sqrt{N})$ parameters are sufficient and necessary up to logarithmic terms (see \cite{vardi2021optimal}), motivating the question of what is the optimal number of parameters for \emph{robust} memorization. 

Finally, it would be interesting to study the extent to which trained neural networks (using standard optimization methods) can robustly memorize datasets, and in particular whether a network width logarithmic with $N$ is still sufficient, as in our existence results. 

\subsection*{Acknowledgments}
This research is supported in part by European Research Council (ERC) grant 754705, and by the Israeli Council for Higher Education (CHE) via the Weizmann Data Science Research Center.

\bibliographystyle{abbrvnat}
\bibliography{ref}

\newpage

\appendix

\addcontentsline{toc}{section}{Appendix} 
\part{Appendix} 
\parttoc 

\section{Additional Notations}\label{sec:additional notations}

For $1\leq k \leq d \in \mathbb{N}$ we will denote by $\text{Gr}_{d,k}$ the real Grassmannian manifold which is the set of all $k$-dimensional vector subspaces of $\mathbb{R}^{d}$. For any vector subspace $W$ there exists a unique orthogonal projection matrix onto it, that is a matrix $P\in M_{d}\left(\mathbb{R}\right)$ with $P=P^{2}=P^{\top}$ and $\text{Im}\left(P\right)=W$. We will denote this matrix by $P_{W}$. We can thus think of $\text{Gr}_{d,k}$ as the set of such projections. More generally, we will denote $\text{End}_{d,k}=\{M\in M_{d}\left(\mathbb{R}\right) \mid \text{rk}(M)=k\}$ for the set of rank $k$ matrices, and note that $\text{Gr}_{d,k}=\{P\in\text{End}_{d,k} \mid P=P^{2}=P^{\top} \}$.

Denote by $\mathbb{S}^{d-1} = \left\{x\in \mathbb{R}^{d} \mid \lVert x \rVert_{2} = 1 \right\}$ the $\left(d-1\right)$-dimensional unit sphere in $\mathbb{R}^{d}$. One can equip the sphere with a metric structure using the geodesic metric $\text{dist}_{\text{arc}}$ given by the angle between two points,  $\text{dist}_{\text{arc}}\left(a, b\right) = \arccos \langle a,b \rangle$. We denote by $B^{d-1}_{\text{arc}}\left(x, \varphi \right)=\left\{x^{\prime}\in \mathbb{S}^{d-1} \mid \text{dist}_{\text{arc}}(x^{\prime},x)\leq \varphi \right\}$ the metric ball in this metric space, sometimes called a geodesic ball, cap or spherical cap.

We will denote by $\nu_{d}$ the unique Haar probability measure of the orthogonal group $O(d)$ and by $\mu_{d-1}, \gamma_{d,k}$ the unique $O(d)$-invariant probability measures of $\mathbb{S}^{d-1}$ and $\text{Gr}_{d,k}$ respectively. For the formal definitions see \subsecref{subsection: inv measures}.

\section{Proofs of the Main Results}\label{sec:proof of main results}

\subsection{Robust Memorization With Large Width}

\begin{proof}[Proof of \thmref{thm:upper bound memorization big k}]
    Denote $r=c^{+}_{p}(d)\sigma$ and let $\mathcal{D}\in\mathcal{D}_{d,N,C}(\delta)$ be a $\delta$-separated dataset (separation under $l_{2}$ norm). We know that $r<\delta / 2$ and so there exists some $0<w$ such that $\delta = 2r+ w$. Denote $\Tilde{\epsilon}=\frac{w^2}{4d(w+2r)^2}=\frac{(\delta - 2r)^2}{4d\delta^2}$ then $0<\Tilde{\epsilon}<1/2$ and so from \lemref{lma:net square} we get that there exists a neural network $g_{\Tilde{\epsilon},2}:\mathbb{R}\rightarrow\mathbb{R}$ with width $3$ and depth $O\left(\log_{2}(\Tilde{\epsilon}^{-1})\right)=O\left(\log_{2}\left(\frac{d\delta}{\delta-2r}\right)\right)$ such that $|g_{\Tilde{\epsilon},2}(\alpha)-\alpha^{2}| \leq \epsilon$ for every $\alpha\in [0,1]$. Using this network, from \lemref{lma:net ball indicator} we get that for every $(x_{i},y_{i})\in\mathcal{D}$ there exists a neural network $f_{x_{i},w,2}:\mathbb{R}^{d}\rightarrow\mathbb{R}$ with width $W=d+2+\mathcal{W}(g_{\Tilde{\epsilon},2})=d+5$, and depth $L=O\left(d\mathcal{L}(g_{\Tilde{\epsilon},2})\right)=O\left(d\log_{2}\left(\frac{d\delta}{\delta-2r}\right)\right)$, such that for all $x \in\mathbb{R}^{d}$ we have $f_{x_{i},w,2}(x)\leq y_{i}$ and     
    \[ f_{x_{i},w,2}(x)=  \begin{cases} 
      y_{i} & \lVert x - x_{i} \rVert_{2}\leq r \\
      0 &   r+w\leq \lVert x - x_{i} \rVert_{2} \\
   \end{cases}~.
    \]
    Finally, because $\mathcal{D}$ is $\delta$-separated (under the $l_{2}$ norm), from \thmref{thm:full width memo} there exists a neural network $F_{d,\delta,r, 2}:\mathbb{R}^{d}\rightarrow\mathbb{R}$ with width $d+6$ and depth $O\left(Nd\log_{2}\left(\frac{d\delta}{\delta - 2r}\right)\right)$ that $(r,2)$-robustly memorizes the dataset $\mathcal{D}$. Define $f=F_{d,\delta,r,2}$ and let $x\in B^{d}_{p}(x_{i},\sigma)$. Then, by definition of $c^{+}_{p}(d)$ and \lemref{lma:lp ball and lq ball}, we have $x\in B^{d}_{2}(x_{i},r)$ so $f(x)=F_{d,\delta,r, 2}(x)=y_{i}$ and hence $f$ indeed $(\sigma, p)$-robustly memorizes $\mathcal{D}$. Now $d+6\leq k$ and so by padding each hidden layer of $f$ with $k-(d+6)$ neurons we obtain $f$ with width $k$ and depth 
    \begin{equation*}
    \begin{aligned}
        &&O\left(Nd\log_{2}\left(\frac{d\delta}{\delta - 2r}\right)\right)&=O\left(Nd\log_{2}\left(\frac{d\delta}{\delta - 2c^{+}_{p}(d)\sigma}\right)\right)\\
        &&&=O\left(Nd\log_{2}\left(d\left(1 - \frac{2c^{+}_{p}(d)\sigma}{\delta}\right)^{-1}\right)\right)
    \end{aligned}
    \end{equation*}
\end{proof}

\begin{proof}[Proof of \thmref{thm:upper bound memorization big k special cases}]
    We prove for $p=1$ and $p=\infty$:
    \begin{itemize}
        \item Case $p=1$:
        Follow the proof of \thmref{thm:upper bound memorization big k} where instead of \lemref{lma:net square} use the identity map $g_{1}(\alpha)=\alpha$ (with width and depth of $1$) to obtain from \lemref{lma:net ball indicator} indicators $f_{x_{i},w,1}$. Because $\mathcal{D}$ is $\delta$-separated (under the $l_{2}$ norm) it satisfies $\delta \leq \lVert x_{i}-x_{j}\rVert_{1}$ for every $x_{i},x_{j}$ with $y_{i}\neq y_{j}$ and so we can use \thmref{thm:full width memo} with the $f_{x_{i},w,1}$'s we have in order to get a neural network $F_{d,\delta,\sigma, 1}:\mathbb{R}^{d}\rightarrow\mathbb{R}$ with width $d+4$ and depth $O\left(Nd\right)$ that $(\sigma,1)$-robustly memorizes the dataset $\mathcal{D}$. Define $f=F_{d,\delta,\sigma,1}$ then $f(x)=F_{d,\delta,\sigma, 1}(x)=y_{i}$. Now $d+4\leq k$ and so by padding each hidden layer of $f$ with $k-(d+4)$ neurons we obtain $f$ with width $k$ and depth $O\left(Nd\right)$.
        \item Case $p=\infty$:
        Denote $\tau = \delta / c^{+}_{p}(d)$ and let $\mathcal{D}\in\mathcal{D}_{d,N,C}(\delta)$ be a $\delta$-separated dataset. From \lemref{lma:net ball indicator p=infty} we get that for every $(x_{i},y_{i})\in\mathcal{D}$ there exists a neural network $f_{x_{i},w,\infty}:\mathbb{R}^{d}\rightarrow\mathbb{R}$ with width $W=d+3$, and depth $L=O(d)$, such that for all $x \in\mathbb{R}^{d}$ we have $f_{x_{i},w,\infty}(x)\leq y_{i}$ and     
        \[ f_{x_{i},w,\infty}(x)=  \begin{cases} 
          y_{i} & \lVert x - x_{i} \rVert_{\infty}\leq \sigma \\
          0 &   \sigma+w\leq \lVert x - x_{i} \rVert_{\infty} \\
       \end{cases}~,
        \]
        for every $0<w$. Finally, because $\mathcal{D}$ is $\delta$-separated (under the $l_{2}$ norm) we get from the definition of $\tau$ and from \lemref{lma:lp lq relations} that $\tau \leq \lVert x_{i}-x_{j}\rVert_{\infty}$ for every $x_{i},x_{j}$ with $y_{i}\neq y_{j}$ and so from \thmref{thm:full width memo} there exists a neural network $F_{d,\tau,\sigma, \infty}:\mathbb{R}^{d}\rightarrow\mathbb{R}$ with width $d+4$ and depth $O(Nd)$ that $(\sigma,\infty)$-robustly memorizes the dataset $\mathcal{D}$. Define $f=F_{d,\tau,\sigma,\infty}$ and let $x\in B^{d}_{\infty}(x_{i},\sigma)$ then $f(x)=F_{d,\tau,\sigma, \infty}(x)=y_{i}$. Now $d+4\leq k$ and so by padding each hidden layer of $f$ with $k-(d+4)$ neurons we obtain $f$ with width $k$ and depth $O\left(Nd\right)$.
    \end{itemize}

\end{proof}

\subsection{Robust Memorization With Small Width}

\subsubsection{Robust Memorization and Preservability}
\begin{definition}\label{def:preservable data}
    Let $\mathcal{D}\in\mathcal{D}_{d,N,C}(\delta)$ be a $\delta$-separated dataset, and let $\sigma < \frac{\delta}{2}$. We say that $\mathcal{D}$ is \textbf{$(\sigma, k)$-preservable} if there exists some $M\in\text{End}_{d,k}$ that $(\sigma, k)$-preserves $\mathcal{D}$. If, furthermore, $M=\frac{1}{\epsilon}P$ for some $P\in\text{Gr}_{d,k}$ 
    we say that $\mathcal{D}$ is \textbf{$(\sigma, \epsilon, k)$-orthogonally preservable}.
\end{definition}

Robust memorization with width smaller than the data dimension is possible only when the data is preservable as the next two theorems show (see \subsecref{subsection: pres mem pos neg} for proofs):

\begin{theorem}\label{thm:memorization of preservable data pos}
    Let $\mathcal{D}\in\mathcal{D}_{d,N,C}(\delta)$ be a $\delta$-separated dataset such that $\mathcal{D}$ is $(\sigma^{\prime}, k)$-preservable under a map $M$ with $\lVert M \rVert_{2} < \frac{\delta}{2\sigma^{\prime}}$, where $\sigma^{\prime}=c^{+}_{p}(d)\sigma$. Then there exists a neural network $f:\mathbb{R}^d\rightarrow \mathbb{R}$ with width $k+6$ and depth $O\left(Nk\log_{2}\left(\frac{k}{1-\lVert M \rVert_{2}\frac{2\sigma^{\prime}}{\delta}}\right)\right)$ that $(\sigma, p)$-robustly memorizes the dataset $\mathcal{D}$.
\end{theorem}

\begin{theorem}\label{thm:memorization of preservable data neg}
    Let $\mathcal{D}\in\mathcal{D}_{d,N,C}( \delta)$ be a $\delta$-separated dataset such that $\mathcal{D}$ is not $(\sigma^{\prime}, k)$-preservable, where $\sigma^{\prime}=c^{-}_{p}(d)\sigma$. Then every neural network $f$ with width $\leq k$ cannot $(\sigma, p)$-robustly memorize the dataset $\mathcal{D}$. 
\end{theorem}

\subsubsection{Characterization of Preservability}\label{subsubsec:pres thms}

\Cref{thm:memorization of preservable data pos,thm:memorization of preservable data neg} highlight the connection between the preservability of a dataset in dimension $k$, and the ability to robustly memorize it with a network of width $k$. Therefore, we look for criteria to ensure $(\sigma, k)$-preservability for some general $\sigma,\delta$ and $k$.

\begin{theorem}\label{thm:ortho pres pos}
    If $\frac{2\sigma}{\delta} \leq \frac{1}{2}\sqrt{\frac{k}{de}}N^{-\frac{2}{k}}$, then every $\delta$-separated dataset $\mathcal{D}\in\mathcal{D}_{d,N,C}(\delta)$ is $(\sigma,\epsilon, k)$-orthogonally preservable with $\epsilon=\frac{1}{2}\sqrt{\frac{k}{de}}N^{-\frac{2}{k}}$.
\end{theorem}

The proof can be found in \subsecref{subsection:ortho pres pos}

\begin{theorem}\label{thm:gen pres neg}
    If $\frac{2\sigma}{\delta} > 4832 N^{-\frac{2}{k}}$ then there exists a $\delta$-separated dataset $\mathcal{D}\in\mathcal{D}_{d,N,2}(\delta)$ which is not $(\sigma, k)$-preservable.
\end{theorem}

The proof can be found in \subsecref{subsection:gen pres neg}

\subsubsection{Proofs of \thmref{thm:upper bound memorization} and \thmref{thm:lower bound memorization}}

Using \Cref{thm:ortho pres pos,thm:gen pres neg} we can now prove the main results:

\begin{proof}[Proof of \thmref{thm:upper bound memorization}]
Let $\frac{\sigma}{\delta} \leq \frac{1}{8c^{+}_{p}(d)\sqrt{e}}\sqrt{\frac{k-6}{d}}N^{-\frac{2}{k-6}}$ and let $\mathcal{D}\in\mathcal{D}_{d,N,C}(\delta)$ be a $\delta$-separated dataset. Denote $\sigma^{\prime}:= c^{+}_{p}(d)\sigma$, then from \thmref{thm:ortho pres pos} we have that $\mathcal{D}$ is $(\sigma^{\prime},\epsilon, k-6)$-orthogonally preservable with $\epsilon=\frac{1}{2}\sqrt{\frac{k-6}{de}}N^{-\frac{2}{k-6}}$. Note that $\frac{1}{\epsilon} < \frac{\delta}{2\sigma^{\prime}}$ and so from \thmref{thm:memorization of preservable data pos} and \lemref{lma:eq def of preservability} we conclude that there exists a neural network $f:\mathbb{R}^d\rightarrow \mathbb{R}$ with width $k$ and depth
\begin{equation*}
    O\left(N(k-6)\log_{2}\left(\frac{k-6}{1-\frac{2\sigma^{\prime}}{\epsilon\delta}}\right)\right) = O\left(Nk\log_{2}\left(\frac{k}{1-\frac{2\sigma^{\prime}}{\epsilon\delta}}\right)\right)~,
\end{equation*}
that $(\sigma,p)$-robustly memorizes the dataset $\mathcal{D}$.

Now $\frac{2c^{+}_{p}(d)\sigma}{\delta} \leq \frac{1}{2}\epsilon$ so $\frac{2\sigma^{\prime}}{\epsilon\delta}\leq \frac{1}{2}$ and hence the depth of $f$ is $O\left(Nk\log_{2}\left(k\right)\right)$. The theorem follows by noting that $\frac{1}{8\sqrt{e}}d^{-\frac{1}{2}+\left[\frac{1}{p}-\frac{1}{2}\right]_{-}} \leq \frac{1}{8c^{+}_{p}(d)\sqrt{e}}\sqrt{\frac{k-6}{d}}$.
\end{proof}

\begin{proof}[Proof of \thmref{thm:lower bound memorization}]
Let $\frac{\sigma}{\delta} > \frac{2416}{c^{-}_{p}(d)} N^{-\frac{2}{k}}$ and denote $\sigma^{\prime}:= c^{-}_{p}(d)\sigma$, then from  \thmref{thm:gen pres neg} we get that there exists a $\delta$-separated dataset $\mathcal{D}\in\mathcal{D}_{d,N,2}(\delta)$ which is not $(\sigma^{\prime}, k)$-preservable. Finally, by \thmref{thm:memorization of preservable data neg} we conclude that there isn't a neural network $f$ with width equal to $k$ that $(\sigma,p)$-robustly memorizes the dataset $\mathcal{D}$. The theorem follows by noting that $2416d^{\left[\frac{1}{p}-\frac{1}{2}\right]_{+}} = \frac{2416}{c^{-}_{p}(d)}$ 
\end{proof}

\subsection{Proofs of \thmref{thm:memorization of preservable data pos} and \thmref{thm:memorization of preservable data neg}}\label{subsection: pres mem pos neg}

\begin{proof}[Proof of \thmref{thm:memorization of preservable data pos}]
    Denote the elements in $\mathcal{D}$ by $\mathcal{D}=\left\{(x_{i},y_{i})\right\}^{N}_{i=1}\in\mathcal{D}_{d,N,C}(\delta)$. 
    We know that $\mathcal{D}$ is $(\sigma^{\prime}, k)$-preservable under the map $M$. By definition of $(\sigma^{\prime}, k)$-preservability and \lemref{lma:eq def of preservability}, $\sigma^{\prime} < \frac{\delta}{2}$ and there exists $P\in\text{Gr}_{d,k}$ such that for every $a_{i}\in B^{d}_{2}(x_{i},\sigma^{\prime}), a_{j}\in B^{d}_{2}(x_{j},\sigma^{\prime})$ with $y_{i}\neq y_{j}$ one has
    \begin{equation*}
        \epsilon{\lVert a_{i}-a_{j} \rVert}_{2} \leq {\lVert P\left(a_{i}-a_{j}\right) \rVert}_{2}~
    \end{equation*}
    with $\epsilon = \frac{1}{\lVert M \rVert_{2}}$. We look at the data points of $\mathcal{D}$ projected by $P$. Denote $\mathcal{D}^{\prime}=\left\{(x^{\prime}_{i},y_{i})\right\}^{N}_{i=1}$ where $x^{\prime}_{i}=P(x_{i})$ then $\mathcal{D^{\prime}}\in \mathcal{D}_{k,N,C}$. For any $y_{i}\neq y_{j}$ and $a^{\prime}_{i}\in B^{k}_{2}(x^{\prime}_{i}, \sigma^{\prime})$, $a^{\prime}_{j}\in B^{k}_{2}(x^{\prime}_{j}, \sigma^{\prime})$ we get from \lemref{lma:ortho proj keeps l2 balls} that there are $a_{i}\in B^{d}_{2}(x_{i}, \sigma^{\prime}), a_{j}\in B^{d}_{2}(x_{j}, \sigma^{\prime})$ such that $Pa_{i} = a^{\prime}_{i}, Pa_{j} = a^{\prime}_{j}$. Hence
    \begin{equation}\label{eq:bound sep of proj}
        \lVert a^{\prime}_{i}-a^{\prime}_{j}\rVert_{2} = \lVert P(a_{i}-a_{j})\rVert_{2} \geq \epsilon\lVert a_{i}-a_{j}\rVert_{2} \geq \epsilon(\delta - 2\sigma^{\prime})~.
    \end{equation}
    
    Denote $\tau=\min \left\{\lVert x^{\prime}_{i} - x^{\prime}_{j}\rVert_{2} \mid  y_{i} \neq y_{j}\right\}$ (note that $0<\epsilon\delta\leq\tau$). Assume that $\tau/2 \leq \sigma^{\prime}$ then there exists $y_{i}\neq y_{j}$ and $a^{\prime}_{i}\in B^{k}_{2}(x^{\prime}_{i}, \sigma^{\prime})$, $a^{\prime}_{j}\in B^{k}_{2}(x^{\prime}_{j}, \sigma^{\prime})$ such that $a_{i}^{\prime} = a_{j}^{\prime}$. From \eqref{eq:bound sep of proj} we get that $\epsilon(\delta - 2\sigma^{\prime}) \leq 0$ so $\delta\leq 2\sigma^{\prime}$ which is a contradiction, and so $\sigma^{\prime} < \tau / 2$.
    
    This means that $\mathcal{D^{\prime}}\in\mathcal{D}_{k,N,C}(\tau)$ and $\frac{\sigma^{\prime}}{\tau} < \frac{1}{2}=\frac{1}{2c^{+}_{2}(k)}$. Denote $k^{\prime}=k+6$ then by \thmref{thm:upper bound memorization big k} there exists a neural network $f^{\prime}:\mathbb{R}^k\rightarrow \mathbb{R}$ with width $k^{\prime}$ and depth 
    \begin{equation*}
        O\left(Nk\log_{2}\left(\frac{k}{1-\frac{2c^{+}_{2}(k)\sigma^{\prime}}{\tau}}\right)\right)=O\left(Nk\log_{2}\left(\frac{k}{1-\frac{2\sigma^{\prime}}{\tau}}\right)\right)
    \end{equation*}
    that $(\sigma^{\prime}, 2)$-robustly memorizes the dataset $\mathcal{D^{\prime}}$. Note that $0<\epsilon\delta\leq \tau$ and that $\frac{1}{\epsilon}=\lVert M \rVert_{2} < \frac{\delta}{2\sigma^{\prime}}$ so $0 < \frac{2\sigma^{\prime}}{\tau} \leq \frac{2\sigma^{\prime}}{\epsilon\delta} < 1$ and the depth of $f^{\prime}$ is $O\left(Nk\log_{2}\left(\frac{k}{1-\frac{2\sigma^{\prime}}{\epsilon\delta}}\right)\right)$.

    We define the function $f=f^{\prime}\circ P$ (where we think of $P$ now as a $k\times d$ matrix). Then $f^{\prime}:\mathbb{R}^{k}\rightarrow\mathbb{R}$ and $f:\mathbb{R}^{d}\rightarrow\mathbb{R}$ have the same width and depth. Let us show that $f$ indeed $(\sigma, p)$-robustly memorizes the dataset $\mathcal{D}$:

    Let $i\in\left[N\right]$ and $x\in B^{d}_{p}\left(x_{i}, \sigma\right)$ then by \lemref{lma:lp ball and lq ball} $x\in B^{d}_{2}\left(x_{i}, c^{+}_{p}(d)\sigma\right)$ and so because $\lVert Pv \rVert_{2}\leq \lVert v \rVert_{2}$ for every $v$, we get
    \begin{equation*}
        P(x)\in B^{k}_{2}\left(P(x_{i}), c^{+}_{p}(d)\sigma\right) = B^{k}_{2}\left(x^{\prime}_{i}, \sigma^{\prime}\right)~.
    \end{equation*}
    But $f^{\prime}$ as shown above $(\sigma^{\prime},2)$-robustly memorizes the dataset $\mathcal{D^{\prime}}$ so $f^{\prime}\left(Px\right)=y_{i}$ from which we conclude that $f(x)=y_{i}$, and so $f$ indeed $(\sigma, p)$-robustly memorizes the dataset $\mathcal{D}$. Furthermore, the width of $f$ is $k+6$ and its depth is $O\left(Nk\log_{2}\left(\frac{k}{1-\lVert M \rVert_{2}\frac{2\sigma^{\prime}}{\delta}}\right)\right)$.
\end{proof}

\begin{proof}[Proof of \thmref{thm:memorization of preservable data neg}]
    Let $f=T^{(L)}\circ [\cdot]_{+} \circ T^{(L-1)} \circ ... \circ [\cdot]_{+} \circ T^{(1)}$ be a neural network with width $\mathcal{W}(f)\leq k$ and architecture $\mathcal{A}(f)=\left(d_{0}, d_{1}, ..., d_{L}\right)$. By the definition of width $\mathcal{W}$, we have $d_{1}\leq k$ and there exists some $W^{(1)}\in M_{d_{1}\times d}(\mathbb{R})$ and $b^{(1)}\in \mathbb{R}^{d_{1}}$ such that $T^{(1)}x=W^{(1)}x+b^{(1)}$. Denote $M=\begin{bmatrix} W^{(1)} \\0 \end{bmatrix} \in M_{d}(\mathbb{R})$, then $M\in \text{End}_{d,d_{1}}$. Now, $\mathcal{D}$ is not $(\sigma^{\prime}, k)$-preservable and hence not $(\sigma^{\prime}, d_{1})$-preservable, where $\sigma^{\prime}=c^{-}_{p}(d)\sigma$. Hence, from \lemref{lma:no pres} there exists some $(x_i,y_i),(x_j,y_j)\in \mathcal{D}$ with $y_i\neq y_j$ and some $a_{i}\in B^{d}_{2}(x_{i},\sigma^{\prime}), a_{j}\in B^{d}_{2}(x_{j},\sigma^{\prime})$ such that $M(a_{i}- a_{j}) = 0$. By \lemref{lma:lp ball and lq ball} we have $a_{i}\in B^{d}_{p}(x_{i},\sigma), a_{j}\in B^{d}_{p}(x_{j},\sigma)$ with $M(a_{i}- a_{j}) = 0$. Hence $W^{(1)}a_{i}=W^{(1)}a_{j}$ so $f(a_{i})=f(a_{j})$ and we conclude that $f$ cannot $(\sigma, p)$-robustly memorizes the dataset $\mathcal{D}$.
\end{proof}

The following is a useful consequence of the lack of preservability.
 
\begin{lemma}\label{lma:no pres}
    If $\mathcal{D}$ is not $(\sigma, k)$-preservable, then for every $M\in\text{End}_{d,k}$ there exists some $a_{i}\in B^{d}_{2}(x_{i},\sigma), a_{j}\in B^{d}_{2}(x_{j},\sigma)$ with $y_{i} \neq y_{j}$ such that $M(a_{i}-a_{j})=0$.
\end{lemma}

\begin{proof}[Proof of \lemref{lma:no pres}]
    Assume that there exists $M$ such that for every $a_{i}\in B^{d}_{2}(x_{i},\sigma), a_{j}\in B^{d}_{2}(x_{j},\sigma)$ with $y_{i} \neq y_{j}$ we have $0 < \lVert M(a_{i}-a_{j}) \rVert_{2}$. Let $i,j$ such that $y_{i} \neq y_{j}$. The Minkowski difference $B_{i,j}:=B^{d}_{2}(x_{i},\sigma) - B^{d}_{2}(x_{j},\sigma)$ is compact and $\lVert M(\cdot) \rVert_{2}$ is continuous and positive on $B_{i,j}$ so it obtains a minimum $0 < t_{i,j}$. Since there are finitely many $t_{i,j}$ we can denote $0 < t = \min\left\{t_{i,j} \mid y_{i}\neq y_{j}\right\}$. Similarly $\lVert \cdot \rVert_{2}$ is continuous and positive on $B_{i,j}$ so it obtains a maximum $0 < \tau^{\prime}_{i,j}$. Denote $0 < \tau^{\prime} = \max\left\{\tau^{\prime}_{i,j} \mid y_{i}\neq y_{j}\right\}$. Let $(x_i,y_i),(x_j,y_j)\in \Dcal$ with $y_i\neq y_j$ and let $a_{i}\in B^{d}_{2}(x_{i},\sigma), a_{j}\in B^{d}_{2}(x_{j},\sigma)$. Then
    \begin{equation*}
    \begin{aligned}
        &&\lVert M(a_{i}-a_{j}) \rVert_{2} &\geq t\\
        &&  &= t\frac{\lVert a_{i}-a_{j} \rVert_{2}}{\lVert a_{i}-a_{j} \rVert_{2}}\\
        &&  &\geq \frac{t}{\tau^{\prime}}\lVert a_{i}-a_{j} \rVert_{2}\\
    \end{aligned}
    \end{equation*}
    so if we define $M^{\prime} = \frac{\tau^{\prime}}{t}M$ we get $M^{\prime}\in \text{End}_{d,k}$ and it $(\sigma, k)$-preserves $\mathcal{D}$ which is a contradiction.
\end{proof}

\begin{lemma}\label{lma:eq def of preservability}
    $\mathcal{D}$ is $(\sigma, \epsilon, k)$-orthogonally preservable if and only if it is $(\sigma, k)$-preservable under some $M$ with $\lVert M \rVert_{2} = \frac{1}{\epsilon}$.
\end{lemma}

\begin{proof}[Proof of \lemref{lma:eq def of preservability}]
    If $\mathcal{D}$ is $(\sigma, \epsilon, k)$-orthogonally preservable under $P$, define $M=\frac{1}{\epsilon}P$ then $M\in \text{End}_{d,k}$ and indeed $\mathcal{D}$ is $(\sigma, k)$-preservable under $M$ with $\lVert M \rVert_{2} = \frac{1}{\epsilon}$.\\
    In the other direction, if $\mathcal{D}$ is $(\sigma, k)$-preservable under $M$, denote by $M=U\Sigma V^{\top}$ the singular value decomposition of $M$, where $U,V\in O(d)$, $\Sigma = \text{diag}(s_{1}, ..., s_{k}, 0, ..., 0)\in M_{d}(\mathbb{R})$ and $0<s_{k}\leq ...\leq s_{1}$ are the singular values of $M$. Note that $\lVert \Sigma x \rVert_{2}\leq s_{1} \lVert P_{W_{0}} x \rVert_{2}$ for any $x$ where $P_{W_{0}}=\begin{bmatrix}I_{k} & 0 \\0 & 0 \end{bmatrix}$. Thus, for any $x$
    \begin{equation*}
    \begin{aligned}
        && \lVert x \rVert_{2} \leq \lVert Mx \rVert_{2}&\Leftrightarrow &\lVert x \rVert_{2} &\leq \lVert U\Sigma V^{\top}x \rVert_{2} \\
        && &\Leftrightarrow &\lVert x \rVert_{2} &\leq \lVert \Sigma V^{\top}x \rVert_{2} \\
        && &\Rightarrow &\lVert x \rVert_{2} &\leq s_{1}\lVert P_{W_{0}}V^{\top}x \rVert_{2} \\
        && &\Leftrightarrow &\lVert x \rVert_{2} &\leq s_{1}\lVert P_{VW_{0}}x \rVert_{2}
    \end{aligned}
    \end{equation*}
    where the last equivalence follows from \lemref{lma:ortho proj equi to ortho act} and the fact that $V^{\top}=V^{-1}$. Hence, if we define $W=VW_{0}$ then $P_{W}\in\text{Gr}_{d,k}$ and for every $a_{i}\in B^{d}_{2}(x_{i},\sigma), a_{j}\in B^{d}_{2}(x_{j},\sigma)$ with $y_{i} \neq y_{j}$ one has
    \begin{equation*}
        \frac{1}{s_{1}}{\lVert a_{i}-a_{j} \rVert}_{2} \leq {\lVert P_{W}\left(a_{i}-a_{j}\right) \rVert}_{2}~.
    \end{equation*}
    So $\mathcal{D}$ is $(\sigma, \epsilon, k)$-orthogonally preservable with $\epsilon = \frac{1}{s_{1}}=\frac{1}{\lVert M \rVert_{2}}$
\end{proof}

\section{Proof of Characterization of Preservability}\label{sec:proof pres bounds}

\subsection{Proof of \thmref{thm:ortho pres pos}}\label{subsection:ortho pres pos}

We begin by proving the following lemma which is a modification of \cite[Lemma~3.11]{mattila1999geometry}:

\begin{lemma}\label{lma:grass to sphere}
    For any $x\in\mathbb{S}^{d-1}$ and $0<r<1$ one has
    \begin{equation*}
        \gamma_{d,k}\left(\left\{P\in\text{Gr}_{d,k} \mid {\lVert Px \rVert}_{2} \leq r \right\}\right)=\mu_{d-1}\left(\left\{ y\in \mathbb{S}^{d-1} \mid y_{1}^{2}+...+y_{k}^{2} \leq r^{2} \right\}\right)~.
    \end{equation*}
\end{lemma}

\begin{proof}[Proof of \lemref{lma:grass to sphere}]
    Let $V_{0} = \text{Span}\left\{e_{k+1}, ..., e_{d}\right\}$, we have that:
    \begin{equation*}
    \begin{aligned}
        &\gamma_{d,k}\left(\left\{W\in\text{Gr}_{d,k} \mid {\lVert P_{W}x
 \rVert}_{2} \leq r \right\}\right) & & &\\
        &=\gamma_{d,k}\left(\left\{W\in\text{Gr}_{d,k} \mid \text{dist}_{2}(x, W^{\perp}) \leq r \right\}\right) & & &\text{(definition of orthogonal projection)}\\
        &=\gamma_{d,d-k}\left(\left\{W^{\perp}\in\text{Gr}_{d,d-k} \mid \text{dist}_{2}(x, W^{\perp}) \leq r \right\}\right) & & & \text{(\lemref{lma:invariance to perp})}\\
        &=\gamma_{d,d-k}\left(\left\{V\in\text{Gr}_{d,d-k} \mid \text{dist}_{2}(x, V) \leq r \right\}\right) & & & \\
        &=\nu_{d}\left(\left\{ g\in O\left(d\right) \mid \text{dist}_{2}(x, gV_{0}) \leq r \right\}\right) & & & \text{(\lemref{lma:pushforward property})}\\
        &=\nu_{d}\left(\left\{ g\in O\left(d\right) \mid \text{dist}_{2}(g^{-1}x, V_{0}) \leq r \right\}\right) & & & \text{($O\left(d\right)$ preserves $l_{2}$ norm)}\\
        &=\mu_{d-1}\left(\left\{ y\in \mathbb{S}^{d-1} \mid \text{dist}_{2}(y, V_{0}) \leq r \right\}\right) & & & \text{(\lemref{lma:pushforward property})}\\
        &=\mu_{d-1}\left(\left\{ y\in \mathbb{S}^{d-1} \mid {\lVert P_{V^{\perp}_{0}}y
 \rVert}_{2} \leq r \right\}\right) & & & \text{(definition of orthogonal projection)}\\
        &=\mu_{d-1}\left(\left\{ y\in \mathbb{S}^{d-1} \mid y_{1}^{2}+...+y_{k}^{2} \leq r^{2} \right\}\right) & & & \text{(definition of $V_{0}$)}
    \end{aligned}
    \end{equation*}
\end{proof}

\begin{lemma}\label{lma:grass zone by center}
    For any $x\in\mathbb{S}^{d-1}$ one has
    \begin{equation*}
        \left\{P\in\text{Gr}_{d,k} \mid \exists b\in B^{d-1}_{\text{arc}}\left(x,\varphi\right) \text{ s.t }   {\lVert Pb
 \rVert}_{2} \leq \epsilon \right\} \subseteq \left\{P\in\text{Gr}_{d,k} \mid {\lVert Px
 \rVert}_{2} \leq \epsilon+\sin\varphi \right\}.
    \end{equation*}
\end{lemma}

\begin{proof}[Proof of \lemref{lma:grass zone by center}]
    Let $P$ be in the set on the left-hand side, and denote $W=Im(P)\cap \mathbb{S}^{d-1}$, then there exists $b\in B^{d-1}_{\text{arc}}\left(x,\varphi\right)$ such that $\lVert Pb \rVert_{2}\leq \epsilon$. Now, $\text{dist}_{\text{arc}}(b,W)$ (the angle between $b$, $Pb$) satisfies
    \begin{equation*}
        \cos(\text{dist}_{\text{arc}}(b,W))=\frac{\langle b, Pb\rangle}{\lVert b\rVert_{2}\lVert Pb\rVert_{2}}=\frac{\lVert Pb\rVert^{2}_{2}}{1\cdot \lVert Pb\rVert_{2}} = \lVert Pb\rVert_{2} \leq \epsilon.
    \end{equation*}    
    Similarly $\lVert P\left(x\right) \rVert_{2} = \cos\left(\text{dist}_{\text{arc}}(x,W)\right)$. From the triangle inequality we have $\text{dist}_{\text{arc}}(b,W)\leq \text{dist}_{\text{arc}}(b,x)+\text{dist}_{\text{arc}}(x,W)\leq \varphi+\text{dist}_{\text{arc}}(x,W)$ and so
    \begin{equation*}
    \begin{aligned}
        & &\lVert P\left(x\right) \rVert_{2} & = \cos\left(\text{dist}_{\text{arc}}(x,W)\right) \\
        && & \leq \cos\left(\text{dist}_{\text{arc}}(b,W) - \varphi\right)\\
        && & = \cos\left(\text{dist}_{\text{arc}}(b,W)\right)\cos\left(\varphi\right) + \sin\left(\text{dist}_{\text{arc}}(b,W)\right)\sin\left(\varphi\right)\\
        && & \leq \cos\left(\text{dist}_{\text{arc}}(b,W)\right) + \sin\left(\varphi\right)\\
        && & \leq \epsilon + \sin\left(\varphi\right)~.
    \end{aligned}
    \end{equation*}
\end{proof}

\begin{lemma}\label{lma:measure of sphere zone}
    For any $0 < r < 1$ one has
    \begin{equation*}
        \mu_{d-1}\left(\left\{ y\in \mathbb{S}^{d-1} \mid y_{1}^{2}+...+y_{k}^{2} \leq \frac{k}{d}r^{2} \right\}\right) < e^{\frac{k}{2}}r^{k}~.
    \end{equation*}
\end{lemma}

\begin{proof}[Proof of \lemref{lma:measure of sphere zone}]
    From \cite[Lemma~2.2]{dasgupta2003elementary}, we get that
    \begin{equation*}
        \mu_{d-1}\left(\left\{ y\in \mathbb{S}^{d-1} \mid y_{1}^{2}+...+y_{k}^{2} \leq \frac{k}{d}r^{2} \right\}\right) \leq r^{k} \cdot \left(1+\frac{k(1-r^{2})}{d-k}\right)^{(d-k)/2}~.
    \end{equation*}
    Now, $x\mapsto \left(1+\frac{k(1-x^{2})}{d-k}\right)^{(d-k)/2}$ is decreasing in $[0,1]$ so for any $k<d$
    \begin{align*}
        \left(1+\frac{k(1-r^{2})}{d-k}\right)^{(d-k)/2} &\leq \left(1+\frac{k}{d-k}\right)^{(d-k)/2}\\
         &= \left(1+\frac{k/2}{(d-k)/2}\right)^{(d-k)/2}\\
        & \leq e^{\frac{k}{2}}
         ~,
    \end{align*}
    which finishes the proof.

\end{proof}

\begin{lemma}\label{lma:grass measure upper bound}
    Let $\sin\varphi < \frac{1}{2}\sqrt{\frac{k}{de}}m^{-\frac{1}{k}}$, and denote $\epsilon = \frac{1}{2\sqrt{e}}m^{-\frac{1}{k}}$, then for any $x\in\mathbb{S}^{d-1}$
    \begin{equation*}
        \gamma_{d,k}\left(\left\{P\in\text{Gr}_{d,k} \mid \exists b\in B^{d-1}_{\text{arc}}\left(x,\varphi\right) \text{ s.t }   {\biggl\lVert \sqrt{\frac{d}{k}}Px
 \biggr\rVert}_{2} \leq \epsilon \right\}\right) < m^{-1}
    \end{equation*}
\end{lemma}

\begin{proof}[Proof of \lemref{lma:grass measure upper bound}]
    We have
    \begin{equation*}
    \begin{aligned}
        &\gamma_{d,k}\left(\left\{P\in\text{Gr}_{d,k} \mid \exists b\in B^{d-1}_{\text{arc}}\left(x,\varphi\right) \text{ s.t }   {\biggl\lVert \sqrt{\frac{d}{k}}Px
 \biggr\rVert}_{2} \leq \epsilon \right\}\right)&&&\\
        &\leq\gamma_{d,k}\left(\left\{P\in\text{Gr}_{d,k} \mid {\lVert Px
 \rVert}_{2} \leq \sqrt{\frac{k}{d}}\epsilon+\sin\varphi \right\}\right)&&&\text{(\lemref{lma:grass zone by center})}\\
        &=\gamma_{d,k}\left(\left\{P\in\text{Gr}_{d,k} \mid {\lVert Px \rVert}_{2} \leq \sqrt{\frac{k}{d}}\left(\epsilon+\sqrt{\frac{d}{k}}\sin\varphi\right) \right\}\right)&&&\\
        &=\mu_{d-1}\left(\left\{ y\in \mathbb{S}^{d-1} \mid y_{1}^{2}+...+y_{k}^{2} \leq \frac{k}{d}\left(\epsilon + \sqrt{\frac{d}{k}}\sin\varphi\right)^{2} \right\}\right)&&&\text{(\lemref{lma:grass to sphere})}\\
        &\leq e^{\frac{k}{2}}\left(\epsilon + \sqrt{\frac{d}{k}}\sin\varphi\right)^{k}&&&\text{(\lemref{lma:measure of sphere zone})}\\
        &= e^{\frac{k}{2}}\left(\frac{1}{2\sqrt{e}}m^{-\frac{1}{k}} + \sqrt{\frac{d}{k}}\sin\varphi\right)^{k}&&&\\
        &<e^{\frac{k}{2}}\left(\frac{1}{2\sqrt{e}}m^{-\frac{1}{k}} + \sqrt{\frac{d}{k}}\frac{1}{2}\sqrt{\frac{k}{de}}m^{-\frac{1}{k}}\right)^{k}&&&\text{(choice of $\varphi$)}\\
        &=e^{\frac{k}{2}}\left(\frac{1}{\sqrt{e}}m^{-\frac{1}{k}}\right)^{k}&&&\\
        &=m^{-1}.
    \end{aligned}
    \end{equation*}
    Where we can use \lemref{lma:measure of sphere zone} because $\epsilon + \sqrt{\frac{d}{k}} \sin\varphi \leq \frac{1}{\sqrt{e}}m^{-\frac{1}{k}} \leq \frac{1}{\sqrt{e}} 2^{-\frac{1}{k}} < 1$.
\end{proof}

\begin{lemma}\label{lma:existance of pres proj}
    Let $\mathcal{V}\subset\mathbb{S}^{d-1}$
    with $\lvert \mathcal{V} \rvert = m$. Denote $\epsilon = \frac{1}{2}\sqrt{\frac{k}{de}}m^{-\frac{1}{k}}$ and let $\sin\varphi < \epsilon$, then there exists some $P\in\text{Gr}_{d,k}$ such that for every $u \in \underset{v\in \mathcal{V}}{\bigcup} B^{d-1}_{\text{arc}}\left(v, \varphi\right)$ one has $\epsilon \leq \lVert Pu \rVert_{2}$. 
\end{lemma}

\begin{proof}[Proof of \lemref{lma:existance of pres proj}]
    Denote $A_{v}=\left\{P\in\text{Gr}_{d,k} \mid \exists b\in B^{d-1}_{\text{arc}}\left(v,\varphi\right) \text{ s.t }   {\lVert Pb \rVert}_{2} \leq \epsilon \right\}$, then
    \begin{equation*}
        \gamma_{d,k}\left( \underset{v\in\mathcal{V}}{\bigcup}A_{v}  \right) \leq \underset{v\in\mathcal{V}}{\sum}\gamma_{d,k}\left(A_{v}\right)<\underset{v\in\mathcal{V}}{\sum}\frac{1}{m} = m\frac{1}{m}=1=\gamma_{d,k}\left( \text{Gr}_{d,k} \right)~,
    \end{equation*}
    where the second inequality follows from \lemref{lma:grass measure upper bound}. Therefore there exists some $P\in\text{Gr}_{d,k}$ such that for every $u \in \underset{v\in \mathcal{V}}{\bigcup} B^{d-1}_{\text{arc}}\left(v, \varphi\right)$ one has $\epsilon \leq \lVert Pu \rVert_{2}$.
\end{proof}

\begin{proof}[Proof of \thmref{thm:ortho pres pos}]
    Let $\mathcal{D}\in\mathcal{D}_{d,N,C}(\delta)$ be a $\delta$-separated dataset. Define $\mathcal{V}=\left\{\frac{x_{i}-x_{j}}{\lVert x_{i}-x_{j}\rVert_{2}} \mid y_{i} \neq y_{j}\right\}$ then $\mathcal{V}\subset\mathbb{S}^{d-1}$ with $\lvert \mathcal{V} \rvert = m$ for some $m\leq N^{2}$. Denote $\varphi = \sin^{-1}\left(\frac{2\sigma}{\delta}\right)$ then $\sin\varphi <  \frac{1}{2}\sqrt{\frac{k}{de}}N^{-\frac{2}{k}}\leq  \frac{1}{2}\sqrt{\frac{k}{de}}m^{-\frac{1}{k}}$. Hence, by \lemref{lma:existance of pres proj} there exists some $P\in\text{Gr}_{d,k}$ such that for every $u \in \underset{v\in \mathcal{V}}{\bigcup} B^{d-1}_{\text{arc}}\left(v, \varphi\right)$ one has $\frac{1}{2}\sqrt{\frac{k}{de}}m^{-\frac{1}{k}} \leq \lVert Pu \rVert_{2}$ so $\frac{1}{2}\sqrt{\frac{k}{de}}N^{-\frac{2}{k}} \leq \lVert Pu \rVert_{2}$. By the definition of $\mathcal{V}, \delta, \varphi$ and $B^{d-1}_{\text{arc}}$ we conclude that for every $a_{i}\in B^{d}_{2}(x_{i},\sigma), a_{j}\in B^{d}_{2}(x_{j},\sigma)$ with $y_{i}\neq y_{j}$ one has
    \begin{equation*}
        \frac{a_{i}-a_{j}}{\lVert a_{i}-a_{j} \rVert_{2}} \in B^{d-1}_{\text{arc}}\left(\frac{x_{i}-x_{j}}{\lVert x_{i}-x_{j} \rVert_{2}}, \sin^{-1}\left(\frac{2\sigma}{\lVert x_{i}-x_{j} \rVert_{2}}\right)\right)\subseteq B^{d-1}_{\text{arc}}\left(\frac{x_{i}-x_{j}}{\lVert x_{i}-x_{j} \rVert_{2}}, \sin^{-1}\left(\frac{2\sigma}{\delta}\right)\right)~,
    \end{equation*}
    and so
    \begin{equation*}
        \epsilon{\lVert a_{i}-a_{j} \rVert}_{2} \leq {\lVert P\left(a_{i}-a_{j}\right) \rVert}_{2}~,
    \end{equation*}
    which completes the proof.
\end{proof}

\subsection{Proof of \thmref{thm:gen pres neg}}\label{subsection:gen pres neg}

In order to prove \thmref{thm:gen pres neg} we show that for a big enough geodesic radius $\varphi$ we can always cover $\mathbb{S}^{k}$ with geodesic balls $B^{k}_{\text{arc}}\left(\varphi\right)$. We do so using the covering number $\tau$ (defined in \defref{def:covering number}) and covering density $\vartheta$ (defined in \defref{def:covering density}). This cover will enable us to construct a non preservable dataset. 

We begin by proving the following lemma
\begin{lemma}\label{lma:covering number general ball}
    Let $r\leq R$, $k\in\mathbb{N}$ and denote $\Delta_{k}=\sqrt{2}\left(5k\ln\left(k+1\right)\sqrt{2\pi\left(k+1\right)}\right)^{1/k}$. We have that if $\left(\frac{\Delta_{k} R}{r}\right)^{k}<m$, then there exists a set $\mathcal{V}=\left\{v_{1}, ..., v_{m}\right\}\subseteq \partial B^{k+1}_{2}(0, R)$ such that
    \begin{equation*}
        \partial B^{k+1}_{2}(0, R) \subseteq \bigcup\limits_{i=1}^{m} B^{k+1}_{2}\left(v_{i}, r\right)~.
    \end{equation*}
\end{lemma}

\begin{proof}[Proof of \lemref{lma:covering number general ball}]
    Let $2\sin(\varphi/2) = r/R$. We have
    \begin{equation*}\label{eq:bound covering number above}
    \begin{aligned}
        &&\tau\left(\partial B^{k+1}_{2}(0, R), \mathcal{G}_{O(k+1)}B^{k+1}_{2}\left(Re_{1}, r\right)\right)&=\tau\left(\mathbb{S}^{k}, \mathcal{G}_{O(k+1)}B^{k+1}_{2}\left(e_{1}, r/R\right) \right)&\\
        &&&=\tau\left(\mathbb{S}^{k}, \mathcal{G}_{O(k+1)} B^{k}_{\text{arc}}\left( \varphi\right) \right)&\\
        && &=\frac{\vartheta\left(\mathbb{S}^{k}, \mathcal{G}_{O(k+1)}B^{k}_{\text{arc}}\left(\varphi\right)\right)}{\mu_{k}\left(B^{k}_{\text{arc}}\left(\varphi\right)\right)}&\\
        && &\leq \frac{5k\ln\left(k+1\right)}{\mu_{k}\left(B^{k}_{\text{arc}}\left(\varphi\right)\right)} &&&\text{(\lemref{lma:covering density bound})}\\
        && &\leq \frac{5k\ln\left(k+1\right)\sqrt{2\pi\left(k+1\right)}}{\sin^{k}\varphi}&&&\text{(\lemref{lma:bound cap measure})}\\
        && &= \frac{2^{\frac{k}{2}}5k\ln\left(k+1\right)\sqrt{2\pi\left(k+1\right)}}{(\sqrt{2}\sin\varphi)^k}\\
        &&&=\frac{\Delta^{k}_{k}}{(\sqrt{2}\sin\varphi)^k}&&&\\
        && &\leq \frac{\Delta_{k}^{k}}{(2\sin(\varphi/2))^k} = \frac{\Delta_{k}^{k}}{(r/R)^k}&&&\text{($2\sin(\varphi/2)\leq \sqrt{2}\sin\varphi$)}\\
        && & < m\\    
    \end{aligned}
    \end{equation*}
    We conclude that there exists a set $\mathcal{V}$ with $|\mathcal{V}|=m$ which satisfies the requirements of the lemma.
\end{proof}

\begin{lemma}\label{lma:radius of inner ball}
    Let $0 \leq \sigma < \frac{\delta}{2}$ then 
    \begin{equation*}
        \frac{1}{2}\left(\sigma + \sqrt{\sigma^{2} + 4\sigma\delta}\right) \leq \sqrt{2\sigma\delta}
    \end{equation*}
\end{lemma}

\begin{proof}[Proof of \lemref{lma:radius of inner ball}]
    Assume $0 < \sigma$ (otherwise the claim is trivial). Note that
    \begin{equation*}
    \begin{aligned}
        & &\frac{1}{2}\left(\sigma + \sqrt{\sigma^{2} + 4\sigma\delta}\right) \leq \sqrt{2\sigma\delta} & \Longleftrightarrow 2\sigma^{2} + 4\sigma\delta + 2\sigma\sqrt{\sigma^{2}+4\sigma\delta}\leq 8\sigma\delta \\
        & & & \Longleftrightarrow \sigma^{2} + \sigma\sqrt{\sigma^{2}+4\sigma\delta}\leq 2\sigma\delta\\
        & & & \Longleftrightarrow \sigma + \sqrt{\sigma^{2}+4\sigma\delta}\leq 2\delta~,
    \end{aligned}
    \end{equation*}
    and indeed
    \begin{equation*}
        \sigma + \sqrt{\sigma^{2}+4\sigma\delta} < \left(\frac{\delta}{2}\right) + \sqrt{\left(\frac{\delta}{2}\right)^{2}+4\left(\frac{\delta}{2}\right)\delta} = 2\delta~,
    \end{equation*}
    so we are done.
\end{proof}

We are now ready to prove \thmref{thm:gen pres neg}:
\begin{proof}[Proof of \thmref{thm:gen pres neg}]
    Denote $r=\sqrt{2\sigma\delta}$ then from \lemref{lma:radius of inner ball} we have that
    \begin{equation*}
        r\geq \frac{1}{2}\left(\sigma + \sqrt{\sigma^2+4\sigma\delta}\right)~,
    \end{equation*}
    and so $0\leq r^{2}-\sigma r -\sigma\delta$, from which we conclude that
    \begin{equation}\label{eq:r sigma ratio}
        \frac{r+\delta}{r} \leq \frac{r}{\sigma}~.
    \end{equation}
    We also have
    \begin{equation}\label{eq:sigma div r}
        \left(\frac{\sigma}{r}\right)^{2}=\left(\frac{\sigma}{\sqrt{2\sigma\delta}}\right)^{2}=\frac{\sigma}{2\delta}~.
    \end{equation}
    Now, 
    \begin{equation*}
    \begin{aligned}
        & &\frac{2\sigma}{\delta} > 4832N^{-\frac{2}{k}}&\Longrightarrow  \frac{2\sigma}{\delta} > 16\left(\sqrt{2}\cdot 10\sqrt{\pi}\ln 2\right)^{2}N^{-\frac{2}{k}} & & &\\
        & & &\Longrightarrow  \frac{2\sigma}{\delta} > 16\Delta^{2}_{k}N^{-\frac{2}{k}} & & &\text{(Definition of $\Delta_{k}$ and \lemref{lma:density computation bound})}\\
        & & &\Longrightarrow  \frac{\sigma}{2\delta} > 4\Delta^{2}_{k}N^{-\frac{2}{k}} & & &\\
        & & &\Longrightarrow \left(\frac{\sigma}{r}\right)^{2} > 4\Delta^{2}_{k}N^{-\frac{2}{k}} & & &\text{(\eqref{eq:sigma div r})}\\
        & & &\Longrightarrow \frac{r}{\sigma} < 
        \frac{1}{2\Delta_{k}}N^{\frac{1}{k}}& & &\\
        & & &\Longrightarrow \frac{r}{\sigma} < 
        \frac{1}{\Delta_{k}}\left(\frac{N}{2}\right)^{\frac{1}{k}}& & &\text{($\frac{1}{2}\leq \frac{1}{2^{\frac{1}{k}}}$ for all $1\leq k$)}\\
        & & &\Longrightarrow \left(\frac{\Delta_{k}r}{\sigma}\right)^{k}<\frac{N}{2}~,
    \end{aligned}
    \end{equation*}
    so by \lemref{lma:covering number general ball} there exists some $a_{1}, ..., a_{N/2}$ points on the sphere $\partial B^{k+1}_{2}(0, r)$ such that
    \begin{equation*}
        \partial B^{k+1}_{2}(0, r) \subseteq \bigcup\limits_{i=1}^{N/2} B^{k+1}_{2}\left(a_{i}, \sigma\right)~.
    \end{equation*}
    By \lemref{lma:action preserves cover} we have $\partial B^{k+1}_{2}(0, r) \subseteq \bigcup\limits_{i=1}^{N/2} B^{k+1}_{2}\left(-a_{i}, \sigma\right)$ hence for any $x^{\prime}$ we have
    \begin{equation*}
        \partial B^{k+1}_{2}(x^{\prime}, r) \subseteq \bigcup\limits_{i=1}^{N/2} B^{k+1}_{2}\left(x^{\prime}-a_{i}, \sigma\right)~.
    \end{equation*}
    Now, by \eqref{eq:r sigma ratio}  we have $\frac{r+\delta}{r} \leq \frac{r}{\sigma}$ and so $\left(\frac{\Delta_{k}(r+\delta)}{r}\right)^{k}\leq \left(\frac{\Delta_{k}r}{\sigma}\right)^{k} <\frac{N}{2}$. Therefore, by \lemref{lma:covering number general ball} there exists some $b_{1}, ..., b_{N/2}$ points on the sphere $\partial B^{k+1}_{2}(0, r+\delta)$ such that
    \begin{equation*}
        \partial B^{k+1}_{2}(0, r+\delta) \subseteq \bigcup\limits_{i=1}^{N/2} B^{k+1}_{2}\left(b_{i}, r\right)~.
    \end{equation*}
    We now define the dataset $\mathcal{D}=\{(a_{1},1), ..., (a_{N/2}, 1), (b_{1},2), ..., (b_{N/2}, 2)\}$. Note that by definition $\delta \leq \lVert a_{i} - b_{j}\rVert_{2}$ and so $\mathcal{D}\in \mathcal{D}_{d,N,2}(\delta)$.

    We now show that $\mathcal{D}$ is not $(\sigma,k)$-preservable. Let $M\in \text{End}_{d,k}$ and denote $K=\ker M$, then $K\in\text{Gr}_{d,d-k}$. Now $\mathbb{R}\partial B^{k+1}_{2}(0, r+\delta)\in\text{Gr}_{d,k+1}$ and since $d-k\geq d-\left(k+1\right)$ + 1 we get from \lemref{lma:grass identity} that there exists some $0\neq x\in K \cap \mathbb{R}\partial B^{k+1}_{2}(0, r+\delta)$. Since $0\neq x\in \mathbb{R}\partial B^{k+1}_{2}(0, r+\delta)$ we can write $x=t_{1} u$ for some $u\in \partial B^{k+1}_{2}(0, r+\delta)$ and $t_{1} \neq 0$.\\ Now, $u\in \partial B^{k+1}_{2}(0, r+\delta)\subseteq \bigcup\limits_{i=1}^{N/2} B^{k+1}_{2}\left(b_{i}, r\right)$ hence there exists some $b_{j}$ such that $u\in B^{k+1}_{2}\left(b_{j}, r\right)$. Therefore, there exists some $v\in \partial B^{k+1}_{2}\left(b_{j}, r\right)$ such that $u=t_{2}v$ for some $t_{2}\neq 0$. Now $Mv=Mt_{2}^{-1}u=Mt_{2}^{-1}t_{1}^{-1}x=t_{2}^{-1}t_{1}^{-1}Mx=0$ so $v\in K$. But $\partial B^{k+1}_{2}(b_{j}, r) \subseteq \bigcup\limits_{i=1}^{N/2} B^{k+1}_{2}\left(b_{j}-a_{i}, \sigma\right)$ and so there exists some $a_{i}$ such that $v\in B^{k+1}_{2}\left(b_{j}-a_{i}, \sigma\right)$. Equivalently, there exists some $\alpha_{i}\in B^{k+1}_{2}\left(a_{i}, \sigma\right)$ such that $v = b_{j} - \alpha_{i}$. But $v\in K$ so $Mv=0$. Treating $\mathbb{R}^{k+1}$ as a subspace of $\mathbb{R}^{d}$, we conclude that there exists some $\alpha_{i}\in B^{d}_{2}\left(a_{i}, \sigma\right)$ and some $\beta_{j}=b_{j}\in B^{d}_{2}\left(b_{j}, \sigma\right)$ such that
    \begin{equation*}
        \lVert M(\beta_{j}-\alpha_{i}) \rVert_{2} = 0 < \delta -\sigma \leq \lVert \beta_{j}-\alpha_{i} \rVert_{2}
    \end{equation*}
    and so by definition $\mathcal{D}$ is not $(\sigma, k)$-preservable.
\end{proof}

\section{Tighter Bounds}\label{sec:tighter bounds}

For a $\delta$-separated dataset $\mathcal{D}$ denote the set $X_{\mathcal{D}}=\left\{\frac{x-x^{\prime}}{\lVert x-x^{\prime} \rVert_{2}} \mid  (x,y),(x^{\prime},y^{\prime})\in \Dcal \text{ with } y\neq y^{\prime}\right\}$. From \defref{def:preserving func}, it follows that the problem introduced in \secref{sec:pres lin map} is equivalent to the following problem: Under what conditions, for any $\delta$-separated dataset $\mathcal{D}$ there exists a rank $k$ map $M\in\text{End}_{d,k}$ such that
\begin{equation*}
    \forall v\in \bigcup\limits_{v\in X_{\mathcal{D}}}B^{d-1}_{\text{arc}}(v,r_{v})\quad \text{we have } 1\leq \rVert Mv\lVert_{2}~.
\end{equation*}
Here the radius $r_{v}$ for $v=\frac{x-x^{\prime}}{\lVert x-x^{\prime} \rVert_{2}}$ is given by $r_{v}=\sin^{-1}\left(\frac{2\sigma}{\lVert x-x^{\prime} \rVert_{2}}\right)\leq \sin^{-1}\left(\frac{2\sigma}{\delta}\right)$. Note that since the set $\bigcup\limits_{v\in X_{\mathcal{D}}}B^{d-1}_{\text{arc}}(v,r_{v})$ is compact, there exists such $M$ if and only if $Mv\neq 0$ on this set (see e.g \lemref{lma:no pres}). We conclude that the problem introduced in \secref{sec:pres lin map} can be solved given a solution to the following more general problem on the sphere: Under what conditions, for any fixed set $X\subset\mathbb{S}^{d-1}$ of size $\lvert X \rvert = m$ there exists a rank $k$ map $M\in\text{End}_{d,k}$ such that
\begin{equation*}
    \forall x\in X^{(\varphi)}, Mx\neq 0. \quad\text{where}\quad X^{(\varphi)}:=\bigcup\limits_{x\in X}B^{d-1}_{\text{arc}}(x,\varphi)~.
\end{equation*}

Or, equivalent, that there exists some subspace $U\in\text{Gr}_{d,d-k}$ (the kernel of $M$) such that $X^{(\varphi)}\cap U = \emptyset$. We define the sets that share this property:

\begin{definition}\label{def:hitting set}
    Let $X\subset\mathbb{S}^{d-1}$. We say that $X$ is \textbf{$(d-k,\varphi)$-hitting} if for every $U\in\text{Gr}_{d,d-k}$ one has $X^{(\varphi)}\cap U \neq \emptyset$. We will denote by $\mathcal{X}_{\varphi,d-k}$ the set of all $(d-k,\varphi)$-hitting sets, and define the \textbf{$(d-k,\varphi)$-hitting number} to be
    \begin{equation*}
        m_{\varphi,d-k}=\min\left\{ \lvert X \rvert \mid X\in\mathcal{X}_{\varphi,d-k} \right\}~.
    \end{equation*}
\end{definition}

If $m < m_{\varphi, d-k}$ then any $X\subset\mathbb{S}^{d-1}$  with $\lvert X \rvert = m$ will satisfy $X^{(\varphi)}\cap U = \emptyset$ for some $U\in\text{Gr}_{d,d-k}$. On the other hand, if $m_{\varphi, d-k}\leq m$ then  there exists some $X\subset\mathbb{S}^{d-1}$  with $\lvert X \rvert = m$ such that for every $U\in\text{Gr}_{d,d-k}$ one has $X^{(\varphi)}\cap U \neq \emptyset$. We conclude that our problem reduces to finding the value of $m_{\varphi, d-k}$. Namely, given $\varphi,d,k$, what is the minimal number of spherical caps on $\mathbb{S}^{d-1}$ of radius $\varphi$ that are required in order to intersect every $d-k$-dimensional subspace non-trivially?

Using covering arguments it follows from \subsecref{subsection:gen pres neg} that this number has an upper bound
\begin{equation*}
    m_{\varphi, d-k} \leq \left(\frac{C_{2}}{\varphi}\right)^{k}~.
\end{equation*}

For the lower bound we inspect the contribution made by each point $x\in X$ separately. Namely, for each $x\in X$ we define the set $A_{d,d-k}(x, \varphi)=\left\{U\in\text{Gr}_{d,d-k} \mid U\cap x^{(\varphi)}\neq \emptyset\right\}$. It follows from the definition that the covering number of $\text{Gr}_{d,d-k}$ by translated copies of $A_{d,d-k}(x, \varphi)$ is exactly $m_{\varphi, d-k}$ (translation is done with respect to the transitive action of $O(d)$. See \defref{def:covering number} in \appref{subsection:trans coverings}). From the proof in \subsecref{subsection:ortho pres pos} we get that the measure of $A_{d,d-k}(x, \varphi)$ is given by
\begin{equation*}
    \gamma_{d,d-k}(A_{d,d-k}(x,\varphi))=\mu_{d-1}\left(y\in\mathbb{S}^{d-1}\mid y^{2}_{1}+...+y^{2}_{k}\leq \sin^{2}(\varphi)\right)~.
\end{equation*}
But $Z=y^{2}_{1}+...+y^{2}_{k} \sim \text{Beta}\left(\frac{k}{2},\frac{d-k}{2}\right)$ (see \cite[Corollary~1.1]{frankl1990some}) so the measure is given by the CDF:
\begin{equation*}
    \gamma_{d,d-k}(A_{d,d-k}(x,\varphi))=\text{CDF}_{Z}(\sin^{2}\varphi)=I_{\sin^{2}\varphi}\left(\frac{k}{2}, \frac{d-k}{2}\right)
\end{equation*}

where $I_{r}(a,b)$ is the regularized incomplete beta function. Note that the term $\gamma_{d,d-k}(A_{d,d-k}(x,\varphi))^{-1}$ decreases to $1$ as a function of $d$ at a rate bounded by $(d/k)^{-\frac{k}{2}}$. To bound $m_{\varphi, d-k}$ from bellow we showed using the union bound that for small enough $m$, no set $X$ of size $m$ will result in a cover of $\text{Gr}_{d,d-k}$. The above discussion shows that this small enough $m$ decreases at a rate proportional to $(d/k)^{-\frac{k}{2}}$, and so the bound we get is
\begin{equation*}
    \left(\frac{k}{d}\right)^{k/2} \left(\frac{C_{1}}{\varphi}\right)^{k} \leq m_{\varphi, d-k}~.
\end{equation*}

This probabilistic approach will always provide a lower bound that depends on $d$ since the measure of each $A_{d,d-k}(x,\varphi)$ depends on $d$. Still, it is not clear whether this dependence on $d$ is unavoidable. One could argue that since the problem involves only subspaces of co-dimension $k$, the number of spherical caps required to hit them should not depend on the ambient dimension $d$. Granted, the spherical caps on $\mathbb{S}^{d-1}$ depend on $d$, but as we saw in the upper bound, there are configurations of hitting caps that do not take advantage of the ambient dimension $d$. The question is whether such configurations are necessary, i.e is it true that caps of higher dimension cannot provide a more efficient configuration of a hitting set. It remain a subject for future work to investigate tight bounds for this problem.

\section{Separation in \(l_{q}\) Norm}\label{sec:Separation in $l_{q}$ norm}

In \secref{sec:main results} we obtained results that connect robust memorization and the width of the memorizing network, for any $\delta$-separated datasets. Our definition of separation used the $l_{2}$ norm. In this section we aim to extend these results and consider datasets where separation is measured in an arbitrary $l_{q}$ norm. All the proofs for this section appear in \secref{sec:proofs for lq sep section}.

In the following, we let $N,d,C\in\mathbb{N}_{\geq 2}$, $k\in\mathbb{N}$, $0 < \delta, \sigma$, $p\in (0,\infty]$, $q\in [1,\infty]$. Note that as before, $p$ can be smaller than $1$. This is because it is used to define the geometric shape of the robust neighborhood of the data points, and for this purpose the properties of quasi-norm suffice. On the other hand, $q$ is used to define a norm that measures the separation distance between data points, and so it has to remain in the range $[1,\infty]$.

We will denote $c^{+}_{p,q}(d)=d^{\left[\frac{1}{q}-\frac{1}{p}\right]_{+}}$ and $c^{-}_{p,q}(d)=d^{\left[\frac{1}{q}-\frac{1}{p}\right]_{-}}$. As before, it follows from \lemref{lma:lp ball and lq ball} that $c^{+}_{p,q}(d)$ is the radius of the $l_{q}$ ball that encloses the unit $l_{p}$ ball, and $c^{-}_{p,q}(d)$ is the radius of the $l_{q}$ ball that is inscribed in the unit $l_{p}$ ball.

\begin{definition}\label{def:dataset lq}
    We say that a dataset $\mathcal{D}\in\mathcal{D}_{d,N,C}$ is a \textbf{$(\delta, q)$-separated} dataset, if we have $\delta = \min \left\{\lVert x_{i} - x_{j}\rVert_{q} \mid  y_{i} \neq y_{j}\right\}$, and denote by $\mathcal{D}_{d,N,C}(\delta,q)$ the set of all such datasets.
\end{definition}

As in \secref{sec:main results}, given some $0<\delta$ and $q\in[1,\infty]$ we wish to find the maximal possible value of $\sigma$ that allows for $(\sigma, p)$-robust memorization, of any $(\delta,q)$-separated dataset, using a width $k$ network. When $q=2$ we saw that the applicable range of $\sigma$ was:
\begin{equation*}
    0\leq \frac{\sigma}{\delta} <\frac{1}{2c^{+}_{p,2}(d)}~.
\end{equation*}
For general $q\in[1,\infty]$, using the same reasoning we get that $(\sigma,p)$-robust memorization of every $(\delta,q)$-separated dataset can only be considered in the range
\begin{equation*}
    0\leq \frac{\sigma}{\delta} <\frac{1}{2c^{+}_{p,q}(d)}~.
\end{equation*}

\subsection{Robust Memorization With Large Width}\label{subsec:large width l_q}

For brevity we will use the notation $\lambda=\left(1-\frac{2c^{+}_{p,q}(d)\sigma}{\delta}\right)^{-1}$ in the following theorems. In the case that the desired width $k$ is bigger than the dimension of the data $d$ we have:

\begin{theorem}\label{thm:upper bound memorization big k, q gen}
    Let $p\in (0, \infty]$, $q\in (1,\infty)\setminus\mathbb{N}$. If $d+12 \leq k$ and 
    \begin{equation*}    
        \frac{\sigma}{\delta} < \frac{1}{2c^{+}_{p,q}(d)}
    \end{equation*}
    then, for every $(\delta,q)$-separated dataset $\mathcal{D}\in\mathcal{D}_{d,N,C}(\delta,q)$, there exists a neural network $f:\mathbb{R}^d\rightarrow \mathbb{R}$ with width $k$ and depth 
    \begin{equation*}
        O\left(Nd^{1+\frac{1}{q}}\lambda q\left(\log_{2}(dq\lambda^{q})+q\right)\right)
    \end{equation*}
    that $(\sigma, p)$-robustly memorizes the dataset $\mathcal{D}$. 
\end{theorem}

In the case that $q\in\mathbb{N}$ an improved result can be obtained:

\begin{theorem}\label{thm:upper bound memorization big k, q natural}
    Let $p\in (0, \infty]$, $q\in \mathbb{N}_{\geq 2}$. If $d+9 \leq k$ and 
    \begin{equation*}    
        \frac{\sigma}{\delta} < \frac{1}{2c^{+}_{p,q}(d)}
    \end{equation*}
    then, for every $(\delta,q)$-separated dataset $\mathcal{D}\in\mathcal{D}_{d,N,C}(\delta,q)$, there exists a neural network $f:\mathbb{R}^d\rightarrow \mathbb{R}$ with width $k$ and depth 
    \begin{equation*}
        O\left(Ndq \log_{2}\left(dq\lambda^{q}\right)\right)
    \end{equation*}
    that $(\sigma, p)$-robustly memorizes the dataset $\mathcal{D}$.
\end{theorem}

Note that for $q=2$ we recover \thmref{thm:upper bound memorization big k}. In the special case that $q\in\{1, \infty\}$ the range of the width can be improved, and the log term in the depth can be removed:

\begin{theorem}\label{thm:upper bound memorization big k, q 1 inf}
    Let $p\in (0, \infty]$, $q\in \{1, \infty\}$. If $d+4 \leq k$ and 
    \begin{equation*}    
        \frac{\sigma}{\delta} < \frac{1}{2c^{+}_{p,q}(d)}
    \end{equation*}
    then, for every $(\delta,q)$-separated dataset $\mathcal{D}\in\mathcal{D}_{d,N,C}(\delta,q)$, there exists a neural network $f:\mathbb{R}^d\rightarrow \mathbb{R}$ with width $k$ and depth $O\left(Nd\right)$ that $(\sigma, p)$-robustly memorizes the dataset $\mathcal{D}$.
\end{theorem}

Note that if we allow the range of the width $k$ to be $3d+1\leq k$ then in the case that $p=q\in\{1,\infty\}$, adjusting the construction in the proof of \thmref{thm:upper bound memorization big k, q 1 inf} according to \remarkref{remark:wider construction better depth}, the depth can be improved to be $O(N)$, thus recovering the result in \cite[Theorem~4.8]{yuoptimal}.

In the special case that $p=q\in\mathbb{N}_{\geq 2}$, \cite[Theorem~B.6]{yuoptimal} obtained the following result:
\begin{quote}
    Let $p=q\in \mathbb{N}_{\geq 2}$. If $k>cd$ for some universal constant $c$ and 
    \begin{equation*}    
        \frac{\sigma}{\delta} < \frac{1}{2}
    \end{equation*}
    then, for every $(\delta,q)$-separated dataset $\mathcal{D}\in\mathcal{D}_{d,N,C}(\delta,q)$, such that $\mathcal{D}\subseteq [-\Delta, \Delta]^{d}$, there exists a neural network $f:\mathbb{R}^d\rightarrow \mathbb{R}$ with width $k$ and depth 
    \begin{equation}\label{eq:p=q depth in lit}
        O\left(Nq\log_{2}\left(qd\left(\frac{\Delta}{\delta/2-\sigma}\right)^{q}\right)\right)
    \end{equation}
    that $(\sigma, q)$-robustly memorizes the dataset $\mathcal{D}$.
\end{quote}
By changing the range of $k$ from $d+9\leq k$ to $8d+1\leq k$, and adjusting our construction for \thmref{thm:upper bound memorization big k, q natural} according to \remarkref{remark:wider construction better depth}, we would get the following result:
\begin{quote}
    Let $p=q\in \mathbb{N}_{\geq 2}$. If $k\geq 8d+1$ and 
    \begin{equation*}    
        \frac{\sigma}{\delta} < \frac{1}{2}
    \end{equation*}
    then, for every $(\delta,q)$-separated dataset $\mathcal{D}\in\mathcal{D}_{d,N,C}(\delta,q)$, there exists a neural network $f:\mathbb{R}^d\rightarrow \mathbb{R}$ with width $k$ and depth 
    \begin{equation}\label{eq:p=q depth in our construction}
        O\left(Nq\log_{2}\left(qd\left(\frac{\delta/2}{\delta/2-\sigma}\right)^{q}\right)\right)
    \end{equation}
    that $(\sigma, q)$-robustly memorizes the dataset $\mathcal{D}$.
\end{quote}
The depth in \eqref{eq:p=q depth in lit} depends on the global spread of the data through the quantity $\Delta$, whereas the depth in \eqref{eq:p=q depth in our construction} is favorable since it is only affected by the local structure through the relation between $\delta$ and $\sigma$. In particular, for any two datasets with domains $\Delta_{1}, \Delta_{2}$, that share values for $\delta, \sigma$ we will obtain memorizing networks of the same complexity regardless of the size of the domains $\Delta_{1}$ and $\Delta_{2}$ of the datasets.

We conclude by remarking that in all of these cases, if one does not care about the complexity of the depth, a network with near optimal width always exists:
\begin{proposition}\label{prop:upper bound memorization big k, big depth}
    Let $p\in (0, \infty]$, $q\in [1,\infty]$. If $d+1 \leq k$ and 
    \begin{equation*}    
        \frac{\sigma}{\delta} < \frac{1}{2c^{+}_{p,q}(d)}
    \end{equation*}
    then, for every $(\delta,q)$-separated dataset $\mathcal{D}\in\mathcal{D}_{d,N,C}(\delta,q)$, such that $\mathcal{D}\subseteq B^{d}_{2}(0,\Delta)$, there exists a neural network $f:\mathbb{R}^d\rightarrow \mathbb{R}$ with width $k$ and depth 
    \begin{equation*}
        O\left(\frac{C\Delta}{\delta - 2c^{+}_{p,q}(d)\sigma}\right)^{d+1}
    \end{equation*}
    that $(\sigma, p)$-robustly memorizes the dataset $\mathcal{D}$.
\end{proposition}

A proof sketch for \propref{prop:upper bound memorization big k, big depth} appears in \ref{prf:sketch for deep upper bound}.

\subsection{Robust Memorization With Small Width}

In the case where the desired width is smaller than the data dimension we can obtain similarly to \thmref{thm:upper bound memorization} the following result for general $l_{q}$ norm:

\begin{theorem}\label{thm:upper bound memorization q gen}
    Let $p\in (0, \infty]$, $q\in [1, \infty]$ and denote $a_{p,q,d}=\frac{1}{8\sqrt{e}}d^{-\frac{1}{2}+\left[\frac{1}{2}-\frac{1}{q}\right]_{-}+\left[\frac{1}{p}-\frac{1}{2}\right]_{-}}$.\\
    If $7\leq k \leq d+5$ and
    \begin{equation*}    
        \frac{\sigma}{\delta} \leq a_{p,q,d}N^{-\frac{2}{k-6}}
    \end{equation*}
    then, for every $(\delta, q)$-separated dataset $\mathcal{D}\in\mathcal{D}_{d,N,C}(\delta, q)$, there exists a neural network $f:\mathbb{R}^d\rightarrow \mathbb{R}$ with width $k$ and depth $O\left(Nk\log_{2}\left(k\right)\right)$ that $(\sigma, p)$-robustly memorizes the dataset $\mathcal{D}$.
\end{theorem}

Note that for $q=2$ we recover \thmref{thm:upper bound memorization}. Similar to \thmref{thm:lower bound memorization}, we also have a lower bound:

\begin{theorem}\label{thm:lower bound memorization q gen}
    Let $p\in (0, \infty]$, $q\in [1, \infty]$ and denote $b_{p,q,d}=2416d^{\left[\frac{1}{2}-\frac{1}{q}\right]_{+} + \left[\frac{1}{p}-\frac{1}{2}\right]_{+}}$.\\
    If $1\leq k \leq d-1$ and
    \begin{equation*}
        \frac{\sigma}{\delta} > b_{p,q,d}N^{-\frac{2}{k}} 
    \end{equation*}
    then, there exists a $(\delta,q)$-separated dataset $\mathcal{D}\in \mathcal{D}_{d,N,2}(\delta,q)$ such that every neural network $f:\mathbb{R}^d\rightarrow \mathbb{R}$ with width $k$ and any depth cannot $(\sigma, p)$-robustly memorize the dataset $\mathcal{D}$.
\end{theorem}

Again, for $q=2$ we recover \thmref{thm:lower bound memorization}. To prove the bounds in Theorems \ref{thm:upper bound memorization q gen}, \ref{thm:lower bound memorization q gen} we reduce (resp. enlarge) $\delta$ by a factor of $c^{-}_{q,2}(d)$ (resp. $c^{+}_{q,2}(d)$) so that the relation between the norms $l_{2},l_{q}$ obtained from \lemref{lma:lp lq relations} would enable us to deal with datasets separated under $l_{2}$ norm and then use Theorems \ref{thm:upper bound memorization}, \ref{thm:lower bound memorization}. This scaling of $\delta$ results in constants $a_{p,q,d}=a_{p,d}\cdot c^{-}_{q,2}(d)=a_{p,d}\cdot d^{\left[\frac{1}{2}-\frac{1}{q}\right]_{-}}$, and $b_{p,q,d}=b_{p,d}\cdot c^{+}_{q,2}(d)=b_{p,d}\cdot d^{\left[\frac{1}{2}-\frac{1}{q}\right]_{+}}$ where $a_{p,d}, b_{p,d}$ are the constant in Theorems \ref{thm:upper bound memorization}, \ref{thm:lower bound memorization} respectively. The gap between $a_{p,q,d}$ and $b_{p,q,d}$ is thus given by 
\begin{equation*}
    \frac{b_{p,q,d}}{a_{p,q,d}}=\frac{b_{p,d}\cdot d^{\left[\frac{1}{2}-\frac{1}{q}\right]_{+}}}{a_{p,d}\cdot d^{\left[\frac{1}{2}-\frac{1}{q}\right]_{-}}}=c\cdot d^{\frac{1}{2}}\cdot d^{\lvert \frac{1}{2}-\frac{1}{p} \rvert}\cdot d^{\lvert \frac{1}{2}-\frac{1}{q} \rvert}
\end{equation*}
for $c=19328\sqrt{e}$. When $1\leq p$ we have $ d^{\lvert \frac{1}{2}-\frac{1}{p} \rvert}\cdot d^{\lvert \frac{1}{2}-\frac{1}{q} \rvert}\leq d$ and so the gap is bounded by $c\cdot d^{\frac{3}{2}}$. In particular, this would be the case when $p=q$ (since $q\in [1,\infty]$).

\section{Proofs for \secref{sec:Separation in $l_{q}$ norm}}\label{sec:proofs for lq sep section}

\subsection{Robust Memorization With Large Width}

\begin{proof}[Proof of \thmref{thm:upper bound memorization big k, q gen}]
    Denote $r=c^{+}_{p,q}(d)\sigma$ and let $\mathcal{D}\in\mathcal{D}_{d,N,C}(\delta,q)$ be a $(\delta, q)$-separated dataset. We know that $r<\delta / 2$ and so there exists some $0<w$ such that $\delta = 2r+ w$. Denote $\Tilde{\epsilon}=\frac{w^q}{4d(w+2r)^q}=\frac{(\delta - 2r)^q}{4d\delta^q}$ then $0<\Tilde{\epsilon}<1/2$, $q\in (1,\infty)\setminus\mathbb{N}$ and so from \lemref{lma:net p power gen} we get that there exists a neural network $g_{\Tilde{\epsilon},q}:\mathbb{R}\rightarrow\mathbb{R}$ with width $9$ and depth
    \begin{equation*}
        O\left(q\Tilde{\epsilon}^{-\frac{1}{q}}\left(\log_{2}(q\Tilde{\epsilon}^{-1})+q\right)\right) = O\left(d^{\frac{1}{q}}\lambda q\left(\log_{2}(dq\lambda^{q})+q\right)\right)~,
    \end{equation*}
    (where $\lambda=\left(1-\frac{2r}{\delta}\right)^{-1}$) such that $|g_{\Tilde{\epsilon},q}(\alpha)-\alpha^{q}| \leq \Tilde{\epsilon}$ for every $\alpha\in [0,1]$. Using this network, from \lemref{lma:net ball indicator} we get that for every $(x_{i},y_{i})\in\mathcal{D}$ there exists a neural network $f_{x_{i},w,q}:\mathbb{R}^{d}\rightarrow\mathbb{R}$ with width $W=d+2+\mathcal{W}(g_{\Tilde{\epsilon},q})=d+11$, and depth $L=O\left(d\mathcal{L}(g_{\Tilde{\epsilon},q})\right)=O\left(d^{1+\frac{1}{q}}\lambda q\left(\log_{2}(dq\lambda^{q})+q\right)\right)$, such that for all $x \in\mathbb{R}^{d}$ we have $f_{x_{i},w,q}(x)\leq y_{i}$ and     
    \[ f_{x_{i},w,q}(x)=  \begin{cases} 
      y_{i} & \lVert x - x_{i} \rVert_{q}\leq r \\
      0 &   r+w\leq \lVert x - x_{i} \rVert_{q} \\
   \end{cases}~,
    \]
    Finally, because $\mathcal{D}$ is $(\delta,q)$-separated, from \thmref{thm:full width memo} there exists a neural network $F_{d,\delta,r, q}:\mathbb{R}^{d}\rightarrow\mathbb{R}$ with width $d+12$ and depth $O\left(Nd^{1+\frac{1}{q}}\lambda q\left(\log_{2}(dq\lambda^{q})+q\right)\right)$ that $(r,q)$-robustly memorizes the dataset $\mathcal{D}$. Define $f=F_{d,\delta,r,q}$ and let $x\in B^{d}_{p}(x_{i},\sigma)$. Then, by definition of $c^{+}_{p,q}(d)$ and \lemref{lma:lp ball and lq ball}, we have $x\in B^{d}_{q}(x_{i},r)$ and so $f(x)=F_{d,\delta,r, q}(x)=y_{i}$. Now $d+12\leq k$ and so by padding each hidden layer of $f$ with $k-(d+12)$ neurons we obtain $f$ with width $k$ and depth 
    \begin{equation*}
        O\left(Nd^{1+\frac{1}{q}}\lambda q\left(\log_{2}(dq\lambda^{q})+q\right)\right)
    \end{equation*}
    that $(\sigma, p)$-robustly memorizes the dataset $\mathcal{D}$.
\end{proof}

\begin{proof}[Proof of \thmref{thm:upper bound memorization big k, q natural}]
    Following the exact same proof as the proof of \thmref{thm:upper bound memorization big k, q gen} where instead of \lemref{lma:net p power gen} we use \lemref{lma:net p power natural}, we obtain a neural network $F_{d,\delta,r, q}:\mathbb{R}^{d}\rightarrow\mathbb{R}$ with width $d+9$ and depth $O\left(Ndq\log_{2}(dq\lambda^{q})\right)$ that $(r,q)$-robustly memorizes the dataset $\mathcal{D}$. Define $f=F_{d,\delta,r,q}$. Now $d+9\leq k$ and so by padding each hidden layer of $f$ with $k-(d+9)$ neurons we obtain $f$ with width $k$ and depth $O\left(Ndq\log_{2}(dq\lambda^{q})\right)$ that $(\sigma, p)$-robustly memorizes the dataset $\mathcal{D}$. 
\end{proof}

\begin{proof}[Proof of \thmref{thm:upper bound memorization big k, q 1 inf}]
    We prove for $q=1$ and $q=\infty$: 
    \begin{itemize}
        \item Case $q=1$:
        Follow the proof of \thmref{thm:upper bound memorization big k, q gen} where instead of \lemref{lma:net p power gen} use the identity map $g_{1}(\alpha)=\alpha$ with width and depth of $1$. We obtain a neural network $F_{d,\delta,r,1}:\mathbb{R}^{d}\rightarrow\mathbb{R}$ with width $d+4$ and depth $O\left(Nd\right)$ that $(r,1)$-robustly memorizes the dataset $\mathcal{D}$. Define $f=F_{d,\delta,r,1}$. Now $d+4\leq k$ and so by padding each hidden layer of $f$ with $k-(d+4)$ neurons we obtain $f$ with width $k$ and depth $O\left(Nd\right)$ that $(\sigma, p)$-robustly memorizes the dataset $\mathcal{D}$.
        \item Case $q=\infty$:
        Follow the proof of \thmref{thm:upper bound memorization big k, q gen} where instead of using \lemref{lma:net p power gen} and \lemref{lma:net ball indicator} use \lemref{lma:net ball indicator p=infty}. We get that for every $(x_{i},y_{i})\in\mathcal{D}$ and every $0<w$, there exists a neural network $f_{x_{i},w,\infty}:\mathbb{R}^{d}\rightarrow\mathbb{R}$ with width $W=d+3$, and depth $L=O(d)$, such that for all $x \in\mathbb{R}^{d}$ we have $f_{x_{i},w,\infty}(x)\leq y_{i}$ and     
        \[ f_{x_{i},w,\infty}(x)=  \begin{cases} 
          y_{i} & \lVert x - x_{i} \rVert_{\infty}\leq r \\
          0 &   r+w\leq \lVert x - x_{i} \rVert_{\infty} \\
       \end{cases}~,
        \]
        Finally, because $\mathcal{D}$ is $(\delta,\infty)$-separated, from \thmref{thm:full width memo} we obtain a neural network $F_{d,\delta,r,\infty}:\mathbb{R}^{d}\rightarrow\mathbb{R}$ with width $d+4$ and depth $O\left(Nd\right)$ that $(r,\infty)$-robustly memorizes the dataset $\mathcal{D}$. Define $f=F_{d,\delta,r,1}$ and conclude that $f$ has width $k$ and depth $O\left(Nd\right)$ and it $(\sigma, p)$-robustly memorizes the dataset $\mathcal{D}$.
    \end{itemize}
\end{proof}

\subsection{Robust Memorization With Small Width}

\begin{proof}[Proof of \thmref{thm:upper bound memorization q gen}]
    Let $\frac{\sigma}{\delta} \leq \frac{c^{-}_{q,2}(d)}{2c^{+}_{p,2}(d)}\sqrt{\frac{k-6}{16ed}}N^{-\frac{2}{k-6}}$ and let $\mathcal{D}\in\mathcal{D}_{d,N,C}(\delta,q)$ be a $(\delta,q)$-separated dataset. Denote $\sigma^{\prime}:= c^{+}_{p,2}(d)\sigma$ and $\delta^{\prime}=c^{-}_{q,2}(d)\delta$, then $\frac{2\sigma^{\prime}}{\delta^{\prime}} \leq \frac{1}{4\sqrt{e}}\sqrt{\frac{k-6}{d}}N^{-\frac{2}{k-6}}$ and from \lemref{lma:lp lq relations} $\mathcal{D}\in\mathcal{D}_{d,N,C}(\delta^{\prime},2)$. Therefore, from \thmref{thm:ortho pres pos} we have that $\mathcal{D}$ is $(\sigma^{\prime},\epsilon, k-6)$-orthogonally preservable with $\epsilon=\frac{1}{2\sqrt{e}}\sqrt{\frac{k-6}{d}}N^{-\frac{2}{k-6}}$. Note that $\frac{1}{\epsilon} < \frac{\delta^{\prime}}{2\sigma^{\prime}}$ and so from \thmref{thm:memorization of preservable data pos} and \lemref{lma:eq def of preservability} we conclude that there exists a neural network $f:\mathbb{R}^d\rightarrow \mathbb{R}$ with width $k$ and depth 
    \begin{equation*}
        O\left(Nk\log_{2}\left(\frac{k}{1-\frac{2\sigma^{\prime}}{\epsilon\delta^{\prime}}}\right)\right)
    \end{equation*}
    that $(\sigma,p)$-robustly memorizes the dataset $\mathcal{D}$.
    
    Now $\frac{2c^{+}_{p,2}(d)\sigma}{c^{-}_{q,2}(d)\delta} \leq \frac{1}{2}\epsilon$ so $\frac{2\sigma^{\prime}}{\epsilon\delta^{\prime}}\leq \frac{1}{2}$ and hence the depth of $f$ is $O\left(Nk\log_{2}\left(k\right)\right)$. The theorem follows by noting that $\frac{1}{8\sqrt{e}}d^{-\frac{1}{2}+\left[\frac{1}{2}-\frac{1}{q}\right]_{-}+\left[\frac{1}{p}-\frac{1}{2}\right]_{-}} \leq \frac{c^{-}_{q,2}(d)}{2c^{+}_{p,2}(d)}\sqrt{\frac{k-6}{16ed}}$.
\end{proof}

\begin{proof}[Proof of \thmref{thm:lower bound memorization q gen}]
    Let $\frac{\sigma}{\delta} > \frac{c^{+}_{q,2}(d)}{c^{-}_{p,2}(d)}2416N^{-\frac{2}{k}}$ and denote $\sigma^{\prime}:= c^{-}_{p,2}(d)\sigma$ and $\delta^{\prime}=c^{+}_{q,2}(d)\delta$, then $\frac{2\sigma^{\prime}}{\delta^{\prime}} > 4832N^{-\frac{2}{k}}$ and so from  \thmref{thm:gen pres neg} we get that there exists a $(\delta^{\prime}, 2)$-separated dataset $\mathcal{D}\in\mathcal{D}_{d,N,2}(\delta^{\prime}, 2)$ which is not $(\sigma^{\prime}, k)$-preservable. Therefore, by \thmref{thm:memorization of preservable data neg} we get that there isn't a neural network $f$ with width equal to $k$ that $(\sigma,p)$-robustly memorizes the dataset $\mathcal{D}$. Note that from \lemref{lma:lp lq relations}, $\mathcal{D}\in\mathcal{D}_{d,N,2}(\delta, q)$ and so by noting that $2416d^{\left[\frac{1}{2}-\frac{1}{q}\right]_{+} + \left[\frac{1}{p}-\frac{1}{2}\right]_{+}}=\frac{c^{+}_{q,2}(d)}{c^{-}_{p,2}(d)}2416$ we are done.
\end{proof}

\section{Lemmas Used for Network Approximations}

\begin{theorem}\label{thm:full width memo}
    Let $p\in [1, \infty]$, $r <\tau /2$ and let $\mathcal{D}\in\mathcal{D}_{k,N,C}$ be a dataset such that for all $x_{i},x_{j}$ with $y_{i}\neq y_{j}$ we have $\tau \leq \lVert x_{i}-x_{j} \rVert_{p}$. Denote $w=\tau - 2r$. Assume that for every $(x_{i},y_{i})\in \mathcal{D}$ there exists a neural network $\hat{f}_{x_{i},w,p}:\mathbb{R}^{k}\rightarrow\mathbb{R}^{k+1}$ with width $W_{w,p}$ and depth $L_{w,p}$ such that for all $x \in\mathbb{R}^{k}$ we have $\hat{f}_{x_{i},w,p}(x)=(f_{x_{i},w,p}(x),x)$, $f_{x_{i},w,p}(x)\leq y_{i}$ and     
    \[ f_{x_{i},w,p}(x)=  \begin{cases} 
      y_{i} & \lVert x - x_{i} \rVert_{p}\leq r \\
      0 &   r+w\leq \lVert x - x_{i} \rVert_{p} \\
   \end{cases}~,
    \]
    Then there exists a neural network $F_{k,\tau,r,p}:\mathbb{R}^{k}\rightarrow\mathbb{R}$ with width $W_{w,p}+1$ and depth $O\left(NL_{w,p}\right)$ that $(r,p)$-robustly memorizes the dataset $\mathcal{D}$.
\end{theorem}

\begin{proof}[Proof of \thmref{thm:full width memo}]
    For every data point $(x_{i},y_{i})\in \mathcal{D}$, we will use the (approximate) indicator network $f_{x_{i},w,p}$ in order to inspect whether the input $x$ lies inside the ball $B^{k}_{p}(x_{i},r)$, and keep the answer in a neuron denoted by $z_{i}$. The construction is done in a way that ensures that if $x\in B^{k}_{p}(x_{j},r)$  for some $j$, then $\max\left\{z_{i} \mid 1\leq i \leq N\right\}=y_{j}$. Hence, by keeping track of the maximum value of the $z_{i}$'s up to $j$ (which we will denote by $m_{j}$) we will manage to return the desired result (by returning $m_{N}$).

    Let $1\leq i \leq N$. Using $\hat{f}_{x_{i},w,p}(x)=(f_{x_{i},w,p}(x),x)$ from the assumption of the theorem, we will compute it in a dedicated sub-net $A_{i}$ (see \figref{fig:net full width memo Ai}) and update the running maximum value $m_{i}$.
    \begin{figure}[H]
    \centering
    \[\begin{tikzcd}[cramped]
	{m_{i-1}} && {m_{i-1}\rightarrow...\rightarrow m_{i-1}} & {m_{i-1}} & {m_{i-1}} & {m_{i}} \\
	&& {\boxed{\Large \qquad \hat{f}_{x_{i},w,p \qquad}}} & {z_{i}} & {[z_{i}-m_{i-1}]_{+}} \\
	x &&& x & x & x
	\arrow[from=3-1, to=2-3]
	\arrow[from=1-1, to=1-3]
	\arrow[from=2-3, to=3-4]
	\arrow[from=2-3, to=2-4]
	\arrow[from=1-3, to=1-4]
	\arrow[from=2-4, to=2-5]
	\arrow[from=1-4, to=2-5]
	\arrow[from=2-5, to=1-6]
	\arrow[from=1-5, to=1-6]
	\arrow[from=3-4, to=3-5]
	\arrow[from=3-5, to=3-6]
	\arrow[from=1-4, to=1-5]
    \end{tikzcd}\]
    \caption{The architecture of $A_{i}$}
    \label{fig:net full width memo Ai}
    \end{figure}
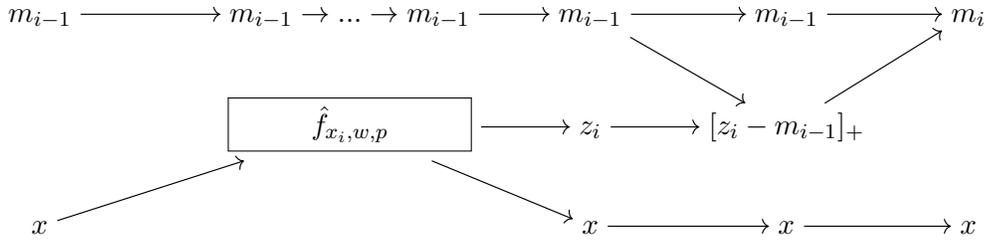

    $A_{i}$ will preform the following: 
    \begin{itemize}
        \item append a neuron with the value $m_{i-1}$, where $m_{0}=0$ (this is the accumulated maximum of previous indicator computations).
        \item compute $z_{i}=f_{x_{i}, w, p}(x)$ (current indicator computation).
        \item compute $[z_{i}-m_{i-1}]_{+}$.
        \item compute $m_{i}=[m_{i-1}+[z_{i}-m_{i-1}]_{+}]_{+}$ ($m_{i}=\max\left\{z_{j} \mid j\leq i\right\}$).
    \end{itemize}
    Note that the width and depth of $A_{i}$ satisfy $\mathcal{W}(A_{i})=W_{w,p}+1$ and $\mathcal{L}(A_{i})=O(L_{w,p})$.
    Now, we will define $F_{k,\tau,r,p}(x)$ to simply return $m_{N}$ (see \figref{fig:net full width memo F}).

    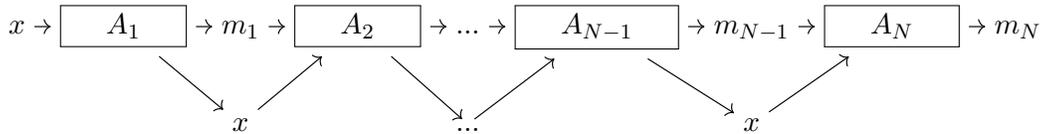
\begin{figure}[H]
    \centering
    \[\begin{tikzcd}[cramped,column sep=tiny]
	x & {\boxed{\Large \quad A_{1} \quad}} & {m_{1}} & {\boxed{\Large \quad A_{2} \quad}} & {...} & {\boxed{\Large \quad A_{N-1} \quad}} & {m_{N-1}} & {\boxed{\Large \quad A_{N} \quad}} & {m_{N}} \\
	&& x && {...} && x
	\arrow[from=1-1, to=1-2]
	\arrow[from=1-2, to=2-3]
	\arrow[from=1-2, to=1-3]
	\arrow[from=2-3, to=1-4]
	\arrow[from=1-3, to=1-4]
	\arrow[from=1-6, to=1-7]
	\arrow[from=1-6, to=2-7]
	\arrow[from=2-7, to=1-8]
	\arrow[from=1-7, to=1-8]
	\arrow[from=1-8, to=1-9]
	\arrow[from=1-4, to=1-5]
	\arrow[from=1-4, to=2-5]
	\arrow[from=2-5, to=1-6]
	\arrow[from=1-5, to=1-6]
    \end{tikzcd}\]
    \caption{The architecture of $F_{k,\tau,r}$}
    \label{fig:net full width memo F}
    \end{figure}

    Note that by construction the maximal width of $F_{k,\tau,r,p}$ is the maximal width of $A_{i}$ so $F_{k,\tau,r,p}:\mathbb{R}^{k}\rightarrow\mathbb{R}$ is a neural network of width $W_{w,p}+1$ and depth $\mathcal{L}(F_{k,\tau,r,p})=O(NL_{w,p})$.

    Let us show that $F_{k,\tau,r,p}(x)$ indeed $(r,p)$-robustly memorizes the dataset $\mathcal{D}$. Let $i\in [N]$ and let $x\in B^{k}_{p}(x_{i},r)$. We inspect the possible values of $f_{x_{j}, w, p}$ for every $j\in \left[N\right]$:
    \begin{itemize}
        \item If $j=i$: In this case $\lVert x - x_{j} \rVert_{p}=\lVert x - x_{i} \rVert_{p}\leq r$ so $f_{x_{j}, w, p}(x)=y_{j}=y_{i}$.
        \item If $j\neq i$: 
        \begin{itemize}
            \item If $y_{i}=y_{j}$: By definition $f_{x_{j}, w, p}(x)\leq y_{j}=y_{i}$.
            \item If $y_{i}\neq y_{j}$: Note that in this case since $x$ is in a different class than that of $x_{j}$ (and because $\lVert \cdot \rVert_{p}$ is a norm as $p\in [1, \infty]$) we have from the triangle inequality $\tau - r \leq \lVert x - x_{j} \rVert_{p}$. Hence, in this case $r+w=\tau - r\leq \lVert x - x_{j} \rVert_{p}$ so $f_{x_{j}, w, p}(x)=0$.
        \end{itemize}  
    \end{itemize}
    We see that when $j\neq i$ one has $z_{j}=f_{x_{j}, w, p}(x)\leq y_{i}$ and when $i=j$, $z_{i}=f_{x_{i}, w, p}(x)=y_{i}$, so $\max\left\{z_{j} \mid 1\leq j \leq N\right\}=y_{i}$, but by construction we have $m_{j}=\max\left\{z_{j^{\prime}} \mid j^{\prime}\leq j\right\}$ so we conclude that $F_{k,\tau,r,p}(x)=m_{N}=y_{i}$. We have thus shown that $F_{k,\tau,r,p}$ does indeed $(r,p)$-robustly memorizes the dataset $\mathcal{D}$.
\end{proof}

The following lemma provides a construction of an approximate indicator function that returns a desired value on a fixed $l_{p}$ ball.

\begin{lemma}\label{lma:net ball indicator}
    Let $p\in[1,\infty)$, $x_{0}\in\mathbb{R}^{k}$, $0 < y_{0}$ and let $0<r$ and any $0<w$. Denote $\Tilde{\epsilon}=\frac{w^p}{4k(w+2r)^p}$. Assume that there exists a network $g_{\Tilde{\epsilon}, p}$ with width $\mathcal{W}(g_{\Tilde{\epsilon}, p})$ and depth $\mathcal{L}(g_{\Tilde{\epsilon}, p})$ such that $\lvert g_{\Tilde{\epsilon},p}(t) - t^{p} \rvert < \Tilde{\epsilon}$ for every $t\in [0,1]$. Then there exists a neural network $f_{x_{0},w,p}:\mathbb{R}^{k}\rightarrow\mathbb{R}$ with width $k+2+\mathcal{W}(g_{\Tilde{\epsilon},p})$ and depth $O\left(k\mathcal{L}(g_{\Tilde{\epsilon},p})\right)$ such that for all $x \in\mathbb{R}^{k}$ we have $f_{x_{0},w,p}(x)\leq y_{0}$ and

    \[ f_{x_{0},w,p}(x)=  \begin{cases} 
      y_{0} & ,\lVert x - x_{0} \rVert_{p}\leq r \\
      0 &   ,r+w\leq \lVert x - x_{0} \rVert_{p} \\
   \end{cases}~,
    \]
    Furthermore, $f_{x_{0},w,p}$ can be modified to return also the input vector $x$ without changing its width and depth.
\end{lemma}

Similar analysis yields the following lemma for the case that $p=\infty$:

\begin{lemma}\label{lma:net ball indicator p=infty}
    Let $x_{0}\in\mathbb{R}^{k}$, $0 < y_{0}$ and let $0<r$ and any $0<w$. Then there exists a neural network $f_{x_{0}, w, \infty}:\mathbb{R}^{k}\rightarrow\mathbb{R}$ with width $k+3$ and depth $O(k)$ such that for all $x \in\mathbb{R}^{k}$ we have $f_{x_{0},w,\infty}(x)\leq y_{0}$ and

    \[ f_{x_{0},w,\infty}(x)=  \begin{cases} 
      y_{0} & ,\lVert x - x_{0} \rVert_{\infty}\leq r \\
      0 &   ,r+w\leq \lVert x - x_{0} \rVert_{\infty} \\
   \end{cases}~,
    \]
    Furthermore, $f_{x_{0},w,\infty}$ can be modified to return also the input vector $x$ without changing its width and depth.
\end{lemma}

\begin{proof}[Proof of \lemref{lma:net ball indicator}]
    For a vector $v$ we will denote by $(v)_{j}$ its $j$-th coordinate. We will use $g_{\Tilde{\epsilon}, p}$ to approximate the values of $\lvert (x)_{j} - (x_{0})_{j} \rvert^{p}$ for $1\leq j\leq k$, and use them to approximate the $p$ norm $\lVert x - x_{0} \rVert^{p}_{p}$. This, in 
    turn will allow us to return the desired result. Let $1\leq j \leq k$. We would like to compute $\lvert (x)_{j} - (x_{0})_{j} \rvert^{p}$ using $g_{\Tilde{\epsilon}, p}$, so we have to modify the input so that it lies in the range $[0, 1]$. $\lvert (x)_{j} - (x_{0})_{j} \rvert$ is unbounded in general so we have to normalize it carefully.
    Denote $\delta=(w+2r)$, and the following functions of $x\in\mathbb{R}^k$:
    \begin{equation*}
    \begin{aligned}
        &&(a)_{j}&=(x)_{j}-(x_{0})_{j}&\\
        &&(b)_{j}&=\left[\left(\left[(a)_{j}\right]_{+}+\left[-(a)_{j}\right]_{+}\right)/\delta\right]_{+}&\\
        &&(c)_{j}&=\left[2(b)_{j}-1\right]_{+}&~.
    \end{aligned}     
    \end{equation*}
    We now show some properties of these quantities.
    \begin{enumerate}
        \item $(b)_{j}=\frac{\lvert (x)_{j} - (x_{0})_{j} \rvert}{\delta}$, $(c)_{j}=\left[\frac{2\lvert (x)_{j} - (x_{0})_{j} \rvert}{\delta}-1\right]_{+}$:
        
        Follows immediately from the definition of $a,b,c$.
        \item If $\lvert (x)_{j} - (x_{0})_{j} \rvert \leq \delta$, then  $0 \leq (b)_{j}\leq 1$:

        Follows immediately from the definition of $a,b$.
        \item If $\lvert (x)_{j} - (x_{0})_{j} \rvert \leq r$, then  $(c)_{j} = 0$:

        We have $\frac{2\lvert (x)_{j} - (x_{0})_{j} \rvert}{\delta}-1\leq \frac{2r}{\delta}-1=\frac{2r}{2r+w}-1\leq 0$, so $(c)_{j}=0$.
        \item If $\delta \leq \lvert (x)_{j} - (x_{0})_{j} \rvert$, then $(c)_{j} \geq 1$:

        We have $\frac{2\lvert (x)_{j} - (x_{0})_{j} \rvert}{\delta}-1\geq \frac{2\delta}{\delta}-1=1$.
    \end{enumerate}
    We will now construct $f_{x_{0},w,p}$ using these quantities. From the assumption of the theorem, there exists a neural network $g_{\Tilde{\epsilon},p}:\mathbb{R}\rightarrow\mathbb{R}$ with width $\mathcal{W}(g_{\Tilde{\epsilon},p})$ and depth $\mathcal{L}(g_{\Tilde{\epsilon},p})$ such that $\lvert g_{\Tilde{\epsilon},p}(t) - t^{p} \rvert < \Tilde{\epsilon}$ for every $t\in [0,1]$. For every coordinate $j$ we apply the sub-network $M_{j}$ (see \figref{fig:net ball indicator Mj} for layout of $M_{j}$) which preforms the following:
    \begin{itemize}
        \item compute the penalty $(c)_{j}$ and the normalized input $(b)_{j}$.
        \item compute $(\gamma)_{j}=g_{\Tilde{\epsilon},p}((b)_{j})$ (this will approximate $(b)^{p}_{j}$ for relevant values of $(x)_{j}$).
        \item compute $(\eta)_{j}=[\delta^{p}(\gamma)_{j}]_{+}$ (this will approximate $\lvert (x)_{j} - (x_{0})_{j} \rvert^{p}$ for relevant values of $(x)_{j}$).
        \item compute the neuron $(s)_{j}=(\eta)_{j}+(r^{p}+\frac{3}{4}w^{p})(c)_{j}$ (current coordinate to the power $p$ with its penalty).
    \end{itemize}

    \begin{figure}[H]
    \centering
    \[\begin{tikzcd}[cramped,column sep=small]
	& {-(a)_{j}} && {(c)_{j}} & {(c)_{j}} & {(c)_{j}} \\
	{(a)_{j}} & {(a)_{j}} & {(b)_{j}} & {\boxed{\Large \qquad g_{\Tilde{\epsilon},p} \qquad}} & {(\gamma)_{j}} & {(\eta)_{j}} & {(s)_{j}}
	\arrow[from=1-2, to=2-3]
	\arrow[from=1-4, to=1-5]
	\arrow[from=1-5, to=1-6]
	\arrow[from=1-6, to=2-7]
	\arrow[from=2-1, to=1-2]
	\arrow[from=2-1, to=2-2]
	\arrow[from=2-2, to=2-3]
	\arrow[from=2-3, to=1-4]
	\arrow[from=2-3, to=2-4]
	\arrow[from=2-4, to=2-5]
	\arrow[from=2-5, to=2-6]
	\arrow[from=2-6, to=2-7]
    \end{tikzcd}\]
    \caption{The architecture of $M_{j}$}
    \label{fig:net ball indicator Mj}
    \end{figure}
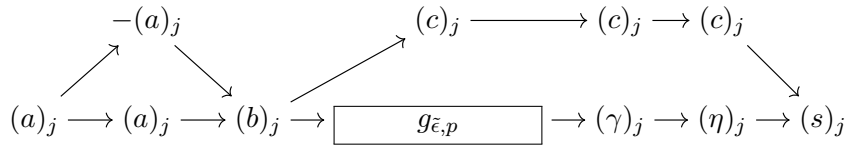

    We note that the width and depth of $M_{j}$ satisfy $\mathcal{W}(M_{j})=1+\mathcal{W}(g_{\Tilde{\epsilon}, p})$, and $\mathcal{L}(M_{j})=5+\mathcal{L}(g_{\Tilde{\epsilon},p})$. In order to compute $M_{j}$ for every coordinate sequentially, we append to every layer in $M_{j}$ memory neurons with $x$, and with the accumulated sum $\Sigma_{j-1}=\sum\limits_{i=1}\limits^{j-1}(s)_{i}$.
    
    After computing $M_{j}$ for all $1\leq j \leq k$ sequentially, we obtain the neuron
    \begin{equation*}
        \Sigma_{k}(x)=\sum\limits_{j=1}\limits^{k}(\eta(x))_{j}+(r^{p}+\frac{3}{4}w^{p})(c(x))_{j}.
    \end{equation*}
    Finally the network will return as output the value
    \begin{equation*}
        f_{x_{0},w,p}(x)=\left[ y_{0}\left(1 - \frac{\left[\Sigma_{k}(x)-r^{p}-\frac{w^p}{4}\right]_{+}}{w^{p}/2}\right) \right]_{+}~,
    \end{equation*}
    as can be seen in the following sketch of the layout of $f_{x_{0},w,p}$:
    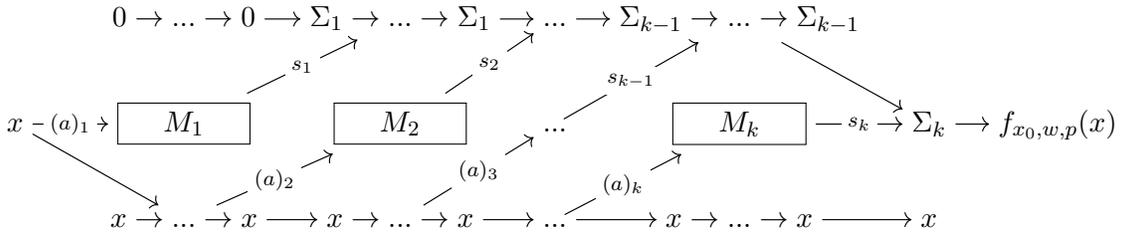
\begin{figure}[H]
    \centering
    \[\begin{tikzcd}[cramped,column sep=small]
	&& {0\rightarrow...\rightarrow 0} & {\Sigma_{1}\rightarrow...\rightarrow \Sigma_{1}} & {...} & {\Sigma_{k-1}\rightarrow...\rightarrow \Sigma_{k-1}} \\
	x && {\boxed{\Large \quad M_{1} \quad}} & {\boxed{\Large \quad M_{2} \quad}} & {...} & {\boxed{\Large \quad M_{k} \quad}} & {\Sigma_{k}} & {f_{x_{0},w,p}(x)} \\
	&& {x\rightarrow...\rightarrow x} & {x\rightarrow...\rightarrow x} & {...} & {x\rightarrow...\rightarrow x} & x
	\arrow["{(a)_{1}}"{description}, from=2-1, to=2-3]
	\arrow[from=2-1, to=3-3]
	\arrow[from=1-6, to=2-7]
	\arrow["{s_{k}}"{description}, from=2-6, to=2-7]
	\arrow["{s_{1}}"{description}, from=2-3, to=1-4]
	\arrow[from=1-3, to=1-4]
	\arrow["{(a)_{2}}"{description}, from=3-3, to=2-4]
	\arrow[from=3-3, to=3-4]
	\arrow[from=2-7, to=2-8]
	\arrow[from=3-6, to=3-7]
	\arrow["{(a)_{3}}"{description}, from=3-4, to=2-5]
	\arrow["{s_{2}}"{description}, from=2-4, to=1-5]
	\arrow["{s_{k-1}}"{description}, from=2-5, to=1-6]
	\arrow[from=3-4, to=3-5]
	\arrow[from=3-5, to=3-6]
	\arrow["{(a)_{k}}"{description}, from=3-5, to=2-6]
	\arrow[from=1-4, to=1-5]
	\arrow[from=1-5, to=1-6]
    \end{tikzcd}\]
    \caption{The architecture of $f_{x_{0},w,p}$}
    \label{fig:net ball indicator f}
    \end{figure}
    Note that one can append to the output neuron the vector $x$ wothout changing the width and depth of $f_{x_{0},w,p}$.
    
    Let us show that $f_{x_{0},w,p}$ behaves as we desired. Let $x\in\mathbb{R}^{k}$.
    \begin{itemize}
        \item Note that $0\leq \frac{\left[\Sigma_{k}(x)-r^{p}-\frac{w^p}{4}\right]_{+}}{w^{p}/2}$ and so because ReLU is increasing we always have 
        \begin{equation*}
            f_{x_{0},w,p}(x)=\left[ y_{0}\left(1 - \frac{\left[\Sigma_{k}(x)-r^{p}-\frac{w^p}{4}\right]_{+}}{w^{p}/2}\right) \right]_{+}\leq y_{0}~.
        \end{equation*}
        \item If $\lVert x - x_{0} \rVert^{p}_{p}\leq  r^{p}$, then for every $1\leq j \leq k$ we have $\lvert (x)_{j} - (x_{0})_{j} \rvert \leq r<\delta$, so $(c)_{j} = 0$ and $0\leq (b)_{j}=\frac{\lvert (x)_{j}-(x_{0})_{j}\rvert}{\delta} < 1$. Hence
        \begin{equation*}
        \begin{aligned}
            &&\Sigma_{k}(x)&=\sum\limits_{j=1}\limits^{k}(\eta)_{j}=\sum\limits_{j=1}\limits^{k}[\delta^{p}(\gamma)_{j}]_{+}\\
            &&&\leq \sum\limits_{j=1}\limits^{k}[\delta^{p}((b)^{p}_{j}+\Tilde{\epsilon})]_{+}=\sum\limits_{j=1}\limits^{k}\left[\delta^{p}\left(\frac{\lvert (x)_{j}-(x_{0})_{j}\rvert^{p}}{\delta^{p}}+\Tilde{\epsilon}\right)\right]_{+}\\
            &&&=\sum\limits_{j=1}\limits^{k}[\lvert (x)_{j}-(x_{0})_{j} \rvert^{p}+\delta^{p}\Tilde{\epsilon}]_{+} = \left(\sum\limits_{j=1}\limits^{k}\lvert (x)_{j}-(x_{0})_{j} \rvert^{p}\right)+k\delta^{p}\Tilde{\epsilon}\\
            &&&\leq r^{p} +k\delta^{p}\frac{w^{p}}{4k\delta^{p}}=r^{p}+\frac{w^{p}}{4}~,
        \end{aligned}
        \end{equation*}

        which yields $f_{x_{0},w,p}(x)=y_{0}$.
        \item If $(r+w)^{p} \leq \lVert x - x_{0} \rVert^{p}_{p}$, then one of the following occurs:
        \begin{enumerate}
            \item There exists some $1\leq j_{0} \leq k$ such that $\delta \leq \lvert (x)_{j_{0}} - (x_{0})_{j_{0}} \rvert$. In this case, $(c)_{j_{0}}\geq 1$, and therefore:
            \begin{equation*}
            \begin{aligned}
                &&\Sigma_{k}(x)&=\sum\limits_{j=1}\limits^{k}(\eta)_{j}+(r^{p}+\frac{3}{4}w^{p})(c(x))_{j}\\
                && &\geq \sum\limits_{j=1}\limits^{k}(r^{p}+\frac{3}{4}w^{p})(c(x))_{j}\geq  (r^{p}+\frac{3}{4}w^{p})(c(x))_{j_{0}}\\
                && &\geq r^{p}+\frac{3}{4}w^{p}~.
            \end{aligned}
            \end{equation*}
            \item For every $1\leq j \leq k$ we have $\lvert (x)_{j} - (x_{0})_{j} \rvert < \delta$, so $0\leq (b)_{j}=\frac{\lvert (x)_{j}-(x_{0})_{j}\rvert}{\delta} < 1$. Hence
            \begin{equation*}
            \begin{aligned}
                &&\Sigma_{k}(x)&=\sum\limits_{j=1}\limits^{k}(\eta)_{j}+(r^{p}+\frac{3}{4}w^{p})(c(x))_{j}\\
                && &\geq \sum\limits_{j=1}\limits^{k}(\eta)_{j} = \sum\limits_{j=1}\limits^{k}[\delta^{p}(\gamma)_{j}]_{+}\geq  \sum\limits_{j=1}\limits^{k}[\delta^{p}((b)^{p}_{j}-\Tilde{\epsilon})]_{+}\\
                & &&=\sum\limits_{j=1}\limits^{k}\left[\delta^{p}\left(\frac{\lvert (x)_{j}-(x_{0})_{j}\rvert^{p}}{\delta^{p}}-\Tilde{\epsilon}\right)\right]_{+} = \sum\limits_{j=1}\limits^{k}\left[\lvert (x)_{j} - (x_{0})_{j}\rvert^{p} -\delta^{p}\Tilde{\epsilon}\right]_{+} \\
                && & \geq \sum\limits_{j=1}\limits^{k}\left(\lvert (x)_{j} - (x_{0})_{j}\rvert^{p} -\delta^{p}\Tilde{\epsilon}\right) \geq (r+w)^{p} - k\delta^{p}\Tilde{\epsilon}= (r+w)^{p} - \frac{w^{p}}{4}\\
                && & \geq r^{p}+w^{p} - \frac{w^{p}}{4}=r^{p}+\frac{3}{4}w^{p}~.&\text{(since $1\leq p$)}
            \end{aligned}
            \end{equation*}
        \end{enumerate}
        In any case we have $\Sigma_{k}(x) \geq r^{p}+\frac{3}{4}w^{p}$, which yields \[f_{x_{0},w,p}(x) \leq \left[ y_{0}\left(1 - \frac{[r^{p}+\frac{3w^{p}}{4} - r^{p} -\frac{w^{p}}{4}]_{+}}{w^{p}/2}\right) \right]_{+}=0~.\]
    \end{itemize}
    From all of the above combined we get that $f_{x_{0},w,p}$ behaves as desired. Note that by construction, the maximal width of $f_{x_{0},w,p}$ is the maximal width of $M_{j}$ (which is $1+\mathcal{W}(g_{\Tilde{\epsilon},p})$) plus one memory neuron for the accumulated sum $\Sigma_{j}$, and $k$ neurons to carry $x$. Thus we have a total of $k+1+\mathcal{W}(g_{\Tilde{\epsilon},p})+1$, so $f_{x_{0},w,p}:\mathbb{R}^{k}\rightarrow\mathbb{R}$ is a neural network of width $k+2+\mathcal{W}(g_{\Tilde{\epsilon},p})$. Additionally, $\mathcal{L}(f_{x_{0},w,p})=O(k\mathcal{L}(M_{j}))$ so $f_{x_{0},w,p}$ has depth $O\left(k\mathcal{L}(g_{\Tilde{\epsilon},p})\right)$.

\end{proof}

\begin{proof}[Proof of \lemref{lma:net ball indicator p=infty}]
    Denote the following functions of $x\in\mathbb{R}^k$:
    \begin{equation*}
    \begin{aligned}
        &&(a)_{j}&=(x)_{j}-(x_{0})_{j}&\\
        &&(b)_{j}&=\left[\left(\left[(a)_{j}\right]_{+}+\left[-(a)_{j}\right]_{+}\right)\right]_{+}&~.
    \end{aligned}     
    \end{equation*}
    Note that for all $j$ we have $(b)_{j}=\lvert (x)_{j} - (x_{0})_{j} \rvert$.
    We will now construct $f_{x_{0},w,\infty}$ using these quantities. For every coordinate $j$ we apply the sub-network $M_{j}$ (see \figref{fig:net ball indicator Mj p=infty} for layout of $M_{j}$) which updates the running maximum $m_{j}$ (initialize $m_{0}=0$).

    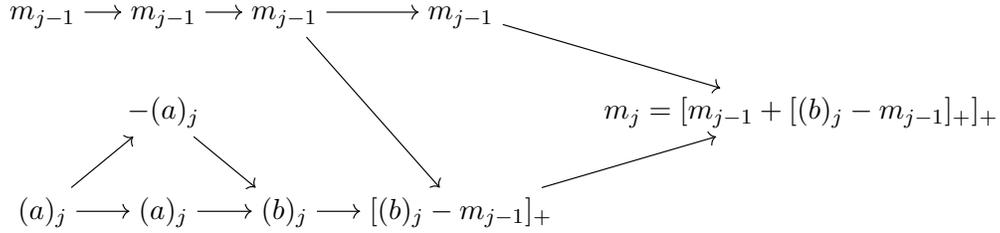
\begin{figure}[H]
    \centering
    \[\begin{tikzcd}[cramped,column sep=small]
	{m_{j-1}} & {m_{j-1}} & {m_{j-1}} & {m_{j-1}} \\
	& {-(a)_{j}} &&& {m_{j}=[m_{j-1}+[(b)_{j}-m_{j-1}]_{+}]_{+}} \\
	{(a)_{j}} & {(a)_{j}} & {(b)_{j}} & {[(b)_{j}-m_{j-1}]_{+}}
	\arrow[from=1-1, to=1-2]
	\arrow[from=1-2, to=1-3]
	\arrow[from=1-3, to=1-4]
	\arrow[from=1-3, to=3-4]
	\arrow[from=1-4, to=2-5]
	\arrow[from=2-2, to=3-3]
	\arrow[from=3-1, to=2-2]
	\arrow[from=3-1, to=3-2]
	\arrow[from=3-2, to=3-3]
	\arrow[from=3-3, to=3-4]
	\arrow[from=3-4, to=2-5]
    \end{tikzcd}\]
    \caption{The architecture of $M_{j}$}
    \label{fig:net ball indicator Mj p=infty}
    \end{figure}

    We note that the width and depth of $M_{j}$ satisfy $\mathcal{W}(M_{j})=3$, and $\mathcal{L}(M_{j})=5$. In order to compute $M_{j}$ for every coordinate sequentially, we append to every layer in $M_{j}$ memory neurons with $x$.
    
    After computing $M_{j}$ for all $1\leq j \leq k$ sequentially, we obtain the neuron
    \begin{equation*}
        m_{k}(x)=\max\{(b)_{j}\mid 1\leq j \leq k\}=\max\{\lvert (x)_{j} - (x_{0})_{j} \rvert\mid 1\leq j \leq k\}=\lVert x-x_{0}\rVert_{\infty}.
    \end{equation*}
    Finally the network will return as output the value
    \begin{equation*}
        f_{x_{0},w,\infty}(x)=\left[ y_{0}\left(1 - \frac{\left[m_{k}(x)-r\right]_{+}}{w}\right) \right]_{+}~,
    \end{equation*}
    as can be seen in the following sketch of the layout of $f_{x_{0},w,\infty}$:
    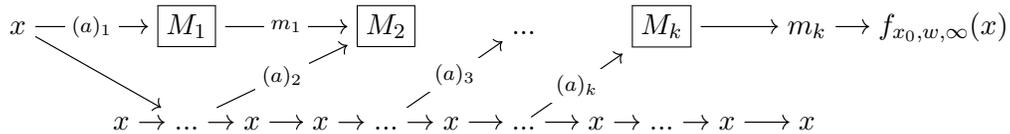
\begin{figure}[H]
    \centering
    \[\begin{tikzcd}[cramped,column sep=small]
	x && {\boxed{\Large M_{1}}} & {\boxed{\Large M_{2}}} & {...} & {\boxed{\Large M_{k}}} & {m_k} & {f_{x_{0},w,\infty}(x)} \\
	&& {x\rightarrow...\rightarrow x} & {x\rightarrow...\rightarrow x} & {...} & {x\rightarrow...\rightarrow x} & x
	\arrow["{{(a)_{1}}}"{description}, from=1-1, to=1-3]
	\arrow[from=1-1, to=2-3]
	\arrow["{m_1}"{description}, from=1-3, to=1-4]
	\arrow[from=1-6, to=1-7]
	\arrow[from=1-7, to=1-8]
	\arrow["{{(a)_{2}}}"{description}, from=2-3, to=1-4]
	\arrow[from=2-3, to=2-4]
	\arrow["{{(a)_{3}}}"{description}, from=2-4, to=1-5]
	\arrow[from=2-4, to=2-5]
	\arrow["{{(a)_{k}}}"{description}, from=2-5, to=1-6]
	\arrow[from=2-5, to=2-6]
	\arrow[from=2-6, to=2-7]
    \end{tikzcd}\]
    \caption{The architecture of $f_{x_{0},w,\infty}$}
    \label{fig:net ball indicator f p=infty}
    \end{figure}
    Note that one can append to the output neuron the vector $x$ wothout changing the width and depth of $f_{x_{0},w,\infty}$.
    
    Let us show that $f_{x_{0},w,\infty}$ behaves as we desired. Let $x\in\mathbb{R}^{k}$.
    \begin{itemize}
        \item Note that $0\leq \frac{\left[m_{k}(x)-r\right]_{+}}{w}$ and so because ReLU is increasing we always have $ f_{x_{0},w,\infty}(x)\leq y_{0}$.
        \item If $\lVert x - x_{0} \rVert_{\infty}\leq  r$, then $m_{k}(x)-r \leq 0$ and so $f_{x_{0},w,\infty}(x)=y_{0}$.
        \item If $r+w \leq \lVert x - x_{0} \rVert_{\infty}$, then $\frac{\left[m_{k}(x)-r\right]_{+}}{w}=\frac{\lVert x - x_{0} \rVert_{\infty}-r}{w}\geq \frac{w}{w}=1$ and so $f_{x_{0},w,\infty}(x)=0$.
    \end{itemize}
    From all of the above combined we get that $f_{x_{0},w,\infty}$ behaves as desired. Note that by construction, the maximal width of $f_{x_{0},w,\infty}$ is the maximal width of $M_{j}$ (which is $3$) plus $k$ neurons to carry $x$. Thus we have a total of $k+3$, so $f_{x_{0},w,\infty}:\mathbb{R}^{k}\rightarrow\mathbb{R}$ is a neural network of width $k+3$. Additionally, $\mathcal{L}(f_{x_{0},w,\infty})=O(k\mathcal{L}(M_{j}))$ so $f_{x_{0},w,\infty}$ has depth $O\left(k\right)$.
\end{proof}

\begin{remark}[Constructions equivalent to \Cref{lma:net ball indicator,lma:net ball indicator p=infty}]\label{remark:wider construction better depth}
    In \lemref{lma:net ball indicator}, if we construct $f_{x_{0},w,p}$ the same way where instead of computing the $M_{j}$ components sequentially we stack them and perform the computation in parallel, we would obtain a network $f_{x_{0},w,p}$ that behaves exactly the same and has width $k\cdot \mathcal{W}(M_{j})=k(1+\mathcal{W}(g_{\Tilde{\epsilon},p}))$ and depth $O(\mathcal{L}(M_{j}))=O(\mathcal{L}(g_{\Tilde{\epsilon},p}))$. Similarly, in \lemref{lma:net ball indicator p=infty} we would get a network with width $2k$ and depth $O(1)$. Returning also the input vector would increase the width by additional $k$ neurons in both cases. 
\end{remark}

The following lemma is used to approximate the square of a given number.

\begin{lemma}\label{lma:net square}\cite[Proposition~III.2]{elbrachter2019deep}
    Let $0 < \epsilon < \frac{1}{2}$. Then there exists a neural network $g_{\epsilon}:\mathbb{R}\rightarrow\mathbb{R}$ with width $3$ and depth $O\left(\log_{2}(\epsilon^{-1})\right)$ such that $|g_{\epsilon}(\alpha)-\alpha^{2}| \leq \epsilon$ for every $\alpha\in [0,1]$.
\end{lemma}

The following lemma is used to approximate the $p$ power of a given number.

\begin{lemma}\label{lma:net p power natural}
    Let $p\in\mathbb{N}_{\geq 2}$, $0 < \epsilon < 1/2$. Then there exists a neural network $g_{\epsilon,p}:\mathbb{R}\rightarrow\mathbb{R}$ with width $6$ and depth $O\left(p\log_{2}(p\epsilon^{-1})\right)$ such that $|g_{\epsilon,p}(\alpha)-\alpha^{p}| < \epsilon$ for every $\alpha\in [0,1]$.
\end{lemma}

In the proof we will use the following lemma to compute multiplication:

\begin{lemma}\label{lma:net mult}\cite[Proposition~III.3]{elbrachter2019deep}
    Let $0 < \epsilon < \frac{1}{2}$. Then there exists a neural network $h_{\epsilon}:\mathbb{R}\rightarrow\mathbb{R}$ with width $5$ and depth $O\left(\log_{2}(\epsilon^{-1})\right)$ such that $|h_{\epsilon}(\alpha,\beta)-\alpha\beta| \leq \epsilon$ for every $\alpha,\beta\in [0,1]$.
\end{lemma}

\begin{proof}[Proof of \lemref{lma:net p power natural}]
    Denote $\epsilon_{2}=\frac{\epsilon}{p-1}$ then from \lemref{lma:net mult} there exists a neural network $h_{\epsilon_{2}}:\mathbb{R}\rightarrow\mathbb{R}$ with width $5$ and depth $O\left(\log_{2}(\epsilon_{2}^{-1})\right)$ such that $|h_{\epsilon_{2}}(\alpha,\beta)-\alpha\beta| \leq \epsilon_{2}$ for every $\alpha,\beta\in [0,1]$. We will apply $h_{\epsilon_{2}}$ repeatedly $p$ times as in the following figure:
    \begin{figure}[H]
    \centering
    \[\begin{tikzcd}
	& \alpha & {\boxed{h_{\epsilon_{2}}}} & {\boxed{h_{\epsilon_{2}}}} & \cdots & {\boxed{h_{\epsilon_{2}}}} & {\boxed{h_{\epsilon_{2}}}} & {g_{\epsilon,p}(\alpha)} \\
	\alpha & \alpha & \alpha & \alpha & \cdots & \alpha
	\arrow[from=1-2, to=1-3]
	\arrow[from=1-3, to=1-4]
	\arrow[from=1-4, to=1-5]
	\arrow[from=1-5, to=1-6]
	\arrow[from=1-6, to=1-7]
	\arrow[from=1-7, to=1-8]
	\arrow[from=2-1, to=1-2]
	\arrow[from=2-1, to=2-2]
	\arrow[from=2-2, to=1-3]
	\arrow[from=2-2, to=2-3]
	\arrow[from=2-3, to=1-4]
	\arrow[from=2-3, to=2-4]
	\arrow[from=2-4, to=2-5]
	\arrow[from=2-5, to=2-6]
	\arrow[from=2-6, to=1-7]
    \end{tikzcd}\]
    \caption{The architecture of $g_{\epsilon, p}$}
    \label{fig:g p natural}
    \end{figure}
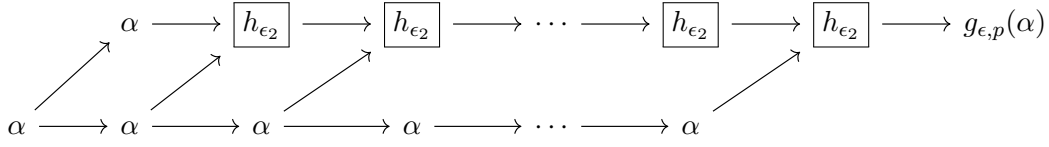
    Now from the definition we have that for all $\alpha\in[0,1]$:
    \begin{equation*}
    \begin{aligned}
        &&|g_{\epsilon,p}(\alpha)-\alpha^{p}|&\leq \epsilon_{2}\sum\limits^{p-2}_{j=0}\alpha^{j}\leq \epsilon_{2}(p-1)=\epsilon\\
    \end{aligned}
    \end{equation*}
    Note that $g_{\epsilon,p}$ has width $5+1=6$ and depth $O\left(p\log_{2}(\epsilon_{2}^{-1})\right)=O\left(p\log_{2}(p\epsilon^{-1})\right)$ and we are done.
\end{proof}

For $p\notin\mathbb{N}$ we have the following generalization:

\begin{lemma}\label{lma:net p power gen}
    Let $p\in(1,\infty)\setminus\mathbb{N}$, $0 < \epsilon < 1$. Then there exists a neural network $g_{\epsilon,p}:\mathbb{R}\rightarrow\mathbb{R}$ with width $9$ and depth $O\left(p\epsilon^{-\frac{1}{p}}\left(\log_{2}(p\epsilon^{-1})+p\right)\right)$ such that $|g_{\epsilon,p}(\alpha)-\alpha^{p}| < \epsilon$ for every $\alpha\in [0,1]$.
\end{lemma}

\begin{proof}[Proof of \lemref{lma:net p power gen}]
    Define $D=\lceil p(\pi\epsilon/2)^{-1/\floor{p}} \rceil+2$, and the polynomial $P_{\epsilon/2, p}:\mathbb{R}\rightarrow\mathbb{R}$ by $P_{\epsilon/2, p}(x)=\sum\limits_{i=0}^{D}\binom{p}{i}(x-1)^{i}$. Then from \lemref{lma:apx p power by poly} we have $|P_{\epsilon/2, p}(x)-x^{p}| < \epsilon/2$ for every $x\in [0,1]$. Define $Q_{\epsilon/2,p}(x)=P_{\epsilon/2,p}(x+1)=\sum\limits_{i=0}^{D}\binom{p}{i}x^{i}$, and note that the coefficients satisfy $B:=\underset{0\leq i \leq D}{\max}\left|\binom{p}{i}\right| \leq 2^{p}$. Now, from \lemref{lma:apx polynom} there exists a neural network $\Phi_{\epsilon/2, D, B}:\mathbb{R}\rightarrow\mathbb{R}$ with width $9$ and depth $O\left(D\left(\log_{2}(\frac{2}{\epsilon})+\log_{2}(D)+\log_{2}(B)\right)\right)$ such that $|\Phi_{\epsilon/2, D, B}(x)-Q_{\epsilon/2,p}(x)| < \epsilon/2$ for every $x\in [-1,1]$. Define the network $g_{\epsilon, p}(x) = \Phi_{\epsilon/2, D, B}(x-1)$. Let $\alpha\in [0,1]$, then $\alpha - 1\in [-1,0]$ and so
    \begin{equation*}
    \begin{aligned}
        &&|g_{\epsilon,p}(\alpha)-\alpha^{p}|& = |\Phi_{\epsilon/2, D, B}(\alpha-1)-\alpha^{p}|\\
        &&& \leq |\Phi_{\epsilon/2, D, B}(\alpha-1)-Q_{\epsilon/2,p}(\alpha-1)| + |Q_{\epsilon/2,p}(\alpha-1)-\alpha^{p}|\\
        &&& \leq |\Phi_{\epsilon/2, D, B}(\alpha-1)-Q_{\epsilon/2,p}(\alpha-1)| + |P_{\epsilon/2,p}(\alpha)-\alpha^{p}|\\
        &&& < \epsilon/2 + \epsilon/2 = \epsilon~.
    \end{aligned}
    \end{equation*}
    The width of $g_{\epsilon,p}$ is $9$, and plugging $D, B$ we get that its depth is $O\left(p\epsilon^{-\frac{1}{p}}\left(\log_{2}(p\epsilon^{-1})+p\right)\right)$.
\end{proof}

\begin{lemma}\label{lma:apx polynom}\cite[Proposition~III.5]{elbrachter2019deep}
    Let $0 < \epsilon < \frac{1}{2}$, $D\in \mathbb{N}$, $0<B$, and $P:\mathbb{R}\rightarrow\mathbb{R}$ a polynomial given by $P(\alpha)=\sum\limits_{i=0}^{D}c_{i}\alpha^{i}$ where $\underset{0\leq i \leq D}{\max}\lvert c_{i} \rvert \leq B$. Then there exists a neural network $\Phi_{\epsilon, D, B}:\mathbb{R}\rightarrow\mathbb{R}$ with width $9$ and depth $O\left(D\left(\log_{2}(\frac{1}{\epsilon})+\log_{2}(D)+\log_{2}(B)\right)\right)$ such that $|\Phi_{\epsilon, D, B}(\alpha)-P(\alpha)| < \epsilon$ for every $\alpha\in [-1,1]$.
\end{lemma}

\begin{proof}[Proof sketch for \propref{prop:upper bound memorization big k, big depth}]\label{prf:sketch for deep upper bound}
    Since the balls $B^{d}_{q}(x_{i},c^{+}_{p,q}(d)\sigma)$ are disjoint and $\delta\leq \lVert x_{i} - x_{j}\rVert_{q}$, there exists a Lipschitz continuous function $g:\mathbb{R}^{d}\rightarrow\mathbb{R}$ with the following properties:
    \begin{itemize}
        \item $g$ has Lipschitz constant $\frac{2C}{\delta - 2c^{+}_{p,q}(d)\sigma}$.
        \item $g$ has compact support $K_{\mathcal{D}}$.
        \item The diameter of $K_{\mathcal{D}}$ w.r.t the $l_{2}$ norm satisfies $\text{diam}(K_{\mathcal{D}})\leq 2\text{diam}(\mathcal{D})$.
        \item for every $1\leq i \leq N$ we have $g( B^{d}_{q}(x_{i},c^{+}_{p,q}(d)\sigma))=y_{i}$.
    \end{itemize}
    Denote by $\omega_{g}$ the modulus of continuity of $g$, and $\omega^{-1}_{g}(\epsilon)=\sup\{\alpha \mid \omega_{g}(\alpha)\leq \epsilon\}$ then by \cite[Theorem~1]{hanin2017approximating} there exists a neural network $\mathcal{N}:\mathbb{R}^{d}\rightarrow\mathbb{R}$ with width $d+1$ and depth $O\left(\frac{\text{diam}(\mathcal{D})}{\omega^{-1}_{g}(1/4)}\right)^{d+1}$ such that
    \begin{equation*}
        \underset{x\in K_{\mathcal{D}}}{\sup}\lvert g(x) - \mathcal{N}(x) \rvert \leq 1/4~.
    \end{equation*}
    For every $\alpha$ we have $\omega_{g}(\alpha) = \frac{2C}{\delta - 2c^{+}_{p,q}(d)\sigma}\alpha$ so $\frac{\delta - 2c^{+}_{p,q}(d)\sigma}{8C}\leq \omega^{-1}_{g}(1/4)$ and hence $\mathcal{N}$ has depth $O\left(\frac{C\text{diam}(\mathcal{D})}{\delta - 2c^{+}_{p,q}(d)\sigma}\right)^{d+1}$.
    From \lemref{lma:step function net} there exists a neural network $\psi:\mathbb{R}\rightarrow\mathbb{R}$ of width $2$ and depth $O(C)$ such that for every $1\leq m \leq C$ and every $t\in [-1/4+m,m+1/4]$ we have $\psi(t)=m$. Define $f=\psi \circ \mathcal{N}$ then $f$ has width $d+1$ and depth $O\left(\frac{C\text{diam}(\mathcal{D})}{\delta - 2c^{+}_{p,q}(d)\sigma}\right)^{d+1}$. Let $1\leq i \leq N$ and $x\in B^{d}_{p}(x_{i},\sigma))$. We have 
    \begin{equation*}
        \lvert y_{i} - \mathcal{N}(x) \rvert = \lvert g(x) - \mathcal{N}(x) \rvert \leq 1/4~.
    \end{equation*}
    Therefore, $\mathcal{N}(x)\in [-1/4+y_{i},y_{i}+1/4]$ and so $f(x)=\psi(\mathcal{N}(x))=y_{i}$. Now $d+1\leq k$ and so by padding each hidden layer of $f$ with $k-(d+1)$ neurons we obtain $f$ with width $k$ and depth 
    \begin{equation*}
        O\left(\frac{C\text{diam}(\mathcal{D})}{\delta - 2c^{+}_{p,q}(d)\sigma}\right)^{d+1}
    \end{equation*}
    that $(\sigma, p)$-robustly memorizes the dataset $\mathcal{D}$.
\end{proof}

\begin{lemma}\label{lma:step function net}
    Let $C\in\mathbb{N}_{\geq 2}$. There exists a neural network $\psi:\mathbb{R}\rightarrow\mathbb{R}$ of width $2$ and depth $O(C)$ such that for every $1\leq m \leq C$ and every $t\in [-1/4+m,m+1/4]$ we have $\psi(t)=m$.
\end{lemma}

\begin{proof}[Proof of \lemref{lma:step function net}]
    Define the following functions for every $x\in \mathbb{R}$ and every $0\leq l \leq C-1$: 
    \begin{itemize}
        \item $\psi_{3l}\left(x\right)=\left[2x-\frac{1}{2}\left(2(C-l)-1\right) - \psi_{3l-1}(x)\right]_{+}$ (where $\psi_{-1}(x)=0$).
        \item $\psi_{3l+1}\left(x\right)=\left[l+1-\psi_{3l}\left(x\right)\right]_{+}$.
        \item $\psi_{3l+2}\left(x\right)=\left[l+1-\psi_{3l+1}\left(x\right)\right]_{+}$~,
    \end{itemize}

    and define the network $\psi:\mathbb{R}\rightarrow\mathbb{R}$ by $\psi = \psi_{3C-1}$. See \figref{fig:arch of step function} bellow:
    \begin{figure}[H]
        \centering
        \[\begin{tikzcd}[cramped,column sep=small]
    	x & {\psi_{0}(x)} & {\psi_{1}\left(x\right)} & {\psi_{2}\left(x\right)} & {\psi_{3}\left(x\right)} & \cdots & {\psi_{3C-3}\left(x\right)} & {\psi_{3C-2}\left(x\right)} & {\psi_{3C-1}\left(x\right)} & {\psi(x)} \\
    	& x & x & x & x & \cdots
    	\arrow[from=1-1, to=1-2]
    	\arrow[from=1-1, to=2-2]
    	\arrow[from=1-2, to=1-3]
    	\arrow[from=1-3, to=1-4]
    	\arrow[from=1-4, to=1-5]
    	\arrow[from=1-5, to=1-6]
    	\arrow[from=1-6, to=1-7]
    	\arrow[from=1-7, to=1-8]
    	\arrow[from=1-8, to=1-9]
    	\arrow[from=1-9, to=1-10]
    	\arrow[from=2-2, to=2-3]
    	\arrow[from=2-3, to=2-4]
    	\arrow[from=2-4, to=1-5]
    	\arrow[from=2-4, to=2-5]
    	\arrow[from=2-5, to=2-6]
    	\arrow[from=2-6, to=1-7]
        \end{tikzcd}\]
    \caption{The architecture of $\psi$}
    \label{fig:arch of step function}
    \end{figure}
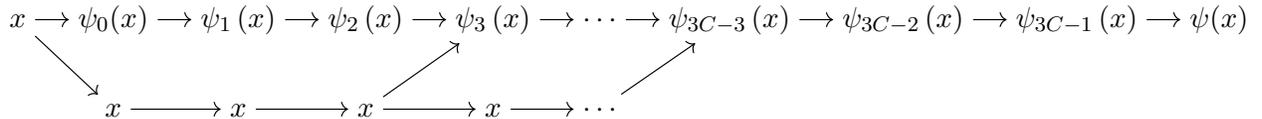
    Then $\psi$ has width $2$ and depth $O(C)$ and a computation using the definitions yields that for all $1\leq m \leq C$ and all $t\in [m-1/4, m+1/4]$ we have $\psi(t)=m$.
\end{proof}

\section{Background Material}

\subsection{Invariant Measures on  $O\left(d\right)$}
\label{subsection: inv measures}

We recall the following facts (see e.g \cite[Chapter~3]{mattila1999geometry}):
\begin{lemma}\label{lma:O_n is compact}
    The group $O\left(d\right)$ is compact (as a topological group with Borel topology induced by operator norm).
\end{lemma}

\begin{lemma}\label{lma:grass is compact}
    The real Grassmannian $\text{Gr}_{d,k}$ is a compact Hausdorff space with respect to the topology induced by the operator norm.
\end{lemma}

\begin{lemma}\label{lma:grass is G-space}
    The real Grassmannian $\text{Gr}_{d,k}$ is a transitive $O\left(d\right)$-space with respect to the action $\left(g,V\right)\mapsto gV$ where for $V=\text{Span}\left\{v_{1}, ..., v_{k}\right\}$, $gV=\text{Span}\left\{gv_{1}, ..., gv_{k}\right\}$.
\end{lemma}

\begin{lemma}\label{lma:sphere is compact}
    The sphere $\mathbb{S}^{d-1}$ is a compact Hausdorff space with respect to the topology induced by the geodesic metric.
\end{lemma}

\begin{lemma}\label{lma:sphere is G-space}
    The sphere $\mathbb{S}^{d-1}$ is a transitive $O\left(d\right)$-space with respect to the action $\left(g,x\right)\mapsto g^{-1}x$.
\end{lemma}

From \lemref{lma:compact group has haar measure}, \lemref{lma:O_n is compact} the group $O\left(d\right)$ has a unique Haar probability measure denoted here by $\nu_{d}$. From \lemref{lma:G-space has inv measure}, \lemref{lma:grass is compact}, \lemref{lma:grass is G-space} we get that $\text{Gr}_{d,k}$ has a unique $O\left(d\right)$-invariant probability measure which we will denote by $\gamma_{d,k}$. Similarly from \lemref{lma:G-space has inv measure}, \lemref{lma:sphere is compact}, \lemref{lma:sphere is G-space} we get that $\mathbb{S}^{d-1}$ has a unique $O\left(d\right)$-invariant probability measure which we will denote by $\mu_{d-1}$. Furthermore these unique measures are simply the pushforward of the measure $\nu_{d}$ as summarized by the following:

\begin{lemma}\label{lma:pushforward property}
    For any $V\in\text{Gr}_{d,k}$ and any measurable $E\subseteq\text{Gr}_{d,k}$,
    \begin{equation*}
        \gamma_{d,k}\left(E\right) = \nu_{d}\left(\left\{g\in O\left(d\right) \mid gV\in E  \right\}\right)~.
    \end{equation*}
    Similarly, for any $x\in\mathbb{S}^{d-1}$ and any measurable $E \subseteq \mathbb{S}^{d-1}$,
    \begin{equation*}
        \mu_{d-1}\left(E\right) = \nu_{d}\left(\left\{g\in O\left(d\right) \mid
g^{-1}x \in E \right\}\right)~.
    \end{equation*}
\end{lemma}
\begin{proof}[Proof of \lemref{lma:pushforward property}]
    The claim follows from \lemref{lma:G-space has inv measure} and from the fact that the push-forward measure doesn't depend on the choice of the fixed point - indeed this measure is unique and is identical for any choice of the fixed point, and in particular for the choice of $V$ and $x$.
\end{proof}

The following presents the relation between the measure of complement spaces
\begin{lemma}\label{lma:invariance to perp}
    Let $E\subseteq\text{Gr}_{d,k}$ measurable then
    \begin{equation*}
        \gamma_{d,k}\left(E\right)=\gamma_{d,d-k}\left(\left\{V^{\perp}\in\text{Gr}_{d,d-k} \mid V\in E\right\}\right)~.
    \end{equation*}
\end{lemma}
\begin{proof}[Proof of \lemref{lma:invariance to perp}]
    Define a new measure on $\text{Gr}_{d,k}$ by \[\gamma^{\prime}_{d,k}\left(E\right) = \gamma_{d,d-k}\left(\left\{V^{\perp}\in\text{Gr}_{d,d-k} \mid V\in E\right\}\right)~.\] It is a probability measure and one can check that it is also $O\left(d\right)$-invariant and hence the claim follows from uniqueness.
\end{proof}

\subsection{Background in $G$-spaces}\label{subsection:G-spaces}

\begin{lemma}\label{lma:compact group has haar measure}
    Let $G$ be a compact group. Then, there exists a unique Haar probability measure (which is both left and right invariant) on its Borel sigma algebra.
\end{lemma}

\begin{proof}[Proof of \lemref{lma:compact group has haar measure}]
    By \cite[Theorem~2.10]{Folland_2015} $G$ has a left Haar measure $\lambda$.
    A Haar measure is a Radon measure and so by definition it is finite on compact sets, so
    $\lambda$ is a finite positive measure and so by normalizing,  $\lambda$ is a left Haar probability measure. Let $\mu$ be a left Haar probability measure on $G$. By \cite[Theorem~2.20]{Folland_2015} $\exists c \in \left(0, \infty\right)$ such that $\mu=c\lambda$, but $1=\mu\left(G\right)=c\lambda\left(G\right)=c\cdot 1$ so $c=1$ and $\mu = \lambda$. Finally, because $G$ is compact it is unimodular and hence the obtained unique left Haar measure is also right Haar measure.
\end{proof}

\begin{lemma}\label{lma:haar invariant to inverse}
    Let $G$ be a compact group and denote by $\nu$ its unique probability Haar measure, then for all measurable sets $E\subseteq G$,
    \begin{equation*}
        \nu\left(E\right) = \nu\left(E^{-1}\right)~.
    \end{equation*}
\end{lemma}

\begin{proof}[Proof of \lemref{lma:haar invariant to inverse}]
    $G$ is compact and hence unimodular and so the claim follows from \cite[Proposition~2.31]{Folland_2015}.
\end{proof}

\begin{lemma}\label{lma: quotient group has inv measure}
    Let $G$ be a compact group, $H\leq G$ a closed subgroup then there exists a unique $G$-invariant probability measure on $G/H$.
\end{lemma}

\begin{proof}[Proof of \lemref{lma: quotient group has inv measure}]
    By \cite[Theorem~2.51]{Folland_2015} and compactness of $G$, there is a $G$-invariant Radon measure $\mu$ on $G/H$ which is unique up to a constant factor. Now, $G/H$ is compact so $\mu$ is finite so we can normalize it to be a probability measure, and by the uniqueness up to a factor we get that this probability measure is unique.
\end{proof}

\begin{lemma}\label{lma:G-space is quotient}
    Let $G$ be a compact group, $X$ a locally compact Hausdorff, transitive $G$-space with the action map $\alpha : G \times X \longrightarrow X$, and $x_{0}\in X$ some fixed point. Define $f:G\longrightarrow X$ by $f\left(g\right)=\alpha\left(g,x_{0}\right)$, $H=\text{Stab}_{G, \alpha}\left(x_{0}\right)$, $q:G\longrightarrow G/H$ the natural quotient map. Then, $H$ is a closed subgroup and $F=f \circ q^{-1} : G/H \longrightarrow X$ is a homeomorphism.
\end{lemma}

\begin{proof}[Proof of \lemref{lma:G-space is quotient}]
    See beginning of \cite[Chapter~2.6]{Folland_2015} and \cite[Proposition~2.46]{Folland_2015}.
\end{proof}

\begin{lemma}\label{lma:G-space has inv measure}
    Let $G$ be a compact group with its unique probability Haar measure $\nu$, $X$ a locally compact Hausdorff, transitive $G$-space with the action map $\alpha : G \times X \longrightarrow X$, and $x_{0}\in X$ some fixed point. Define $f:G\longrightarrow X$ by $f\left(g\right)=\alpha\left(g,x_{0}\right)$, then $X$ has a unique $G$-invariant probability radon measure which is given by $f_{*}\nu$, where $f_{*}\nu$ denotes the pushforward measure of $\nu$ under $f$.
\end{lemma}

\begin{proof}[Proof of \lemref{lma:G-space has inv measure}]
    By \lemref{lma: quotient group has inv measure}, \lemref{lma:G-space is quotient} $X$ has a unique $G$-invariant probability radon measure. $f$ is a continuous map and hence a measurable map and hence $f_{*}\nu$ is a measure on $X$.
    Now $f_{*}\nu\left(X\right) = \nu\left(f^{-1}\left(X\right)\right) = \nu\left(\left\{ g\in G \mid \alpha\left(g,x_{0}\right)\in X\right\} \right)= \nu\left(G\right)=1$ so $f_{*}\nu$ is a Radon probability measure on $X$. Let $g\in G$, $A\subset X$ measurable then $f_{*}\nu\left(\alpha\left(g,A\right)\right)=\nu\left(f^{-1}\left(\alpha\left(g,A\right)\right)\right)=\nu\left(\left\{ g^{\prime}\in G \mid \alpha\left(g^{\prime},x_{0}\right)\in \alpha\left(g,A\right)\right\} \right) = \nu\left(\left\{ g^{\prime}\in G \mid \alpha\left(g^{-1}g^{\prime},x_{0}\right)\in A\right\} \right)= \nu\left(\left\{ g\hat{g}\in G \mid \alpha\left(\hat{g},x_{0}\right)\in A\right\} \right)\\=\nu\left( g\left\{ \hat{g}\in G \mid \alpha\left(\hat{g},x_{0}\right)\in A\right\} \right)=\nu\left( \left\{ \hat{g}\in G \mid \alpha\left(\hat{g},x_{0}\right)\in A\right\} \right)=\nu\left(f^{-1}\left(A\right)\right)=f_{*}\nu\left(A\right)$. We conclude that $f_{*}\nu$ is a $G$-invariant probability Radon measure on $X$, and from uniqueness it is the only one.
\end{proof}

\subsection{Translative Coverings}\label{subsection:trans coverings}
We follow \cite{naszodi2016some} and define the following
\begin{definition}\label{def:covering number}
    Let $X$ be a set, $A\subseteq X$, 
    $\mathcal{F}$ a collection of subsets of $X$ i.e $\mathcal{F}\subseteq\mathcal{P}(X)$. A covering of $A$ by $\mathcal{F}$ is a subset of $\mathcal{F}$ whose union contains $A$. We define the covering number of $A$ by $\mathcal{F}$ to be the minimal cardinality of its coverings by $\mathcal{F}$:
    \begin{equation*}
        \tau(A,\mathcal{F})=\min\left\{\lvert\mathcal{F}^{\prime}\rvert \mid \mathcal{F}^{\prime}\subseteq \mathcal{F} , A\subseteq\underset{F\in\mathcal{F}^{\prime}}{\bigcup}F\right\}~.
    \end{equation*}
    If $X$ is a transitive $G$-space, $A,B\subseteq X$ some sets, we can look at the covering of $A$ by translations of $B$, i.e by the collection $\mathcal{G}_{G}B=\left\{g.B \mid g\in G\right\}$. The covering number of $A$ by translations of $B$ is therefore denoted by $\tau(A,\mathcal{G}_{G}B)$.
\end{definition}

\begin{definition}\label{def:covering density}
    Let $X$ be a transitive $G$-space with a $G$-invariant measure $\mu$. Let $A,B\subseteq X$, we define the covering density of $A$ by (translations of) $B$ to be
    \begin{equation*}
        \vartheta(A,\mathcal{G}_{G}B)=\mu(B)\tau(A,\mathcal{G}_{G}B)~.
    \end{equation*}
\end{definition}

\begin{theorem}\label{thm:covering density of sphere}\cite[Corollary~3.1]{rolfes2018covering}
    For every $k\in\mathbb{N}$, $0<\varphi<\pi/2$
    \begin{equation*}
        \vartheta\left(\mathbb{S}^{k},\mathcal{G}_{O(k+1)}B^{k}_{\text{arc}}(\varphi)\right)\leq \underset{1<\alpha}{\inf} \left(1+\frac{1}{\alpha - 1}\right)(k\ln(\alpha k) + 1)~.
    \end{equation*}
\end{theorem}

\begin{lemma}\label{lma:covering density bound}
    For every $k\in \mathbb{N}$ and every $0<\varphi<\frac{\pi}{2}$
    \begin{equation*}
        \vartheta\left(\mathbb{S}^{k}, \mathcal{G}_{O(k+1)}B^{k}_{\text{arc}}\left(\varphi\right)\right) \leq 5k\ln\left(k+1\right)~.
    \end{equation*}
\end{lemma}

\begin{proof}[Proof of \lemref{lma:covering density bound}]
    One has $\underset{1<\alpha}{\inf} \left(1+\frac{1}{\alpha - 1}\right)(k\ln(\alpha k) + 1) \leq 5k\ln\left(k+1\right)$ and so the claim follows from \thmref{thm:covering density of sphere}.
\end{proof}

\begin{lemma}\label{lma:bound cap measure}\cite[Corollary~3.2(i)]{boroczky2003covering}
    For every $k\in \mathbb{N}$ and every $0<\varphi<\frac{\pi}{2}$
    \begin{equation*}
        \mu_{k}\left(B^{k}_{\text{arc}}\left(\varphi\right)\right) \geq \frac{\sin^{k}\varphi}{\sqrt{2\pi\left(k+1\right)}}~.
    \end{equation*}
\end{lemma}

\begin{proof}[Proof of \lemref{lma:bound cap measure}]
    Denote by $\lvert \cdot \rvert$ the surface area, and by $V_{l}$ the volume of $\mathbb{S}^{l-1}$ for any $l\in\mathbb{N}$. We have from \cite[Lemma~3.1(i)]{boroczky2003covering}
    \begin{equation*}
        \lvert B^{k}_{\text{arc}}\left(\varphi\right) \rvert \geq V_{k}\sin^{k}\varphi~,
    \end{equation*}
    and so 
    \begin{equation*}
    \begin{aligned}
        &&\mu_{k}\left(B^{k}_{\text{arc}}\left(\varphi\right)\right)&=\frac{\lvert B^{k}_{\text{arc}}\left(\varphi\right) \rvert }{\lvert \mathbb{S}^{\left(k+1\right)-1} \rvert } = \frac{\lvert B^{k}_{\text{arc}}\left(\varphi\right) \rvert }{\left(k+1\right)V_{k+1} }\\
        &&&\geq \frac{V_{k}\sin^{k}\varphi}{\left(k+1\right)V_{k+1}}\\
        &&&=\frac{\pi^{\frac{k}{2}}\sin^{k}\varphi\Gamma\left(\frac{k+1}{2}+1\right)}{\left(k+1\right)\pi^{\frac{k+1}{2}}\Gamma\left(\frac{k}{2}+1\right)}=\frac{\sin^{k}\varphi\Gamma\left(\frac{k+1}{2}+1\right)}{\left(k+1\right)\pi^{\frac{1}{2}}\Gamma\left(\frac{k+1}{2}+\frac{1}{2}\right)}\\
        &&&>\frac{\sin^{k}\varphi}{\left(k+1\right)\pi^{\frac{1}{2}}}\sqrt{\frac{k+1}{2}}& \text{(Gautschi's inequality)}\\
        &&&=\frac{\sin^{k}\varphi}{\sqrt{2\pi\left(k+1\right)}}~.
    \end{aligned}
    \end{equation*}
    
\end{proof}

\subsection{Additional Lemmas}\label{subsec:additional lemmas}

\begin{lemma}\label{lma:apx p power by poly}
    Let $p\in(1,\infty)\setminus\mathbb{N}$, $0 < \epsilon < 1$, and $\psi_{p}:\mathbb{R}\rightarrow \mathbb{R}$ defined by $\psi_{p}(x)=x^{p}$. Define $D=\lceil p(\pi\epsilon)^{-1/\floor{p}} \rceil+2$, and $P_{\epsilon, p}:\mathbb{R}\rightarrow\mathbb{R}$ by $P_{\epsilon, p}(x)=\sum\limits_{i=0}^{D}\binom{p}{i}(x-1)^{i}$. Then $|P_{\epsilon, p}(x)-\psi_{p}(x)| < \epsilon$ for every $x\in [0,1]$.
\end{lemma}
    
\begin{proof}[Proof of \lemref{lma:apx p power by poly}]
    If $x=0$ then 
    \begin{equation*}
    \begin{aligned}
        &&\left|P_{\epsilon, p}(x)-\psi_{p}(x)\right|&=\left|P_{\epsilon, p}(0)-0^p\right|=\left|P_{\epsilon, p}(0)\right|=\left|\sum\limits_{i=0}^{D}\binom{p}{i}(-1)^{i}\right|\\
        &&&=\left| (-1)^{D}\binom{p-1}{D}\right|=\left| \binom{p-1}{D}\right|\\
        &&&\leq \frac{1}{\pi}\left(\frac{p}{D}\right)^{\floor{p}}<\frac{1}{\pi}\left(\frac{p}{p(\pi\epsilon)^{-1/\floor{p}}}\right)^{\floor{p}} &\text{(\lemref{lma:bound binom coef})}\\
        &&&=\frac{1}{\pi}\left((\pi\epsilon)^{1/\floor{p}}\right)^{\floor{p}}=\epsilon\\
    \end{aligned}
    \end{equation*}
    
    Let $x\in (0, 1]$. From \lemref{lma:bin series shifted}, we have $\psi_{p}(x) = \sum\limits_{i=0}^{\infty}\binom{p}{i}(x-1)^{i}$ with a remainder $R_{D}(x) = \psi_{p}(x)-P_{\epsilon, p}(x)$. Now, $\psi_{p}$ is $(D+1)$-differentiable on $(\frac{x}{2}, x)$ and $\psi^{(D+1)}_{p}$ is continuous on $[\frac{x}{2}, x]$ and so by the mean-value form of the remainder, there exists some $c\in [\frac{x}{2},x]$ such that $R_{D}(x)=\frac{\psi^{(D+1)}_{p}(c)}{(D+1)!}(x-x/2)^{D+1}$. Therefore
    \begin{equation*}
    \begin{aligned}
        &&\left|P_{\epsilon, p}(x)-\psi_{p}(x)\right|&=\left|R_{D}(x)\right|=\left|\frac{\psi^{(D+1)}_{p}(c)}{(D+1)!}(x-x/2)^{D+1}\right|\\
        &&& = \left| \frac{p(p-1)\cdots (p-D)}{(D+1)!}c^{p-D-1}(x/2)^{D+1} \right|\\
        &&& = \left| \binom{p}{D+1}c^{p-D-1}(x/2)^{D+1} \right|\\
        &&& = \left| \binom{p}{D+1}\right|c^{p-D-1}(x/2)^{D+1} 
    \end{aligned}
    \end{equation*}
    Note that $p < D$ and so $c\mapsto c^{p-D-1}$ is decreasing on $[x/2,x]$ and so 
    \begin{equation*}
    \begin{aligned}
        &&\left| \binom{p}{D+1}\right|c^{p-D-1}(x/2)^{D+1}&\leq \left| \binom{p}{D+1}\right|(x/2)^{p-D-1}(x/2)^{D+1} \\
        &&& = \left| \binom{p}{D+1}\right|(x/2)^{p} \leq \left|\binom{p}{D+1}\right|\\
        &&&\leq \frac{1}{\pi}\left(\frac{p}{D}\right)^{\floor{p}}<\frac{1}{\pi}\left(\frac{p}{p(\pi\epsilon)^{-1/\floor{p}}}\right)^{\floor{p}} &\text{(\lemref{lma:bound binom coef})}\\
        &&&=\frac{1}{\pi}\left((\pi\epsilon)^{1/\floor{p}}\right)^{\floor{p}}=\epsilon
    \end{aligned}
    \end{equation*}
\end{proof}

\begin{lemma}\label{lma:bin series shifted}
    Let $p\in(1,\infty)$, $x\in [0, 2]$ then the binomial series $\sum\limits_{i=0}^{\infty}\binom{p}{i}(x-1)^{i}$ converges absolutely to $x^{p}$.
\end{lemma}

\begin{proof}[Proof of \lemref{lma:bin series shifted}]
    Let $x\in [0,2]$ and denote $u=x-1$, then $u\in [-1,1]$ and so by \lemref{lma:bin series} the series $\sum\limits_{i=0}^{\infty}\binom{p}{i}u^{i}$ converges absolutely to $(1+u)^{p}$. We conclude that series $\sum\limits_{i=0}^{\infty}\binom{p}{i}(x-1)^{i}$ converges absolutely to $x^{p}$.
\end{proof}

\begin{lemma}\cite[Theorem~247(a)]{knopp1964theorie}\label{lma:bin series}
    Let $p\in(1,\infty)$, $|x|\leq 1$ then the binomial series $\sum\limits_{i=0}^{\infty}\binom{p}{i}x^{i}$ converges absolutely to $(1+x)^{p}$.
\end{lemma}

\begin{lemma}\label{lma:bound binom coef}
    Let $D\in\mathbb{N}$ and any real $1 < p < D$ such that $p\notin\mathbb{N}$. Then
    \begin{equation*}
        \left|\binom{p}{D+1}\right| \leq \left|\binom{p-1}{D}\right| \leq \frac{1}{\pi}\left(\frac{p}{D}\right)^{\floor{p}}.
    \end{equation*}
\end{lemma}

\begin{proof}[Proof of \lemref{lma:bound binom coef}]
    For the first inequality note that because $p<D$ we have
    \begin{equation*}
    \begin{aligned}
        &&\left|\binom{p}{D+1}\right|&=\left|\frac{\Gamma(p+1)}{\Gamma(D+2)\Gamma(p-D)}\right|=\left|\frac{p\Gamma(p)}{(D+1)\Gamma(D+1)\Gamma(p-D)}\right|\\
        &&&=\frac{p}{D+1}\left| \binom{p-1}{D} \right| \leq \left| \binom{p-1}{D} \right|~.
    \end{aligned}
    \end{equation*}
    Now $p-D\notin \mathbb{Z}$ so
    \begin{equation*}
        \left|\frac{1}{\Gamma(p-D)}\right| = \left|\frac{\sin(\pi(p-D))\Gamma(1+D-p)}{\pi}\right| \leq \frac{1}{\pi}\left|\Gamma(1+D-p)\right|~,
    \end{equation*}
    and therefore
    \begin{equation*}
        \left| \binom{p-1}{D} \right| = \left| \frac{\Gamma(p)}{\Gamma(D+1)\Gamma(p-D)} \right| \leq \frac{1}{\pi}\left| \Gamma(p)\frac{\Gamma(D-p+1)}{\Gamma(D+1)} \right|~.
    \end{equation*}
    
    Note that $\Gamma$ is increasing on $[1.5, \infty)$ and we have $D-\floor{p}+1 \geq 2$, $D-\floor{p}+1 \geq D-p+1\geq 1$, $\Gamma(1)=\Gamma(2)=1$ and so $\Gamma(D-\floor{p}+1)\geq \Gamma(D-p+1)$. Hence
    \begin{equation*}
        \Gamma(D+1)=\prod\limits^{\floor{p}-1}_{i=0}(D-i) \cdot \Gamma(D-\floor{p}+1)\geq \prod\limits^{\floor{p}-1}_{i=0}(D-i) \cdot \Gamma(D-p+1)~.
    \end{equation*}
    After rearranging we get $\frac{\Gamma(1+D-p)}{\Gamma(D+1)} \leq \frac{1}{\prod\limits^{\floor{p}-1}_{i=0}(D-i)}$. Furthermore, since $0<\Gamma(p-\floor{p}+1)\leq 1$ we have
    \begin{equation*}
        \Gamma(p)=\prod\limits^{\floor{p}-1}_{i=1}(p-i) \cdot \Gamma(p-\floor{p}+1) \leq \prod\limits^{\floor{p}-1}_{i=1}(p-i)\leq \prod\limits^{\floor{p}-1}_{i=0}(p-i)~.
    \end{equation*}
    Together we obtain
    \begin{equation*}
        \frac{1}{\pi}\left| \Gamma(p)\frac{\Gamma(D-p+1)}{\Gamma(D+1)} \right| \leq \frac{1}{\pi}\prod\limits^{\floor{p}-1}_{i=0}\frac{p-i}{D-i}\leq \frac{1}{\pi}\prod\limits^{\floor{p}-1}_{i=0}\frac{p}{D}=\frac{1}{\pi}\left(\frac{p}{D}\right)^{\floor{p}}.
    \end{equation*}
\end{proof}

\begin{lemma}\label{lma:l_alpha l_beta relations}
    For every $0<\alpha < \beta \leq \infty$ and every $v\in\mathbb{R}^d$ one has
    \begin{equation*}
        \lVert v \rVert_{\beta}\leq \lVert v \rVert_{\alpha} \leq d^{\frac{1}{\alpha}-\frac{1}{\beta}}\lVert v \rVert_{\beta}~,
    \end{equation*}
    where we define $\frac{1}{\infty}=0$.
\end{lemma}

\begin{proof}[Proof of \lemref{lma:l_alpha l_beta relations}]
    We prove the two inequalities separately.
    \begin{itemize}
        \item $\lVert v \rVert_{\beta}\leq \lVert v \rVert_{\alpha}$:
        \begin{itemize}
            \item If $\beta=\infty$, then 
            \begin{equation*}
                \lVert v \rVert_{\beta} = \lvert v_{\text{max}}\rvert = \left(\lvert v_{\text{max}}\rvert^{\alpha}\right)^{\frac{1}{\alpha}}\leq \left( \sum\limits_{i=1}\limits^{d}\lvert v_{i}\rvert^{\alpha} \right)^{\frac{1}{\alpha}}\leq\lVert v \rVert_{\alpha}
            \end{equation*}
            \item If $\beta < \infty$, then for every $1 \leq i \leq d$ we have $\frac{\lvert v_{i} \rvert}{\lVert v \rVert_{\beta}}\leq 1$ and so $\left( \frac{\lvert v_{i} \rvert}{\lVert v \rVert_{\beta}} \right)^{\beta} \leq \left( \frac{\lvert v_{i} \rvert}{\lVert v \rVert_{\beta}} \right)^{\alpha}$. Therefore, 
            \begin{equation*}
                \frac{\lVert v \rVert_{\alpha}}{\lVert v \rVert_{\beta}}=\left( \sum\limits_{i=1}\limits^{d} \left( \frac{\lvert v_{i} \rvert}{\lVert v \rVert_{\beta}} \right)^{\alpha} \right)^{\frac{1}{\alpha}}\geq \left( \sum\limits_{i=1}\limits^{d} \left( \frac{\lvert v_{i} \rvert}{\lVert v \rVert_{\beta}} \right)^{\beta} \right)^{\frac{1}{\alpha}}=\frac{\lVert v \rVert_{\beta}^{\frac{\beta}{\alpha}}}{\lVert v \rVert_{\beta}^{\frac{\beta}{\alpha}}}=1~.
            \end{equation*}
        \end{itemize}
        \item $\lVert v \rVert_{\alpha} \leq d^{\frac{1}{\alpha}-\frac{1}{\beta}}\lVert v \rVert_{\beta}$:
        \begin{itemize}
            \item If $\beta=\infty$, then 
            \begin{equation*}
                \lVert v \rVert_{\alpha} = \left( \sum\limits_{i=1}\limits^{d}\lvert v_{i}\rvert^{\alpha} \right)^{\frac{1}{\alpha}} \leq \left( \sum\limits_{i=1}\limits^{d}\lvert v_{\text{max}}\rvert^{\alpha} \right)^{\frac{1}{\alpha}}=d^{\frac{1}{\alpha}}\lVert v \rVert_{\infty}=d^{\frac{1}{\alpha}-\frac{1}{\beta}}\lVert v \rVert_{\beta}~.
            \end{equation*}
            \item If $\beta < \infty$, denote $r_{1}=\frac{\beta}{\alpha}>1$ and $r_{2}=\frac{r_{1}}{r_{1}-1}$ then $1<r_{1},r_{2}$ and $\frac{1}{r_{1}}+\frac{1}{r_{2}}=1$ hence by Holder's inequality
            \begin{equation*}
                \sum\limits_{i=1}\limits^{d}\left(\lvert v_{i}\rvert^{\alpha}\right)\cdot 1 \leq  \left( \sum\limits_{i=1}\limits^{d}\left(\lvert v_{i}\rvert^{\alpha}\right)^{r_{1}} \right)^{\frac{1}{r_{1}}} \left( \sum\limits_{i=1}\limits^{d}\left(1\right)^{r_{2}} \right)^{\frac{1}{r_{2}}}=\left( \sum\limits_{i=1}\limits^{d}\lvert v_{i}\rvert^{\beta} \right)^{\frac{\alpha}{\beta}}d^{1-\frac{\alpha}{\beta}}.
            \end{equation*}
            Therefore, $\lVert v \rVert_{\alpha} \leq \left( \sum\limits_{i=1}\limits^{d}\lvert v_{i}\rvert^{\beta} \right)^{\frac{1}{\beta}}d^{\frac{1}{\alpha}\left(1-\frac{\alpha}{\beta}\right)} = d^{\frac{1}{\alpha}-\frac{1}{\beta}}\lVert v \rVert_{\beta}$
        \end{itemize}
    \end{itemize}
\end{proof}

\begin{lemma}\label{lma:lp lq relations}
    For every $p\in(0,\infty]$, $q\in [1,\infty]$ we denote $c^{+}_{p,q}(d)=d^{\left[\frac{1}{q}-\frac{1}{p}\right]_{+}}, c^{-}_{p,q}(d)=d^{\left[\frac{1}{q}-\frac{1}{p}\right]_{-}}$. Then for every $v\in\mathbb{R}^{d}$ one has
    \begin{equation*}
        c^{-}_{p,q}(d)\lVert v \rVert_{p}\leq \lVert v \rVert_{q} \leq c^{+}_{p,q}(d)\lVert v \rVert_{p}~,
    \end{equation*}
    where we define $\frac{1}{\infty}=0$. Furthermore, when $q<p$ the upper bound is an equality for $v=(d^{-1/p}, \cdots ,d^{-1/p})$ and the lower bound is an equality for $v=e_{1}$. When $p<q$ the upper bound is an equality for $v=e_{1}$ and the lower bound is an equality for $v=(d^{-1/p}, \cdots, d^{-1/p})$.
\end{lemma}

\begin{proof}[Proof of \lemref{lma:lp lq relations}]
    We consider the cases $p<q$ and $q < p$ (the case $p=q$ is trivial).
    \begin{itemize}
        \item If $q < p$, then by the definition of $c^{-}_{p,q}(d), c^{+}_{p,q}(d)$ and from \lemref{lma:l_alpha l_beta relations} we have
        \begin{equation*}
            c^{-}_{p,q}(d)\lVert v \rVert_{p} = \lVert v \rVert_{p} \leq \lVert v \rVert_{q} \leq d^{\frac{1}{q}-\frac{1}{p}}\lVert v \rVert_{p} = c^{+}_{p,q}(d)\lVert v \rVert_{p}~.
        \end{equation*}
        For $v=(d^{-1/p}, \cdots ,d^{-1/p})$ we get $\lVert v \rVert_{q} = c^{+}_{p,q}(d)\lVert v \rVert_{p}$, and for $v=e_{1}$ we get $c^{-}_{p,q}(d)\lVert v \rVert_{p} = \lVert v \rVert_{q}$
        \item If $p < q$, then by the definition of $c^{-}_{p,q}(d), c^{+}_{p,q}(d)$ and the right inequality in \lemref{lma:l_alpha l_beta relations} we have $\lVert v \rVert_{p}\leq d^{\frac{1}{p}-\frac{1}{q}} \lVert v \rVert_{q}$ so $c^{-}_{p,q}(d)\lVert v \rVert_{p}\leq \lVert v \rVert_{q}$, and from the left inequality in \lemref{lma:l_alpha l_beta relations} we have $\lVert v \rVert_{q}  \leq \lVert v \rVert_{p} = c^{+}_{p,q}(d)\lVert v \rVert_{p}$.

        For $v=e_{1}$ we get $\lVert v \rVert_{q} = c^{+}_{p,q}(d)\lVert v \rVert_{p}$, and for $v=(d^{-1/p}, \cdots ,d^{-1/p})$ we get $c^{-}_{p,q}(d)\lVert v \rVert_{p} = \lVert v \rVert_{q}$
    \end{itemize}
\end{proof}

\begin{lemma}\label{lma:lp ball and lq ball}
    For any $d\in\mathbb{N}$, $0<r$ and $p\in (0,\infty]$, $q\in[1,\infty]$ we have $B^{d}_{p}(r)\subseteq B^{d}_{q}(c^{+}_{p,q}(d)r)$, and $B^{d}_{q}(c^{-}_{p,q}(d)r)\subseteq B^{d}_{p}(r)$. Furthermore, any $\alpha<c^{+}_{p,q}(d)r$ does not satisfy the first inclusion and any $c^{-}_{p,q}(d)r<\beta$ does not satisfy the second inclusion.
\end{lemma}

\begin{proof}[Proof of \lemref{lma:lp ball and lq ball}]
    Follows immediately from \lemref{lma:lp lq relations}.
\end{proof}

\begin{lemma}\label{lma:ortho proj keeps l2 balls}
    Let $P\in\text{Gr}_{d,k}$ be an orthogonal projection, $x\in \mathbb{R}^{d}$, $0\leq r$ and  $a^{\prime}\in B^{k}_{2}(Px, r)$ (where $B^{k}_{2}$ here is a ball in $\text{Im}P\cong\mathbb{R}^{k}$, $\text{Im}P\subset\mathbb{R}^{d}$), then there exists some $a\in B^{d}_{2}(x, r)$ such that $Pa=a^{\prime}$.
\end{lemma}

\begin{proof}[Proof of \lemref{lma:ortho proj keeps l2 balls}]
    Define $a=a^{\prime} + x - Px$, then $\lVert a - x\rVert_{2} = \lVert a^{\prime} + x - Px - x\rVert_{2} = \lVert a^{\prime} - Px\rVert_{2} \leq r$ and so $a\in B^{d}_{2}(x, r)$. Finally, $P(a)=P(a^{\prime} + x - Px) = Pa^{\prime} + Px - PPx = a^{\prime}$.
\end{proof}

\begin{lemma}\label{lma:ortho proj equi to ortho act}
    $\forall g\in O\left(d\right), x\in\mathbb{R}^{d}$, and every subspace $W\in\text{Gr}_{d,k}$ we have $\lVert P_{gW}\left(x\right) \rVert_{2} = \lVert P_{W}\left(g^{-1}x\right) \rVert_{2}$.
\end{lemma}

\begin{proof}[Proof of \ref{lma:ortho proj equi to ortho act}]
    By definition, if $W=Sp\left\{w_{1}, ..., w_{k}\right\}$ for some orthonormal basis, then $P_{W}=A_{W}A_{W}^{\top}$ where $A_{W}$ has $w_{1}, ..., w_{k}$ as column vectors, and $gW=Sp\left\{gw_{1}, ..., gw_{k}\right\}$ and so $A_{gW}=gA_{W}$ which means that $P_{gW} = gP_{W}g^{\top}=gP_{W}g^{-1}$. Therefore $\lVert P_{gW}\left(x\right) \rVert_{2} =\lVert gP_{W}g^{-1}\left(x\right) \rVert_{2} = \lVert gP_{W}\left(g^{-1}x\right) \rVert_{2} = \lVert g\left(P_{W}\left(g^{-1}x\right)\right \rVert_{2} = \lVert P_{W}\left(g^{-1}x\right) \rVert_{2}$, where the last equality follows from the fact that the orthogonal group consists of endomorphisms that preserve the Euclidean norm.
\end{proof}

\begin{lemma}\label{lma:action preserves cover}
    Let $X$ be a $G$-space, and let $\mathcal{U}=\left\{U_{\alpha}\right\}_{\alpha\in A}$ be a cover of $X$. Then for every $g\in G$, $g\mathcal{U}=\left\{gU_{\alpha}\right\}_{\alpha\in A}$ is a cover of $X$.
\end{lemma}

\begin{proof}[Proof of \lemref{lma:action preserves cover}]
    Let $g\in G$, and let $x\in X$. Denote $x^{\prime}=g^{-1}x$, then there exists some $\alpha \in A$ such that $x^\prime \in U_{\alpha}$, so $x = gx^\prime\in gU_{\alpha}$ and we are done.
\end{proof}
\begin{lemma}\label{lma:grass identity}
    Let $A\subseteq\mathbb{R}^{d}$ be some subset of $\mathbb{R}^{d}$ such that $\mathbb{R}A\in\text{Gr}_{d,l_{1}}$ for some $1\leq l_{1} \leq d$. Let $V\in\text{Gr}_{d,l_{2}}$ with $l_{2}\geq d-l_{1}+1$, then there exists some $0\neq x\in V \cap \mathbb{R}A$.
\end{lemma}

\begin{proof}[Proof of \lemref{lma:grass identity}]
    Note that $\reals A$ is a vector subspace, hence by Grassmann's Identity we have 
    \begin{equation*}
    \begin{aligned}
    &dim\left(V \cap \mathbb{R}A\right) = dimV + dim\mathbb{R}A - dim\left(V + \mathbb{R}A\right)\geq \\ 
    & dimV + dim\mathbb{R}A - d =l_{2}+l_{1}-d \geq  d-l_{1}+1 + l_{1} - d =1~.
    \end{aligned}
    \end{equation*}
\end{proof}

\begin{lemma}\label{lma:density computation bound}
    For every $k\in \mathbb{N}$
    \begin{equation*}
        \left(5k\ln\left(k+1\right)\sqrt{2\pi\left(k+1\right)}\right)^{\frac{1}{k}} \leq 10\sqrt{\pi}\ln 2~.
    \end{equation*}
\end{lemma}

\begin{proof}[Proof of \lemref{lma:density computation bound}]
    Denote $c=\ln(5) + \frac{\ln(2\pi)}{2}$ then $\frac{c}{k}$ and $\frac{\ln(k+1)}{2k}$ both decrease for every $1\leq k$, and $\frac{\ln k}{k}$ decreases for every $e\leq k$. Furthermore, $\frac{\ln(\ln(k+1))}{k}$ decreases for every $4.14\leq k$, and so we conclude that $\frac{1}{k}\left(c+\ln k + \ln(\ln(k+1)) + \frac{\ln(k+1)}{2} \right)$ decreases for every $4.14 \leq k$. Computing for $k=1,2,3,4,5$ we get that $\frac{1}{k}\left(c+\ln k + \ln(\ln(k+1)) + \frac{\ln(k+1)}{2} \right)$ decreases for every $k\in\mathbb{N}$. Hence, for every $k\in\mathbb{N}$ we have:
    \begin{equation*}
    \begin{aligned}
        &&\frac{1}{k}\left(c+\ln k + \ln(\ln(k+1)) + \frac{\ln(k+1)}{2} \right)&\leq \frac{1}{1}\left(c+\ln 1 + \ln(\ln(1+1)) + \frac{\ln(1+1)}{2} \right)&\\
        &&&=\ln 5 + \frac{1}{2}\ln(2\pi)+\ln\ln 2 + \frac{1}{2}\ln 2 ~.
    \end{aligned}
    \end{equation*}
    Since $x\mapsto \exp(x)$ increases monotonically we get for all $k\in\mathbb{N}$:
    \begin{equation*}
    \begin{aligned}
        &&\left(5k\ln\left(k+1\right)\sqrt{2\pi\left(k+1\right)}\right)^{\frac{1}{k}}&= \exp\left[\ln \left(5k\ln\left(k+1\right)\sqrt{2\pi\left(k+1\right)}\right)^{\frac{1}{k}}\right] &\\
        &&& = \exp\left[\frac{1}{k}\left(c+\ln k + \ln(\ln(k+1)) + \frac{\ln(k+1)}{2} \right)\right] &\\
        &&& \leq \exp\left[\ln 5 + \frac{1}{2}\ln(2\pi)+\ln\ln 2 + \frac{1}{2}\ln 2\right] &\\
        &&& = 5\sqrt{2\pi}\ln 2 \cdot \sqrt{2} = 10\sqrt{\pi}\ln 2~.
    \end{aligned}
    \end{equation*}
\end{proof}

\end{document}